\documentclass[twoside,11pt]{article}
\usepackage[abbrvbib, preprint]{jmlr2e}
\usepackage{hyperref}
\usepackage{nameref}
\usepackage{epstopdf}
\usepackage{amsmath,amsfonts}
\usepackage{algpseudocode}
\usepackage{algorithm}
\usepackage{array}
\usepackage[caption=false,font=footnotesize]{subfig}
\usepackage{textcomp}
\usepackage{url}
\usepackage{verbatim}
\usepackage{graphicx}
\usepackage{amssymb}
\usepackage{booktabs} % for professional tables
\usepackage[dvipsnames]{xcolor}
\usepackage{amstext, amsmath,bm,rotating}
\usepackage{multirow}
\usepackage{multicol}
\usepackage{threeparttable}
\usepackage{footnote}
\usepackage{float}
\usepackage{rotating}

\newtheorem{assumption}{Assumption}

\newcommand{\vx}{{\bf x}}

\newcommand{\valpha}{{\bm \alpha}}

\newcommand{\vbeta}{{\bm \beta}}

\newcommand{\vsigma}{{\bm \sigma}}
\newcommand{\vTheta}{{\bm \Theta}}
\newcommand{\vtheta}{{\bm \theta}}
\newcommand{\ve}{{\bf e}}

\newcommand{\vv}{{\bf v}}
\newcommand{\vz}{{\bf z}}
\newcommand{\vw}{{\bf w}}

\newcommand{\vt}{{\bf t}}

\newcommand{\vu}{{\bf u}}

\newcommand{\vr}{{\bf r}}

\newcommand{\vC}{{\bf C}}

\newcommand{\vB}{{\bf B}}
\newcommand{\vD}{{\bf D}}

\newcommand{\vA}{{\bf A}}

\newcommand{\vY}{{\bf Y}}

\newcommand{\vX}{{\bf X}}

\newcommand{\vI}{{\bf I}}
\newcommand{\vK}{{\bf K}}

\usepackage{lastpage}
\jmlrheading{23}{2024}{1-\pageref{LastPage}}{}{}{}{F. He, M. He, L. Shi, X. Huang, and J.A.K. Suykens}

\ShortHeadings{Learning and Analyzing LAB RBF Kernels}{F. He, M. He, L. Shi, X. Huang, and J.A.K. Suykens}
\firstpageno{1}

\begin{document}

\title{Learning Analysis of Kernel Ridgeless Regression with Asymmetric Kernel Learning}
\author{Authors}

\author{\name Fan He 
        \email fan.he@esat.kuleuven.be\\
        \addr Department of Electrical Engineering (ESAT), STADIUS Center for Dynamical Systems, \\
       Signal Processing and Data Analytics, KU Leuven, 
       Leuven, Belgium
       \AND
       \name Mingzhen He \email mingzhen\_he@sjtu.edu.cn \\
       \addr MOE Key Laboratory of System Control and Information Processing\\
       Institute of Image Processing and Pattern Recognition,
       Shanghai Jiao Tong University\\
       Shanghai, P.R. China
       \AND
       %\name Sanli Tang \email tangsanli@sjtu.edu.cn \\
       %\addr Free Researcher of Shanghai Jiao Tong University\\
       %Shanghai, P.R.  China
       %\AND
       %\addr MOE Key Laboratory of System Control and Information Processing\\
       %Institute of Image Processing and Pattern Recognition\\
       %Institute of Medical Robotics,
       %Shanghai Jiao Tong University,
       %200240, Shanghai,  P.R. China
       \AND
       \name Lei Shi
       \email leishi@fudan.edu.cn \\
       \addr Shanghai Key Laboratory for Contemporary Applied Mathematics\\
       School of Mathematical Sciences, Fudan University, 200433, Shanghai, P.R. China\\
       Shanghai Artificial Intelligence Laboratory, 200232, Shanghai, P.R. China
       \AND
       \name Xiaolin Huang \email xiaolinhuang@sjtu.edu.cn \\
       \addr MOE Key Laboratory of System Control and Information Processing\\
       Institute of Image Processing and Pattern Recognition\\
       Institute of Medical Robotics,
       Shanghai Jiao Tong University,
       200240, Shanghai, P.R. China
       \AND
       \name Johan A.K. Suykens \email johan.suykens@esat.kuleuven.be \\
       \addr Department of Electrical Engineering (ESAT), STADIUS Center for Dynamical Systems, \\
       Signal Processing and Data Analytics, KU Leuven,
       Leuven, Belgium
       }

\editor{My editor}

\maketitle

\begin{abstract}%   <- trailing '%' for backward compatibility of .sty file
Ridgeless regression has garnered attention among researchers, particularly in light of the ``Benign Overfitting'' phenomenon, where models interpolating noisy samples demonstrate robust generalization. However,  kernel ridgeless regression does not always perform well due to the lack of flexibility. This paper enhances kernel ridgeless regression with Locally-Adaptive-Bandwidths (LAB) RBF kernels, incorporating kernel learning techniques to improve performance in both experiments and theory. 
For the first time, we demonstrate that functions learned from LAB RBF kernels belong to an integral space of Reproducible Kernel Hilbert Spaces (RKHSs). Despite the absence of explicit regularization in the proposed model, its optimization is equivalent to solving an $\ell_0$-regularized problem in the integral space of RKHSs, elucidating the origin of its generalization ability.
Taking an approximation analysis viewpoint, we introduce an $l_q$-norm analysis technique (with $0<q<1$) to derive the learning rate for the proposed model under mild conditions. This result deepens our theoretical understanding, explaining that our algorithm's robust approximation ability arises from the large capacity of the integral space of RKHSs, while its generalization ability is ensured by sparsity, controlled by the number of support vectors.
Experimental results on both synthetic and real datasets validate our theoretical conclusions.

\end{abstract}


\section{Introduction}
Kernel methods play a foundational role within the machine learning community, and maintain their importance thanks to their interpretability, strong theoretical foundations, and versatility in handling diverse data types \citep{ghorbani2020neural, bach2022information, jerbi2023quantum}. However, as newer techniques like deep learning gain prominence, kernel methods reveal a shortcoming: the learned function's flexibility often falls short of expectations.
A sufficiently flexible model, often characterized by over-parameterization \citep{allen2019convergence, zhou2024learning}, has attracted researchers' attention due to the phenomenon of \textquotedblleft Benign Overfitting\textquotedblright.
This phenomenon, supported by extensive empirical evidence, particularly in deep learning models, suggests that over-parameterized models have the capacity to interpolate noisy training data and yet exhibit effective generalization on test data \citep{Ma2017ThePO, Montanari2020TheIP, cao2022benign, tsigler2023benign}. 

The identification of benign overfitting has motivated the exploration of ridgeless regression \citep{bartlett2020benign, tsigler2023benign}, particularly  kernel ridgeless regression  \citep{liang2020just}, because the analysis of kernel interpolation methods proves more tractable and provides valuable insights for understanding the behavior of deep neural networks \citep{jacot2018neural, belkin2018understand}.
Let the data space $\mathcal{X}$ be a compact subset of $\mathbb{R}^d$.
We call a kernel a Mercer kernel \citep{aronszajn1950theory} if it is continuous, symmetric and positive semi-definite on $\mathcal{X}\times \mathcal{X}$.
We denote a Mercer kernel by $\mathcal{K}(\cdot, \cdot):\mathcal{X}\times \mathcal{X}\mapsto \mathbb{R}$, and it is defined via $\mathcal{K}(\vx,\vx') = \langle \phi(\vx), \phi(\vx')\rangle,\forall \vx,\vx'\in\mathcal{X}$. 
Its generated RKHS is denoted as $(\mathcal{H},\langle \cdot, \cdot\rangle_{\mathcal{K}})$, where
$
    \mathcal{H}_\mathcal{K} = \overline{\mathrm{span}}\{\mathcal{K}(\vx,\cdot): \vx\in\mathcal{X}\}
$ with $\langle \mathcal{K}(\vx,\cdot), \mathcal{K}(\vx',\cdot)\rangle_{\mathcal{K}}=\mathcal{K}(\vx,\vx')$.
Let $\mathcal{Y}\subset\mathbb{R}$ and $\mathcal{Z}=\mathcal{X}\times\mathcal{Y}$.
Denote observations $\vz=\{\vz_i =(\vx_i,y_i)\}_{i=1}^N\subset\mathcal{Z}^N$, which are independently drawn from some Borel probability distribution $\rho$ on $\mathcal{Z}$.
%Then given a Mercer kernel $\mathcal{K}$, kernel-based learning aims to search the optimal decision function $f^*:\mathcal{X}\rightarrow\mathcal{Y}$ in the corresponding $\mathcal{H}$.
Then by adding a regularization to the least-squares loss function, a classical regression model is obtained as follows,
\begin{equation}\label{equ: t problem}
    f_{\mathrm{ridge}} = \arg\min_{f\in\mathcal{H}}\;\;
    \sum_{(\vx_i,y_i)\in\vz}\|f(\vx_i)-y_i\|_2^2+\lambda\|f\|^2_{\mathcal{H}},
\end{equation}
where $\lambda>0$ is a pre-given trade-off parameter and $\|\cdot\|_{\mathcal{H}}$ is the norm induced by the inner product $\langle \cdot,\cdot \rangle_{\mathcal{K}}$. 
The model described in Equation (\ref{equ: t problem}) is referred to as kernel ridge regression \citep{vovk2013kernel}. According to the Representer theorem, its solution can be  represented as a linear combination of function evaluations on the training dataset, i.e.,
\begin{equation}\label{equ: solution formulation}
    f(\vt) = \sum_i\alpha_i\mathcal{K}(\vt, \vx_i),
\end{equation}
where $\valpha\in\mathbb{R}^N$ denotes the combination coefficients.
Then kernel interpolation is achieved via the kernel ridgeless regression model by setting $\lambda=0$ in (\ref{equ: t problem}). That is,
\begin{equation}\label{equ: ridgeless problem}
    f_{\mathrm{ridgeless}} = \lim_{\lambda\to 0}\left\{\arg\min_{f\in\mathcal{H}}\;\;
    \sum_{(\vx_i,y_i)\in\vz}\|f(\vx_i)-y_i\|_2^2+\lambda\|f\|^2_{\mathcal{H}}\right\},
\end{equation}
of which the solution is not unique, but one of them takes the same form as Equation (\ref{equ: solution formulation}) \citep{rakhlin2019consistency, lin2024kernel}.

However, kernel ridgeless regression does not always performs well. In theoretical analysis, current investigations show that ridgeless regression only exhibits the benign overfitting phenomenon under the assumption of a high-dimensional regime \citep{hastie2022surprises, mei2022generalization}. In low-dimensional scenarios \citep{buchholz2022kernel} or fixed-dimensional setups \citep{beaglehole2023inconsistency}, the phenomenon is not valid for interpolating kernel machines with popular kernels, such as Gaussian, Laplace, and Cauchy kernels.

This coincides with our practical observation that the performance of kernel ridgeless regression can be unsatisfactory.
As shown in Figure~\ref{fig: example} (a), the traditional kernel interpolation model using a single RBF kernel is not robust to noisy data and fails to fit signals with varying frequency.
As the key insight of benign overfitting or the double descent phenomenon is to leverage over-parameterized models for sample interpolation \citep{allen2019learning,chatterji2021finite, tsigler2023benign}, the imperfect interpolation observed in kernel machines can be attributed to its inherent lack of flexibility.
In the context of kernel ridgeless regression models, as shown in Equation (\ref{equ: solution formulation}), the resulting interpolation function only has only $N$ free parameters, making its flexibility considerably less than that of over-parameterized deep models.   Due to this challenge, current experimental investigations of over-parameterized kernel machines often resort to techniques such as random feature \citep{liu2022double} or neural tangent kernels \citep{adlam2020neural}.

%Figure~\ref{fig: example} illustrates a toy example where we aim to approximate the underlying function $y=\sin(3x)+x^2$ through the interpolation of noisy samples. Using traditional RBF kernels, as depicted in (a) and (b), the interpolated function suffers from overfitting, regardless of the number of samples considered, whether small or large. Due to this challenge, current experimental investigations of over-parameterized kernel machines often resort to techniques such as random feature regression \citep{liu2022double} or NTK \citep{adlam2020neural}.

\begin{figure*}[tb]
\begin{center}
\centerline{\includegraphics[width=0.95\textwidth]{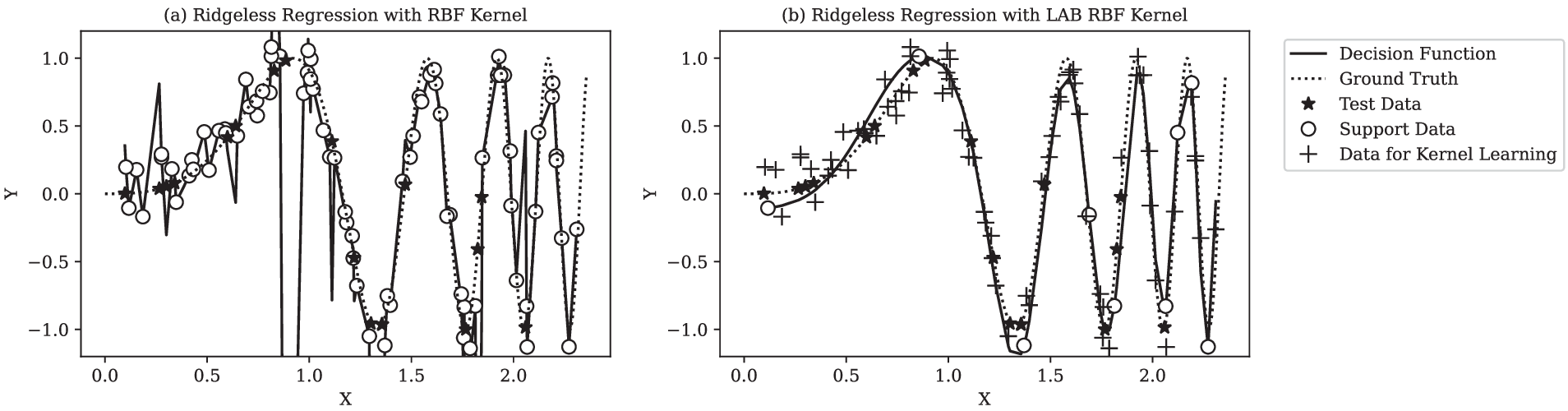}}
\caption{A toy example illustrating kernel ridgeless regression applied to a one-dimensional signal $y=\sin(2x^3)$. In (a), the traditional RBF kernel is utilized, directly interpolating all data points. In (b), asymmetric kernel learning is applied, where a small subset is used as support data and the LAB RBF kernel is learnt from the remaining data.}
\label{fig: example}
\end{center}
\end{figure*}

Recognizing this problem, this paper introduces a solution by enhancing the model with an asymmetric kernel leaning technique.
Specifically, we propose to utilize an asymmetric RBF kernel incorporating with locally adaptive bandwidths as follows,
\begin{equation}\label{equ: kernel function}
    \mathcal{K}(\vx,\vx_i) 
    = \exp\left\{- \|\theta_i\odot(\vx-\vx_i)\|_2^2\right\},\qquad \forall \vx_i\in\vX_{tr}.
\end{equation}
We name the above kernel function as the Local-Adaptive-Bandwidth RBF (LAB RBF) kernel. The distinguishing feature of LAB RBF kernels, in comparison to conventional RBF kernels, is the assignment of distinct bandwidths to each sample $\vx_i$ rather than utilizing a uniform bandwidth across all data points. In this approach, we discretely define the bandwidth for each training data point individually.

By incorporating asymmetric kernel learning, a new framework for kernel ridgeless regression is proposed in this paper.
As illustrated in Figure~\ref{fig: framework}, our approach not only learns the coefficient $\valpha$, but also estimates the specific values of $\theta_i\in \mathbb{R}^d,\forall i$ from the training data.
The inclusion of data-dependent $\theta_i$ greatly enhances flexibility, thereby expanding the hypothesis space significantly, as we will explore in subsequent sections. Leveraging this expanded hypothesis space, it becomes feasible to search for an interpolation function with fewer support data, facilitating a discrete optimization of support data. As shown in Figure~\ref{fig: example} (b), this method provides an estimator with varying bandwidths, enabling it to accurately approximate different frequency components of the signal. Furthermore, it demonstrates good generalization ability in the presence of noise, despite the absence of an explicit regularization term in our approach.

\begin{figure*}[tb]
\begin{center}
\centerline{\includegraphics[width=0.5\textwidth]{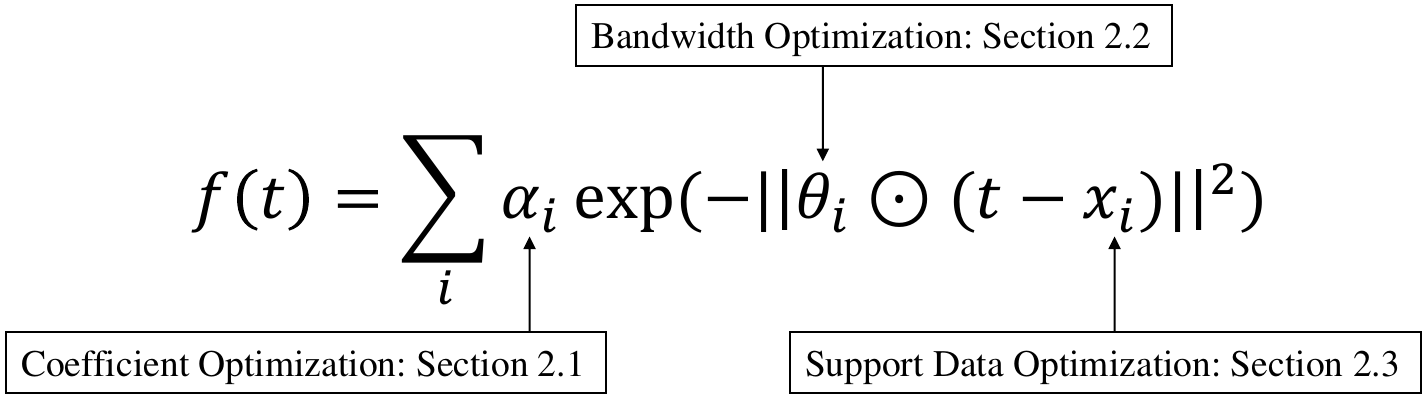}}
\caption{Optimization for evaluating $f$ in our kernel ridgeless regression framework. To enhance the model's flexibility, we introduce trainable bandwidths, which further enable the reduce of required number of support data.}
\label{fig: framework}
\end{center}
\end{figure*}

However, the absence of regularization term and the inherent asymmetry of LAB RBF kernels bring challenges to corresponding theoretical analysis. Specifically, $\theta_i$ and $\theta_j$ may differ, causing $\mathcal{K}(\vx_i,\vx_j)$ to potentially differ from $\mathcal{K}(\vx_j,\vx_i)$.
The loss of symmetry precludes the direct application of traditional analysis tools within the scope of Reproducing Kernel Hilbert Spaces (RKHS, \cite{Cucker2007Learning}), and even more general cases such as Reproducing Kernel Kre{\u\i}n Spaces (RKKS, \cite{oglic2018learning}) and Reproducing Kernel Banach Spaces (RKBS, \cite{zhang2009reproducing}).

In this paper, we overcome these challenges and successfully establish the generalization analysis for the proposing algorithm that addresses kernel ridgeless regression using LAB RBF kernel learning; see Theorem~\ref{the: main result}. In particular,

\begin{enumerate}
    \item We demonstrate a novel approach to analyze the asymmetric LAB RBF kernels within the existing framework of approximation theory \citep{Cucker2007Learning} by introducing the integral space of Reproducing Kernel Hilbert Spaces (RKHSs). This integral space can be viewed as a non-trivial extension from the direct sum of Hilbert space, a method previously employed for analyzing Multiple Kernel Learning (MKL).
    To the best of our knowledge, this marks the first effort to introduce the integral space of RKHSs to machine learning. %Its expansive capacity presents challenges in both optimization and analysis.
    \item 
    We uncover the inherent sparsity of the estimator produced from LAB RBF kernels. Subsequently, we establish an equivalent $\ell_0$-related model within the integral space of RKHSs. This exploration addresses the origins of generalization ability and sheds light on the implicit regularization mechanisms at play.
\end{enumerate}

Our key insights are twofold: (i) the trainable bandwidths effectively enrich the expansive functional spaces, enhancing the representation ability of LAB RBF kernels. This enhancement allows our algorithm to interpolate the training dataset with only a few support data. (ii) Simultaneously, the inherent sparsity of LAB RBF kernels, controlled by the number of support vectors, ensures their robust generalization ability, as evidenced in the analysis of sample error. Notably, the number of support vectors plays a pivotal role in balancing the approximation ability within the training data and the generalization ability within the test data, a observation validated by our experimental results.

%Illustrated in Figure~\ref{fig: example} (c), the kernel ridgeless regression model successfully interpolates the same sample as depicted in (a). However, the two functions differ due to the utilization of additional information (akin to (b)) through trained bandwidths, resulting in notably improved generalization ability. Notably, the entire process lacks explicit regularization, yet the outcome exhibits a similar property in mitigating the impact of noise.

%\textbf{Organization.} 
The remainder of this paper is organized as follows. In Section \ref{Chat:LAB_RBF}, we first establish the framework of asymmetric kernel ridgeless regression. Subsequently, we incorporate LAB RBF kernels into this framework and introduce a solving algorithm for learning local bandwidths and the regression function.
In Section \ref{sec: space}, we define the function space corresponding to LAB RBF kernels, which is an integral space of RKHSs. We determine the corresponding learning model, setting the foundations for the subsequent analysis.
In Section \ref{Chapt:Analysis}, we derive theoretical results on the error analysis of kernel ridgeless regression with LAB RBF kernels. 
In Section \ref{sec: exp}, we substantiate our theoretical findings with experimental results, demonstrating the practical implications of our proposed approach.
Related works are discussed in Section~\ref{sec: Link}.
Finally, a conclusion is provided in Section \ref{Chapt:conclusion}.

\section{Kernel Ridgeless Regression with LAB RBF Kernels}\label{Chat:LAB_RBF}
In this paper, we use calligraphic letters to denote datasets like $\mathcal{X} = \{\vx_1,\cdots,\vx_N\}\subset\mathbb{R}^d, \mathcal{Y}= \{y_1,\cdots,y_N\}\subset\mathbb{R}$.
We use captain letters in bold to denote data matrix, i.e., $\vX=[\vx_1,\vx_2,\cdots,\vx_N]\in\mathbb{R}^{M\times N}, \vY=[y_1,y_2,\cdots,y_N]^\top\in\mathbb{R}^N$.
We use $\vK(\vX_1,\vX_2)$ to denote the kernel matrix computed on datasets $\vX_1$ and $\vX_2$. That is, $[\vK(\vX_1,\vX_2)]_{ij}=\mathcal{K}(\vx_i,\vx_j)$, $\forall \vx_i\in \mathcal{X}_1,\vx_j \in \mathcal{X}_2$.
The task is to find a linear function in a high dimensional feature space, denoted as $\mathbb{R}^F$, which models the dependencies between the features $\phi(\vx_i), \forall\vx_i\in\mathcal{X}$ of input and
response variables $y_i, \forall y_i\in\mathcal{Y}$.
Throughout this paper, RBF kernels are considered.
In order to distinguish LAB RBF kernels from the conventional RBF kernels, we use $\theta\in \mathbb{R}^M_{+}$ to denote trainable bandwidths and use $\vsigma\in \mathbb{R}^M_{+}$ for fixed bandwidths.

\subsection{Asymmetric Kernel Ridgeless Model: Coefficient Optimization}
In this paper, we propose the utilization of LAB RBF kernels (\ref{equ: kernel function}) in the kernel ridgeless regression model (\ref{equ: ridgeless problem}). However, determining the solution of the kernel ridgeless regression model with asymmetric kernels remains unresolved, as the conventional kernel trick is no longer applicable. In this section, we derive the solution using the asymmetric kernel trick, beginning with a brief review of kernel ridge regression.

Kernel ridge regression \citep{vovk2013kernel} is one of the most elementary kernelized algorithms.
The task is to find a linear function in a high dimensional feature space, denoted as $\mathbb{R}^F$, which models the dependencies between the features $\phi(\vx_i), \forall\vx_i\in\mathcal{X}$ of input and
response variables $y_i, \forall y_i\in\mathcal{Y}$.
Here, $\phi: \mathbb{R}^M\to \mathbb{R}^F$ denotes the feature mapping from the data space to the feature space.
Define $\phi(\vX) = [\phi(\vx_1),\phi(\vx_2),\cdots,\phi(\vx_N)]$, then the classical optimization model is as follow:
\begin{equation}\label{equ: skrr}
    \min_{\vw}\;\; \frac{\lambda}{2}\vw^\top \vw + \frac{1}{2}\|\vY - \phi(\vX)^\top \vw\|_2^2,
\end{equation}
where $\lambda>0$ is a trade-off hyper-parameter.
By utilizing the following well-known matrix inversion lemma (see \cite{petersen2008matrix} for more information),
\begin{equation}\label{equ: eq}
        (\vA+\vB \vD^{-1}\vC)^{-1}\vB \vD^{-1} = \vA^{-1}\vB(\vC\vA^{-1}\vB +\vD)^{-1},
    \end{equation}
one can obtain the solution of KRR as follow
\begin{equation*}
\begin{aligned}    
    \vw^* &= (\phi(\vX)\phi(\vX)^\top+\lambda\vI_F)^{-1}\phi(\vX)\vY  \overset{(a)}= \phi(\vX)(\lambda\vI_N + \phi(\vX)^\top\phi(\vX))^{-1}\vY,
\end{aligned}
\end{equation*}
where (\ref{equ: eq}) is applied in (a) with $\vA=\lambda\vI_F, \; \vB=\phi(\vX),\; \vC =\phi^\top(\vX),\; \vD=\vI_N$. 

Next, we consider applying asymmetric kernels in KRR framework. In recent research on asymmetric kernel-based learning, asymmetric kernels are commonly assumed to be the inner product of two distinct feature mappings (see \cite{suykens2016svd, he2023learning, chen2024primal} for reference). 
That is, $\mathcal{K}(\vt,\vx) = \langle \phi(\vt), \psi(\vx)\rangle, \;\;\forall \vx,\vt \in \mathbb{R}^M.$
Then, imitating the model (\ref{equ: skrr}), we can formulate the asymmetric kernel ridge regression as follows,
\begin{equation}\label{equ: asy krr}
\begin{aligned}
    & \min_{\vw,\vv} \; \lambda \vw^\top \vv + (\phi^\top(\vX)\vw-\vY)^\top(\psi^\top(\vX)\vv-\vY)\\
    \Longleftrightarrow&
    \min_{\vw,\vv} \; \lambda \vw^\top \vv + \frac{1}{2}\|\phi^\top(\vX)\vw-\vY\|_2^2
    +\frac{1}{2}\|\psi^\top(\vX)\vv-\vY\|_2^2 - \frac{1}{2}\|\psi^\top(\vX) \vv - \phi^\top(\vX) \vw\|_2^2.
\end{aligned}
\end{equation}
Here, $\lambda>0$ serves as a trade-off hyper-parameter between the regularization term $\vw^\top \vv$ and the error term $(\phi^\top(\vX)\vw-\vY)^\top(\psi^\top(\vX)\vv-\vY)$. Given the existence of two feature mappings, we have two regressors in $\mathbb{R}^F$: $f_1(\vt) = \phi^\top(\vt)\vw$ and $f_2(\vt) = \psi^\top(\vt)\vv$.
To enhance clarity regarding the meaning of the error term, we decompose it into the sum of three terms, as shown in the second line. The terms $\frac{1}{2}\|\phi^\top(\vX)\vw-\vY\|_2^2+\frac{1}{2}\|\psi^\top(\vX)\vv-\vY\|_2^2$ are employed to minimize the regression error. Additionally, the term $\lambda \vw^\top \vv- \frac{1}{2}\|\psi^\top(\vX) \vv - \phi^\top(\vX) \vw\|_2^2$ aims to emphasize the substantial distinction between the two regressors.

As a bilinear optimization problem, the one presented in Equation (\ref{equ: asy krr}) is non-convex. Therefore, our attention shifts to its stationary points, leading to the following result.

\begin{theorem}\label{the: solution}
One of the stationary points of (\ref{equ: asy krr}) is
\begin{equation}
\begin{aligned}    
    &\vw^* = \psi(\vX)(\phi^\top(\vX)\psi(\vX)+\lambda \vI_N)^{-1}\vY, \qquad\qquad
    \vv^* = \phi(\vX)(\psi^\top(\vX)\phi(\vX)+\lambda \vI_N)^{-1}\vY.
\end{aligned}
\end{equation}
\end{theorem}
The proof is presented in Appendix~\ref{apdx: 0}.
Theorem~\ref{the: solution} establishes a crucial result, demonstrating that the stationary points can still be represented as a linear combination of function evaluations on the training dataset. This validates the practical feasibility of the proposed framework. 
Theorem~\ref{the: solution} indicates the proposed asymmetric KRR framework includes the symmetric one. That is, model~(\ref{equ: asy krr}) and (\ref{equ: skrr}) share the same stationary points when the two feature mappings are equivalent, as shown in the following corollary.

\begin{corollary}
If the two feature mappings $\phi$ and $\psi$ are equivalent, i.e. $\phi(\vx) = \psi(\vx), \forall \vx\in\mathbb{R}^M$, then stationary conditions of the asymmetric KRR model (\ref{equ: asy krr}) and the symmetric KRR model (\ref{equ: skrr}) are equivalent. And the stationary point is
$\vw^* =\vv^*=\phi(\vX)(\lambda\vI_N + \phi(\vX)^\top\phi(\vX))^{-1}\vY.$
\end{corollary}

With the conclusion in Theorem~\ref{the: solution}, we can easily apply asymmetric kernel trick $\mathcal{K}(\vt,\vx) = \langle \phi(\vt), \psi(\vx)\rangle$ and  obtain two regression functions.
By denoting a kernel matrix $[\vK(\vX,\vX)]_{ij}=\mathcal{K}(\vx_i,\vx_j)=\langle \phi(\vx_i), \psi(\vx_j)\rangle$, $\forall \vx_i,\vx_j \in \mathcal{X}$, we have:
\begin{equation}\label{equ: regressors}
    \begin{aligned}
        &f_1(\vt) = \phi(\vt)^\top\vw^* = \vK(\vt,\vX)(\vK(\vX,\vX)+\lambda\vI_N)^{-1}\vY,\\
        &f_2(\vt) = \psi(\vt)^\top\vv^* = \vK^\top(\vX,\vt)(\vK^\top(\vX,\vX)+\lambda\vI_N)^{-1}\vY.
    \end{aligned}
\end{equation}

The scenario of obtaining two regressors does not occur in a symmetric setting and the existance of the second regressor is often overlooked in prior works on asymmetric kernel regression. Consequently, the relationship between these regressors remains unclear. We discuss this question in Appendix~\ref{apdx: f explanation} from a primal-dual perspective. Our analysis reveals that the approximation error of these two regressors can be computed analytically. Notably, if $\vK(\vX,\vX)$ is asymmetric, the errors generally differ, leading to a significant observation: the two regressors represent distinct functions converging toward the ground truth from divergent directions.

When LAB RBF kernels are utilized, the computation of 
$\vK(\vX,\vt)$  necessitates a bandwidth that is dependent on the testing data  $\vt$. Given the impracticality of estimating bandwidths for testing data, we restrict our computations to $\vK(\vt,\vX)$. As a result, only  $f_1$  in Equation~(\ref{equ: regressors}) is applicable for our algorithm.
From Mercer's theorem we know that for traditional RBF kernels, there exists a feature mapping function $\phi(\cdot):\mathcal{X}\to \mathcal{F}$ satisfying that  $\mathcal{K}_\sigma(\vt,\vx) = \langle \phi_\sigma(\vt), \phi_\sigma(\vx)\rangle$.
Recall the definition of the proposed LAB RBF kernel function over dataset $\mathcal{X}$ in Equation~(\ref{equ: kernel function}), we can define
\begin{equation*}
\begin{aligned}
    &\phi(\vt) = [\phi^\top_{\theta_1}(\vt)\;\; \phi^\top_{\theta_2}(\vt)\;\; \cdots\;\;\phi^\top_{\theta_N}(\vt)]^\top,\\
    &\psi(\vx) = [\phi^\top_{\theta_1}(\vx)\delta(\vx-\vx_1)\; \;\phi^\top_{\theta_2}(\vx)\delta(\vx-\vx_2)\; \;\cdots\;\;\phi^\top_{\theta_N}(\vx)\delta(\vx-\vx_N)]^\top,
\end{aligned}
\end{equation*}
where $\theta_i$ is the corresponding bandwidth for data $\vx_i$ and $\delta(\cdot)$ denotes the Dirac delta function.
Then we can decompose the asymmetric LAB RBF kernels defined over a dataset $\mathcal{X}$ as the inner product of $\phi$ and $\psi$. That is, given dataset $\mathcal{X}$ and corresponding bandwidth set $\vTheta=\{\theta_1, \cdots, \theta_N\}$, the LAB RBF kernel defined over $\mathcal{X}$ and $\vTheta$ satisfies
\begin{equation}\label{equ: k-inner product}
    \mathcal{K}_\vTheta(\vt,\vx_i) 
    = \exp\{-\|\theta_i\odot(\vt-\vx_i)\|^2)\}= \langle \phi(\vt), \psi(\vx_i)\rangle,\qquad \forall \vt\in\mathbb{R}^d, \vx_i\in\mathcal{X}.
\end{equation}
Finally, substituting Equation~(\ref{equ: k-inner product}) into $f_1$, we obtain the solution of kernel ridgeless regression model with LAB RBF kernels by setting regularization coefficient $\lambda$ in (\ref{equ: regressors}) equals to zero: 
\begin{equation}\label{equ: LAB-interpolation}
\begin{aligned}
    f_{\mathcal{Z},\vTheta}(\vt) 
    &\triangleq \phi(\vt)^\top\vw  = \phi(\vt)^\top\psi(\vX)\valpha
    = \sum_{i=1}^N \alpha_i \exp\left\{- \|\theta_i\odot(\vt-\vx_i)\|_2^2\right\},\\
        \valpha &= \lim_{\lambda\rightarrow0}\left\{(\vK_\vTheta(\vX,\vX) +\lambda\vI_N)^{-1}\vY\right\},
\end{aligned}
\end{equation}
where $\mathcal{Z}\triangleq\{\mathcal{X},\mathcal{Y}\}$ and $[\vK_\vTheta(\vX,\vX)]_{ij} \triangleq \exp\left\{- \|\theta_j\odot(\vx_i-\vx_j)\|_2^2\right\}, \forall \vx_i,\vx_j\in\mathcal{X}$.

\subsection{LAB RBF Kernel Learning: Bandwidth Optimization}
The solutions presented in (\ref{equ: LAB-interpolation}) represent simple interpolation functions that may be susceptible to noise in the data. To enhance generalization ability, we employ kernel learning techniques to augment model flexibility and subsequently reduce model complexity.
It is essential to note that the solution in (\ref{equ: LAB-interpolation}) corresponds to a stationary point, necessitating additional data for optimizing the bandwidths $\vTheta$. Our algorithm thus divides the available data into two parts: (i) a subset of the available data serves as support data, used for constructing the regression function according to Equation~(\ref{equ: LAB-interpolation}), and (ii) the remaining data, termed training data, is utilized for the optimization of bandwidths.

Assume a support dataset $\mathcal{Z}_{sv}=\{\mathcal{X}_{sv},\mathcal{Y}_{sv}\}$ and a training dataset $\mathcal{Z}_{tr}=\{\mathcal{X}_{tr},\mathcal{Y}_{tr}\}$ are pre-given, then according to Equation~(\ref{equ: LAB-interpolation}), the optimization model for bandwidths $\vTheta$ is,
\begin{equation}\label{equ: interpolation f on sv}
\begin{aligned}
     \vTheta^* &= \arg\min_{\vTheta} \sum_{\{\vx_i,y_i\}\in\mathcal{Z}_{tr}
    }\left(y_i - f_{\mathcal{Z}_{sv},\vTheta}(\vx_i) \right)^2 \\
    &= \arg\min_{\vTheta} \sum_{\{\vx_i,y_i\}\in\mathcal{Z}_{tr}
    }\left(y_i - \vK_\vTheta(\vx_i,\vX_{sv})\vK^{-1}_\vTheta(\vX_{sv},\vX_{sv})\vY_{sv}\right)^2.
\end{aligned}
\end{equation}
Being a function that interpolates a small dataset without any regularization, it is apparent that the generalization performance of $f_{\mathcal{Z}_{sv},\vTheta}$ in Equation~\ref{equ: LAB-interpolation} does not meet expectations. Nevertheless, through the training of bandwidths $\vTheta$ with additional data, we can significantly enhance the generalization capacity of $f_{\mathcal{Z}_{sv},\vTheta}$.

\subsection{Dynamic Strategy: Support Data Optimization}
While in kernel methods we can always achieve perfect interpolation of training data when all data points are used as support data, the resulting interpolation function often lacks robust generalization ability. In our approach, integrating asymmetric kernel learning techniques, i.e., the optimization in (\ref{equ: interpolation f on sv}), enables us to achieve a good fit with fewer support data points. Drawing from traditional regularization scenarios, we aim to minimize the support data while effectively approximating the training data. However, this strategy introduces a discrete optimization problem, posing challenges for accurate solution finding.
\begin{equation}
    \begin{aligned}
        \mathcal{Z}_{sv}=\min_{\mathcal{Z}\subset\mathcal{Z}_{tr}} &|\mathcal{Z}|\\
        \mathrm{s.t.}\;\;& y_i = f_{\mathcal{Z},\vTheta}(\vx_i), \quad\forall \{\vx_i, y_i\}\in\mathcal{Z}_{tr},
    \end{aligned}
\end{equation}
where $|\mathcal{Z}|$ denotes the cardinality of set $\mathcal{Z}$, i.e. the number of data in $\mathcal{Z}$.

Though this discrete optimization presents challenges for direct optimization, numerous existing strategies for data selection can be employed. For instance, it resembles the selection of centers in Nystr\"om approximation (see \cite{williams2000using, rudi2017falkon} for details). Consequently, the subset selection strategies utilized in these existing works are applicable to the proposed algorithm. Additional experiments evaluating the performance of the proposed algorithm with various reasonable strategies are elaborated in Appendix~\ref{apdx: sv}.

In this paper, to facilitate theoretical analysis, we apply a \textit{dynamic strategy} for selecting support data. Initially, we uniformly select $N_0$ support data points according to their labels, and then: 
(i) Optimize (\ref{equ: interpolation f on sv}) to obtain $f_{\mathcal{Z}_{sv},\vTheta}$.
(ii) Compute approximation error $e_i=(f_{\mathcal{Z}_{sv},\vTheta}(\vx_i) - y_i)^2, \forall \{\vx_i,y_i\}\in\mathcal{Z}_{tr}$.
(iii) Add data with first $k$ largest error to support dataset. 
Repeat the above process until all approximation error is less than a pre-given threshold $B$.
The overall algorithm is presented in Algorithm~\ref{alg: AKL}.

By dynamic strategy, we actually obtain an important property of the resulting estimator, i.e.,
\begin{equation}\label{equ: strategy}
    (f_{\mathcal{Z}_{sv},\vTheta}(\vx_i) - y_i)^2\leq B,\qquad \forall \{\vx_i,y_i\}\in\mathcal{Z}_{sv}\cup\mathcal{Z}_{tr},
\end{equation}
which essentially stands the accuracy of the interpolation of $f_{\mathcal{Z}_{sv},\vTheta}$ on the training dataset.
And in the next section, we will shown it helps when analyzing the approximation behavior of LAB RBF kernels.

\begin{algorithm}[tb]
    \begin{algorithmic} [1]
    \caption{Learning LAB RBF kernels with SGD and dynamic strategy. }
    \label{alg: AKL}
    \State \textbf{Input:} Data $\mathcal{Z}=\{\mathcal{X},\mathcal{Y}\}$.  \State Initialization: Error tolerance $B>0$, initial bandwidth $\vTheta^{(0)}>0$, learning rate for gradient descent method $\eta>0$, $k$ for the dynamic strategy, and uniformly sampled support dataset $\mathcal{Z}_{sv}^{(0)}=\{\mathcal{X}_{sv}^{(0)},\mathcal{Y}_{sv}^{(0)}\}\subset\mathcal{Z}$.  
    \State t=0.
    \Repeat
        %\State Compute the function $ f_{\mathcal{Z}_{sv}^{(t)},\vTheta^{(t)}}$ according to (\ref{equ: kernel function}) and (\ref{equ: LAB-interpolation}).
        \State $\tilde{\vTheta}^{(0)} = \vTheta^{(t)}$.
        \For{$l = 1,\cdots, L$}
            \Comment{\textbf{Optimize $\vTheta$ via SGD}}
    	    \State Randomly sample a subset $\{\mathcal{X}_{s},\mathcal{Y}_{s}\}\subset\mathcal{Z}\setminus\mathcal{Z}_{sv}$.  
    	    \State Compute $\tilde{\vTheta}^{(l)}=\tilde{\vTheta}^{(l-1)}-\eta\frac{\partial}{\partial\vTheta}\|f_{\mathcal{Z}_{sv}^{(t)},\tilde{\vTheta}^{(l)}}(\vX_{s})-\vY_{s}\|^2$ according to (\ref{equ: interpolation f on sv}).
        \EndFor
        \State $ \vTheta^{(t+1)} = \tilde{\vTheta}^{(L)}$.
        \State Compute error $e_i=(f_{\mathcal{Z}_{sv}^{(t)},\vTheta^{(t)}}(\vx_i)-y_i)^2$ for all data $\{\vx_i,y_i\}\in\mathcal{Z}\setminus\mathcal{Z}_{sv}$.
        \If{$\max_i e_i\leq B$}
            \Comment{\textbf{Dynamically adding support data}}
        \State break.
        \Else
        \State Select the first $k$ samples with the highest errors and include them in the support dataset, resulting in $\mathcal{Z}_{sv}^{(t+1)}$.
        \EndIf           
        \State t=t+1.
	\Until{the maximal number of iteration is exceeded.}
    \State Compute the $\valpha = \vK^{-1}_{\vTheta^{(t)}}(\vX^{(t)}_{sv},\vX_{sv}^{(t)})\vY_{sv}^{(t)} $.   \Comment{\textbf{Compute the final function}}
    \State \textbf{Return} $\valpha,\mathcal{Z}^{(t)}_{sv}$ and $\vTheta^{(t)}$.
    \end{algorithmic}
\end{algorithm}

\section{Theoretical Interpretation}\label{sec: space}
\subsection{Enlarged hypothesis space: Integral Space of RKHSs }
To comprehend the learning dynamics of Algorithm~\ref{alg: AKL} and LAB RBF kernels, it is imperative to clarify the underlying function spaces. The LAB RBF kernel, defined in Equation~(\ref{equ: kernel function}), employs distinct bandwidths for individual samples, thus associating itself with multiple RKHSs. Unlike Multiple Kernel Learning (MKL, \cite{gonen2011multiple}), which explores a search space comprised of a finite number of RKHSs with a discrete domain of bandwidths (i.e., the kernel dictionary, as discussed in \cite{suzuki2011unifying}), LAB RBF kernels exhibit a continuous feasible domain of bandwidths. This characteristic results in a function space that surpasses a direct sum of RKHSs. To enhance the understanding of LAB RBF kernels, this paper introduces the concept of the \textit{integral space of RKHSs} as a novel hypothesis space.

Given a continuous bandwidth candidate set $\Omega\subset \mathbb{R}_+^M$, a traditional RBF kernel with a fixed uniform bandwidth $\vsigma\in\Omega$ has a form of $\mathcal{K}_\vsigma(\vx_i,\vx_j)=\exp\{-\|\vsigma\odot(\vx_i-\vx_j)\|_2^2\}$.
The RKHS introduced by $\mathcal{K}_\vsigma$ is denoted as $\mathcal{H}_\vsigma$. That is, $\forall \vx\in\mathcal{X}$, $\mathcal{K}_\vsigma(\cdot,\vx)\in\mathcal{H}_\vsigma$ and we use $f_\vsigma$ to denote functions belonging to $\mathcal{H}_\vsigma$.
Then the integral space of RKHSs defined over $\Omega$ takes the following form,
\begin{equation*}
    \begin{aligned}
    \mathcal{H}_{\Omega} &= \int_{\vsigma\in\Omega} \mathcal{H}_\vsigma d\mu(\vsigma) 
    = \left\{f=(f_\vsigma)_{\vsigma\in\Omega} :\left.\int_{\vsigma\in\Omega}\|f_\vsigma\|^2_{\mathcal{H}_\vsigma} d\mu(\vsigma)<\infty\right.\right\},
    \end{aligned}
\end{equation*}
where $(f_\vsigma)_{\vsigma\in\Omega}$ is a measurable cross-section and $\mu(\vsigma)$ denotes a probability distribution of $\vsigma$.
For more theoretical discussion of integral spaces of RKHSs, one can refer to \cite{Wils1970DirectIO,Hotz2012integrating}.
It has been proved that $\mathcal{H}_{\Omega}$ is again a Hilbert space, where the inner product between $f=(f_\vsigma)_{\vsigma\in\Omega} ,\;g=(g_\vsigma)_{\vsigma\in\Omega} \in \mathcal{H}_{\Omega}$ is defined as
\begin{equation*}
    \langle f,g\rangle_{\mathcal{H}_{\Omega}}
    := \int_{\vsigma\in\Omega} \langle f_\vsigma, g_\vsigma\rangle_{\mathcal{H}_\vsigma} d\mu(\vsigma).
\end{equation*}
Consequently, it holds that $f(\vx)=\int_{\vsigma\in\Omega} f_\vsigma(\vx) d\mu(\vsigma), \forall \vx\in\mathcal{X}.$
Then the corresponding norm is defined as
\begin{equation*}
\begin{aligned}
    &\|f\|^2_{\mathcal{H}_{\Omega}} := \min\left\{\int_{\vsigma\in\Omega} \|f_\vsigma\|^2_{\mathcal{H}_\vsigma} d\mu(\vsigma) : f=(f_\vsigma)_{\vsigma\in\Omega} \right\},
\end{aligned}
\end{equation*}
where $\|f_\vsigma\|^2_{\mathcal{H}_\vsigma} = \langle f_\vsigma, f_\vsigma\rangle_{\mathcal{H}_\vsigma}$.

Recall that in Algorithm~\ref{alg: AKL}, $\vTheta$ represents a discrete set. Consequently, the estimator $f_{\mathcal{Z}_{sv},\vTheta}$ is constructed from a finite number of kernels, thereby situating it within a sum space of RKHSs associated with these kernels. It is important to note that this inference relies on the assumption of a fixed $\vTheta$. When optimizing $\vTheta$, this sum space also changes according to the variations in bandwidths, as shown in Figure~\ref{fig: space}.
Mathematically, we assume a sum space of RKHSs generated from a bandwidth set $\vTheta$ is denoted as $\mathcal{H}_\vTheta$. Recall that in our approach, a continuous feasible domain of bandwidth is considered, i.e., $\vTheta\subset\Omega$.
Consequently, the hypothesis space involved in our approach is the union of all possible $\mathcal{H}_\vTheta$, i.e.
$$\mathrm{Hypothesis\; Space:} \bigcup_{\vTheta\subset\Omega}\mathcal{H}_\vTheta = \mathcal{H}_\Omega,$$
which indicates that the hypothesis space remains an integral space rather than a fixed sum space.

\begin{figure}[tb]
    \centering
  \begin{minipage}{0.5\textwidth}
    \centering
    \includegraphics[width=\textwidth]{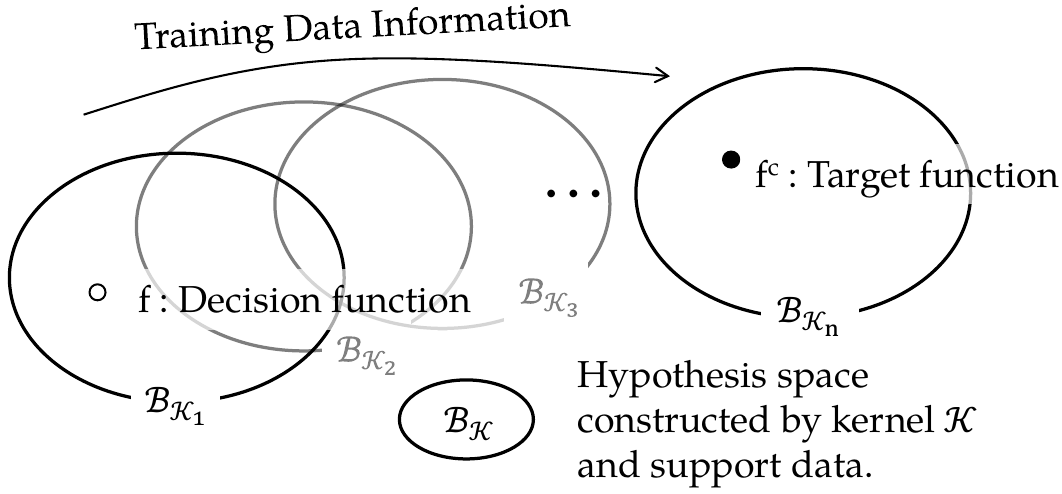}
    \caption{Optimal subspace selection when learning kernels.}
    \label{fig: space}
  \end{minipage}\hfill
  \begin{minipage}{0.45\textwidth}
    \centering    \includegraphics[width=0.6\textwidth]{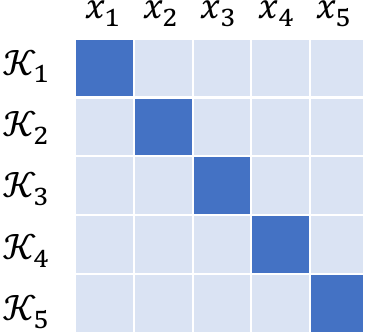}
    \caption{Coefficient matrix of $f_{\mathcal{Z},\vTheta}$, exhibiting sparse property.}
    \label{fig: coefficient}
  \end{minipage}
\end{figure}

\subsection{Sparsity of the Estimator}
With the established hypothesis space $\mathcal{H}_\Omega$, this section delineates the sparse property of the estimator generated by LAB RBF kernels. This characterization aids in our deeper comprehension of the generalization ability of the proposed model.
In Algorithm~\ref{alg: AKL}, two levels of sparsity are observed:
\begin{itemize}
    \item Reduced support data: The number of support data points is significantly lower than the total number of training data points. While this sparsity is artificially determined algorithmically, its essential reason lies in the sufficiently large hypothesis space. This expansive hypothesis space enables us to employ fewer support data points to effectively approximate the entire training dataset. 
    This sparsity leads to a fact that $f_{\mathcal{Z}_{sv},\vTheta}$ belongs to a small subspace $\mathcal{H}_\vTheta$ of $\mathcal{H}_\Omega$.
    \item Inherent sparsity of LAB RBF kernels: $f_{\mathcal{Z}_{sv},\vTheta}$ demonstrate sparsity within the hypothesis space $\mathcal{H}_\vTheta$, contributing to a more efficient representation of the data.
\end{itemize}
In the following, we show the latter sparsity of $f_{\mathcal{Z}_{sv},\vTheta}$ mathematically by comparing with general function in $\mathcal{H}_\vTheta$.
Given dataset $\{\vx_i,y_i\}_{i=1}^N$, the hypothesis space considered here is taken to be the linear span of the set $ \{\mathcal{K}_\vsigma(\cdot, \vx_i)\}$, $\forall i=1,\cdots,N,\;\forall \vsigma\in\Omega$.  This space forms a subspace of $\mathcal{H}_\Omega$.
In $f_{\mathcal{Z}_{sv},\vTheta}$, only bandwidths in $\vTheta$ are valid, therefore we constrain ${\mu}(\sigma) = \sum_{\theta_i\in\vTheta} \delta(\sigma - \theta_i)$.
Then function in this hypothesis space takes a formulation as
$f(\cdot) = \sum_{\sigma\in\Theta} f_\sigma(\cdot) = \sum_{\sigma\in\Theta} \sum_{i=1}^{N_{sv}}\alpha_{\sigma,i}\mathcal{K}_\sigma(\cdot,{\bf x}_i),$
where the coefficients $\valpha_\vsigma\in\mathbb{R}^{{N_{sv}}}$.

Let $\|\valpha\|_0=\sum_{i=1}^N\mathbb{I}(\alpha_i\neq 0)$, where $\mathbb{I}$ is a indicator function.
We can define a $\ell_0$-related sparse regularization penalty as below,

\begin{equation}\label{equ: def R0}
\begin{aligned}
\mathcal{R}_0(f)&:=\min\left\{ \int_{\vsigma
\in \Omega} \|\valpha_\vsigma\|_0 d\mu(\vsigma): f=(f_\vsigma)_{\vsigma\in\Omega},  f_\vsigma(\cdot) =\sum_{i=1}^N\alpha_{\vsigma,i}\mathcal{K}_\vsigma(\cdot,\vx_i) \right\}.
\end{aligned}
\end{equation}
Without sparsity, a general function in $\mathcal{H}_\vTheta$ typically results in  $\mathcal{R}_0(f)$ being approximately equal to  $|\Theta|\times {N_{sv}}$.
However, recall the function estimated by LAB RBF kernels in  (\ref{equ: LAB-interpolation}), generated from the same kernels and data, takes a formulation like:
$$f_{\mathcal{Z}_{sv},\vTheta}(\cdot) = \sum_{i=1}^{N_{sv}} \hat{\alpha}_{i} \mathcal{K}_{\theta_i}(\cdot, x_i) = \sum_{\sigma\in\Theta} \sum_{i=1}^{N_{sv}} \alpha_{\sigma, i} \mathcal{K}_\sigma(\cdot, x_i).$$ 
Comparing their coefficients, we say $f_{\mathcal{Z}_{sv},\vTheta}$ exhibits sparsity because 
$$\alpha_{\sigma, i}=\begin{cases}
\hat{\alpha}_i,\;&\mathrm{if}\;\sigma=\theta_i,\\
    0,\;&\mathrm{otherwise}.
    \end{cases}$$
This sparsity can be quantified by the measurement $\mathcal{R}_0(f_{\mathcal{Z}{sv},\vTheta}) = N_{sv}$, or visually depicted in Figure~\ref{fig: coefficient}.
In this regard, $f_{\mathcal{Z}_{sv},\vTheta}$ demonstrates enhanced sparsity compared to a typical function within the sum space $\mathcal{H}_{\Theta}$, not to mention functions within the integral space $\mathcal{H}_\Omega$. 

\subsection{Equivalence to a $\ell_0$-related Model}
Utilizing this sparse property, we can gain deeper insights into $f_{\mathcal{Z}_{sv},\vTheta}$ by formulating a sparse optimization model, leveraging the $\mathcal{R}_0$ regularization term.
Let us define the empirical error $\mathcal{E}_z(f)$ and the generalization error $\mathcal{E}_\rho(f)$ as follows,
\begin{equation*}
    \mathcal{E}_\vz(f)  = \frac{1}{N}\sum_{{\vx_i,y_i}\in\mathcal{Z}_{tr}}(f(\vx_i)-y_i)^2,
\quad\quad\quad
    \mathcal{E}_\rho(f) = \int_{\mathcal{Z}}(f(\vx)-y)^2 d\rho.
\end{equation*}
Then we have
\begin{equation}\label{equ: unconstrained}
    \begin{aligned}     &f_{\vz,\lambda} = \arg\min_{\substack{f\in\mathcal{H}_\Omega\\ \{\vtheta_i\}\subset\Omega}}
    \mathcal{E}_\vz(f) + \lambda \mathcal{R}_0(f)\\
    &\;\;\;
    \mathrm{s.t.} \quad {\mu}(\vsigma) = \sum_{i=1}^{N_{sv}} \delta(\vsigma - \theta_i),
    \end{aligned}
\end{equation}
where $\lambda, N_{sv}>0$ are pre-given parameters, and $\delta(\cdot)$ denotes the Dirac delta function. 

Optimizations involving $\ell_0$ norm are generally challenging to solve directly as they often give rise to an NP-hard discrete optimization \citep{natarajan1995sparse}, and the problem in (\ref{equ: unconstrained}) is no exception. However, as demonstrated by the following proposition, Algorithm~\ref{alg: AKL} for learning LAB RBF kernel yields an estimator that closely approximates the optimal solution of (\ref{equ: unconstrained}).

\begin{proposition}\label{lem: alg err}
Let $f_{\mathcal{Z}_{sv},\vTheta}$ denotes the regressor produced by Algorithm~\ref{alg: AKL} with dynamics strategy (i.e., $f_{\mathcal{Z}_{sv},\vTheta}$ satisfies (\ref{equ: strategy})) on dataset $\mathcal{Z}=\{\vx_i,y_i\}_{i=1}^N$. Then there exists a $\lambda>0$ such that the optimal solution $f_{\vz,\lambda}$ of (\ref{equ: unconstrained}) satisfies that 
$\mathcal{R}_0(f_{\mathcal{Z}_{sv},\vTheta}) = \mathcal{R}_0({f}_{\vz,\lambda})$, and 
$0<\mathcal{E}_\vz(f_{\mathcal{Z}_{sv},\vTheta})-\mathcal{E}_\vz({f}_{\vz,\lambda})\leq B.$
    
\end{proposition}
The proof is presented in Appendix~\ref{apdx: 5}. 
This proposition shows that  $f_{\mathcal{Z}_{sv},\vTheta}$ is a $B$-optimal solution of model (\ref{equ: unconstrained}) with some $\lambda$. It establishes a link between a well-trained function derived by kernel ridgeless regression with a LAB RBF kernel — exhibiting good interpolation performance on the training dataset — and the optimal solution of an $\ell_0$-related model within the integral space of RKHSs. 
This relationship effectively highlights the superiority of our method's strategy for enhancing model flexibility through the learning of the LAB kernel function. Specifically,
\begin{itemize}
    \item The model's enhanced flexibility is primarily achieved through the expansion of the hypothesis space, where the estimator is optimized from an integral space of RKHSs $\mathcal{H}_{\Omega}$. This expansion is enabled by optimizing ${\vtheta_i}$, which selects the optimal subspace from $\mathcal{H}_{\Omega}$, as illustrated in Figure~\ref{fig: space}. Consequently, the algorithm efficiently minimizes the distance to the underlying function, leading to a small bias.
    \item The large capacity of hypothesis space also raises the probability of interpolating training data with fewer support data, evidenced by the sparse coefficients of $f_{\mathcal{Z}_{sv},\vTheta}$. This sparsity characteristic clarifies the origin of the model's generalization capability: with an implicit $\ell_0$-related term in effect, controlled by the number of support data,  our algorithm effectively reduces variance.
\end{itemize}

It is worth noting that, despite the absence of a regularization term in the kernel ridgeless regression framework, our approach effectively maintains a balance between bias and variance through our dynamic strategy.
From model (\ref{equ: unconstrained}) and Proposition~\ref{lem: alg err}, it is determined that a smaller number of support data implies a stronger regularization effect. Note that kernel machines can always interpolate all training data; that is, the value of 
$B$ can be arbitrarily close to 0 as the number of support data increases.
Therefore, our proposed dynamic strategy proves effective as it seeks a balance between the number of support data and the empirical approximation error. The hyper-parameter $B$, which varies with datasets, essentially serves as a trade-off parameter between bias and variance, resembling most regularization schemes.

\section{Approximation Analysis}\label{Chapt:Analysis}
In the preceding sections, we introduced our kernel learning algorithm and corresponding theoretical explanation. The extensive capacity of the integral space of RKHSs and the utilization of the $\ell_0$ norm contribute to the exceptional performance of LAB RBF kernels, as evidenced by the experiments in Section \ref{sec: exp}. Nevertheless, these characteristics also pose significant challenges for approximation analysis.
To derive the learning rate of ${f}_{\vz, \lambda}$, this section employs three key techniques:
\begin{itemize}\setlength\itemsep{0em}
    \item Addressing the discrete nature of the $\ell_0$ norm, we use the optimal estimator of a $\ell_q\;(0<q<1)$-regularized model as a stepping-stone function.
    \item The Rademacher chaos complexity is employed to establish the upper bound of the sample error in the integral space of RKHSs, leveraging the properties of the optimal solution ${f}_{\vz, \lambda}$. 
    \item  A refined iteration technique is applied to obtain an accurate upper bound of $\mathcal{R}_0({f}_{\vz, \lambda})$.
\end{itemize}
The main result is presented in Theorem~\ref{the: main result}.

\subsection{Assumptions and Main Result}

We prepare some notations and assumptions for the following analysis.
Let $\rho$ be a Borel probability distribution on $\mathcal{Z}$.
Then from Proposition 1.8 in \cite{Cucker2007Learning}, the target \textit{regression function } can be expressed as
$
    f_\rho= \int_\mathcal{Y}y d\rho(y|\vx),\; \vx\in\mathcal{X},
$
where $\rho(\cdot|\vx)$ is the conditional probability measure induced by $\rho$ at $\vx$.
Thoughout this paper, we assume that $f_\rho$ belongs to a Sobelve space $\mathcal{H}^s(\mathbb{R}^d)$ with some $s>0$, and $\rho(\cdot|\vx)$ is support on $[-M,M]$. That is,
\begin{assumption}\label{asmp: gt bound}
For some constant $1\leq M<\infty$, there hold $|f_{\rho}(\vx)|\leq M$ and $|y|\leq M$.
\end{assumption}
Such uniformly boundedness assumptions of the output has been widely used in learning theory e.g., \cite{Zhou2003CapacityOR, Wu2006LearningRO, Smale2007LearningTE}.
And it also indicates a bounded noise level that $|y-f_\rho(\vx)|\leq 2M$.
Based on this assumption, we can apply the following projection operator to our analysis for better estimates.

\begin{definition}
For $M>0$, the projection operator $\pi_M$ is defined as
\begin{equation*}
\begin{aligned}
\pi_M(t)=\begin{cases}
	-M &\mathrm{if\;} t\leq -M,\\
	t &\mathrm{if\;} -M< t\leq M,\\
	M &\mathrm{if\;} t> M.\\
	\end{cases}
\end{aligned}
\end{equation*}
The projection of a function $f:X\to \mathbb{R}$ is defined by $\pi_M(f)(\vx)=\pi_M(f(\vx)), \forall \vx \in \mathcal{X}.$
\end{definition}
Such projection operator is introduced in \cite{Chen2004Support} and is helpful in estimating the $\|\cdot\|_\infty$ bound in the following analysis. %, and is also widely applied in the literature of learning theory, e.g., \cite{}
Under Assumption~\ref{asmp: gt bound}, it is natural to project the estimator $f$ into the same interval as $f_\rho$.
Thus, we shall consider the error $\|\pi_M(f)-f_\rho\|_{\mathcal{L}^2_{\rho_\vX}}$.

In the previous section, we point out that LAB RBF kernels are actually thed combination of RBF kernels with trainable bandwidths belonging to a pre-given closed interval $\Omega$.
Theoretical properties of RBF kernels have been well investigated before. One can refer to \cite{ye2008learning,eberts2011optimal} for RBF kernels with fixed bandwidths, and \cite{ying2007learnability} for RBF kernels with flexible bandwidths.
Consider $\mathcal{K}_\sigma(\vx,\vx') = \exp\{-\sigma\|\vx-\vx'\|^2_2\}, \forall \sigma\in\Omega.$
It has been proved that $\mathcal{K}_\sigma\in\mathcal{C}^\infty(\mathcal{X}\times \mathcal{X})$ and there exists a bound 
$ \|\mathcal{K}_\sigma\|_{C^\infty(\mathcal{X}\times \mathcal{X})} <\infty,\;\; \forall \sigma\in\Omega.$
Therefore we can define
$$\kappa := \sup_{\sigma\in\Omega} \|\mathcal{K}_\sigma\|_{C^\infty(\mathcal{X}\times \mathcal{X})}<\infty.$$
Given a continuous kernel $\mathcal{K}_\sigma$, it can define an integral operator on $\mathcal{L}_{\rho_\vX}^2(\mathcal{X})$ as follows
\begin{equation}\label{equ: int oper}
    \mathcal{L}_{\mathcal{K}} f(\vx) = \int_\mathcal{X} \mathcal{K}_{\sigma^*}(\vx,\vt)f(\vt)d\rho_\vX(\vt),\quad \vx\in\mathcal{X},\quad \forall f\in\mathcal{L}_{\rho_\vX}^2(\mathcal{X}).
\end{equation}
And a Mercer kernel can be defined as $
    \tilde{\mathcal{K}}_{\sigma}(\vt,\vx) = \int_{\mathcal{X}}\mathcal{K}_\sigma(\vt,\vu)\mathcal{K}_\sigma(\vx,\vu)d\rho_\vX(\vu).
$
In this paper, we use  the RKHS $\mathcal{H}_{\tilde{K}_\sigma}$ satisfying $\sigma\in\Omega$ to approximate $f_\rho$.
Following the definition in \cite{ying2007learnability}, the regularization error associated with a flexible $\mathcal{H}_{\tilde{K}_\sigma}$ is defined as follows,
\begin{equation} \label{equ: D-gamma}
    \mathcal{D}(\gamma) = \inf_{\sigma\in\Omega}\inf_{f\in\mathcal{H}_{\tilde{K}_\sigma}}\left\{\mathcal{E}_\rho(f) - \mathcal{E}_\rho(f_\rho) + \gamma\|f\|_{\mathcal{H}_{\tilde{K}_\sigma}}\right\}.
\end{equation}
As $\gamma\rightarrow 0$, the decay rate of $\mathcal{D}(\gamma)$ measures the lower bound of the approximation ability of the hypothesis space, which in the literature of learning theory are generally assumed as follows (e.g.,  \cite{Cucker2007Learning,steinwart2008support}).
\begin{assumption}\label{asmp: D}
For some constant $0<\beta\leq 1$ and $c_{\beta} \geq 1$, it holds
\begin{equation*}
    \mathcal{D}(\gamma)\leq c_{\beta}\gamma^{\beta},\;\;\forall \gamma>0.
\end{equation*}
\end{assumption}

Then we present the main result.

\begin{theorem}\label{the: main result}
Assume the regression $f_\rho\in\mathcal{H}^s(\mathcal{X})$ with some $s>0$.
Given $\Omega$, suppose Assumption~\ref{asmp: gt bound} and Assumption~\ref{asmp: D} hold with $M\in[1,+\infty)$ and $\beta\in(0,1]$. 
If $\xi\in(0,\frac{\beta}{4})$ (which can be arbitrarily small), $\lambda = (\frac{\xi}{\kappa^2})^\frac{1-\xi}{1+\xi} N^{-\tilde{\tau}}$ with $\tilde{\tau} = \beta/2(1+\xi)(\xi+\beta-\xi\beta)$ and $ \delta\in(0,1)$, then with at least $1-\delta$ confidence, it holds that

\begin{equation*}
    \begin{aligned}
    \|\pi_M(f_{\vz,\lambda})&-f_\rho\|_{L_{\rho_\vX}}
    \leq \tilde{C_3}\log^2\left(\frac{36}{\delta}\right)\left(\log\left(\frac{72}{\delta}\right)+\log(J(d,s,\beta,\xi)+1)\right)^{2J(d,s,\beta,\xi)}N^{-(\frac{\beta}{4}-\xi)},
    \end{aligned}
\end{equation*}
and 
\begin{equation*}
    \mathcal{R}_0(f_{\vz,\lambda}) \leq C_3(\log J(d,s,\beta,\xi))^{3}\left(\log\left(\frac{24}{\delta}\right) + \log(J(d,s,\beta,\xi)+1)\right)^{2J(d,s,\beta,\xi)} N^{\tilde{\tau}+\xi-\frac{\beta}{4}},
\end{equation*}
where $C_3$ and $\tilde{C_3}$ are some positive constant independent of $N$ and $\delta$ and  $J(d,s,\beta,\xi)$ is defined as
\begin{equation*}
    J(d,s,\beta,\xi) = \max\left\{\frac{\log\frac{2s\beta+2s\xi(1-\beta)-2d\beta}{6s\beta+8s\xi(1-\beta)-d\beta}}{\log\frac{d+2s}{2d}},\frac{\beta}{(1-\beta)\xi^2+\xi}\right\}.
\end{equation*}

\end{theorem}

The convergence rate of Algorithm~(\ref{equ: unconstrained}) concerning the accuracy, as well as the model complexity, with respect to the data number is provided by Theorem~\ref{the: main result}. This pioneering analysis in the integral space of RKHSs unveils valuable insights.
The convergence rate can be arbitrarily close to $N^{-\beta/4}$. The sparsity of $f_{\vz, \lambda}$, determined by $\mathcal{R}_0(f_{\vz,\lambda})/N$, demonstrates a superior convergence rate compared to $O\left(N^{-\left(\frac{1}{2}+\frac{\beta}{4}-\xi\right)}\right)$, thanks to the constraint $\tilde{\tau}\leq \frac{1}{2}$. Importantly, $\mathcal{R}_0(f_{\vz,\lambda})$ also acts as an upper bound for the number of valid bandwidths, facilitating the practical selection of support data.
In the literature of approximation analysis, the specific value of $\beta$ in Assumption~\ref{asmp: D} is determined by certain assumptions regarding the relationship between the underlying function $f_\rho$ and the hypothesis spaces.
Here we suppose $f_\rho\in\mathcal{H}^s(\mathcal{X})$ for some $s>0$. According to Proposition 22 in \cite{ying2007learnability}, Assumption~\ref{asmp: D} holds true. However, we refrain from introducing additional assumptions for $s$ to ascertain the specific value of $\beta$. Further details on the value of $\beta$ can be found in \cite{ying2007learnability}.

\subsection{Framework of Convergence Analysis}
From \cite{Cucker2007Learning}, it holds that $\|f-f_\rho\|^2_{\mathcal{L}_{\rho_\vX}} = \mathcal{E}_\rho(f)-\mathcal{E}_\rho(f_\rho), \forall f:\mathcal{X}\rightarrow\mathbb{R}$.
Thus Theorem~\ref{the: main result} can be obtained by estimating the upper bound of $\mathcal{E}_\rho(f)-\mathcal{E}_\rho(f_\rho)$.
To establish our main result, we employ the well-established framework of error decomposition commonly applied in kernel-based regression with regularization schemes (e.g., \citet{Cucker2007Learning, shi2019sparse, mao2023approximating}).
In this context, the proof sketch proceeds through several key steps.
Firstly, we introduce the error decomposition framework and review pertinent results from previous work.
In Section 4.3, we delve into presenting the sample error analysis. Additionally, the upper bound of $\mathcal{R}_0(f{\vz,\lambda})$ is discussed in Section 4.4.
Finally, consolidating all the findings, we present the proof of Theorem~\ref{the: main result} in Section 4.5.

To facilitate the error decomposition, we require some stepping-stone functions. We designate the minimizer of the regularization error in (\ref{equ: D-gamma}) as the regularization function, denoted by $f_\gamma$. The corresponding kernel is denoted as $\mathcal{K}_{\sigma^*}$ with a bandwidth $\sigma^*$.
From \cite{Cucker2007Learning}, it holds that $\mathcal{L}_{\mathcal{K}}$ and its adjoint $\mathcal{L}^*_{\mathcal{K}}$ are compact operators because we assume $\mathcal{X}$ is compact.
Recall the definition of integral operator $\mathcal{L}_\mathcal{K}$ and the Mercer kernel $
\tilde{\mathcal{K}}$ in (\ref{equ: int oper}) , then $f_\gamma$ can be explicitly given by $f_\gamma = (\gamma \vI+\mathcal{L}_{\tilde{\mathcal{K}}_{\sigma^*}})^{-1}\mathcal{L}_{\tilde{\mathcal{K}}_{\sigma^*}}f_\rho$.
We additionally define $f_{\vz,\gamma}(\cdot) = \frac{1}{N}\sum_{i=1}^N \mathcal{K}_{\sigma^*}(\cdot,\vx_i)g_\gamma(\vx_i)$ where $g_\gamma = \mathcal{L}^*_{{\mathcal{K}}}(\gamma\vI+\mathcal{L}_{\tilde{\mathcal{K}}_{\sigma^*}})^{-1}f_\rho$.
Finally, to deal with the $\ell_0$ regularization term, we introduce the $\ell_q$-regularized learning model as shown below,
\begin{equation}\label{equ: f lq}
    \begin{aligned}
    {f}_{\vz,\gamma}^q= \arg\min_{f\in{\mathcal{H}_{\mathcal{K}_{\sigma^*}}}}
    \mathcal{E}_\vz(f) + \gamma \|f\|^q_q, \qquad \forall q\in(0,1),
    \end{aligned}
\end{equation}
where $\|f\|_q \triangleq (\sum_{i=1}^N|\alpha_i|^q)^{1/q}$. And we use $f_{\vz,\gamma}^1$ denotes the optimal function when $q=1$.

With the above stepping-stone functions, we establish the following error decomposition to prove Theorem~\ref{the: main result}.
\begin{equation} \label{equ: err decomposition}
    \begin{aligned}
    &\mathcal{E}_\rho(\pi_M(f_{\vz,\lambda})) - \mathcal{E}_\rho(f_\rho) +\lambda \mathcal{R}_0(f_{\vz,\lambda})
    \leq\mathcal{S}_1 +\mathcal{S}_2 +\gamma N^{1-q}\|f^1_{\vz,\gamma}\|_1^q,
    \end{aligned}
\end{equation}
where
\begin{equation*}
    \begin{aligned}
    &\mathcal{S}_1 = \{\mathcal{E}_\rho(\pi_M(f_{\vz,\lambda}))-\mathcal{E}_\vz(\pi_M(f_{\vz,\lambda}))\} + \{\mathcal{E}_\vz(f_{\vz,\gamma})-\mathcal{E}_\rho(f_{\vz,\gamma})\},\\
    &\mathcal{S}_2= \mathcal{E}_\rho(f_{\vz,\gamma}) + \gamma N^{-1}\sum_{i=1}^N|g_\gamma(\vx_i)| -\mathcal{E}_\rho({f}_{\rho}).
    \end{aligned}
\end{equation*}
Recall that $\ell_0$ norm is involved in $\mathcal{R}_0$.
To associate $\ell_0$-related models with existing results, we need some useful conclusions in $\ell_q$-related models, which focuses on the non-zero coefficient of the global minimizier of $\ell_q$-regularized kernel regression.

\begin{lemma}\label{lem: lower bound of lq}
    (\cite{shi2019sparse}, Proposition 18) Let $f^q_{\vz,\gamma}(\cdot) = \sum_{i=1}^N K_{\sigma^*}(\cdot, \vx_i) (\valpha^q_{\vz,\gamma})_i$ be the global optimal solution of problem (\ref{equ: f lq}) with $0<q<1$. Then for $i\in(1,\cdots, N) $ and $(\valpha^q_{\vz,\gamma})_i\neq 0$, it holds that
    \begin{equation*}
        |(\valpha^q_{\vz,\gamma})_i|\geq \left(\frac{1-q}{\kappa^2}\right)^{1/(2-q)}\gamma^{1/(2-q)}.
    \end{equation*}
\end{lemma}
According to Lemma~\ref{lem: lower bound of lq}, the following lemma can be obtained.

\begin{lemma}\label{lem: hypothesis error}
Let $f_{\vz,\lambda}$ be the optimal of (\ref{equ: unconstrained}) and $f_{\vz,\gamma}(\cdot) = \sum_{i=1}^N \mathcal{K}_{\sigma^*}(\cdot,\vx_i)g_\gamma(\vx_i)$ with  $g_\gamma = \mathcal{L}^*_\mathcal{K}(\gamma\vI+\mathcal{L}_{\sigma^*})f_\rho$.
For some $0<q<1$, assume $\gamma$ is carefully selected according to $\lambda $ and $q$, such that $\lambda=(\frac{1-q}{\kappa^2})^\frac{q}{2-q}\gamma^\frac{2}{2-q}$.
Then it holds that
\begin{equation*}
    \mathcal{E}_\vz(f_{\vz,\lambda})  +\lambda \mathcal{R}_0(f_{\vz,\lambda})
    \leq \mathcal{E}_\vz(f_{\vz,\gamma}) + \gamma N^{-1}\sum_{i=1}^N|g_\gamma(\vx_i)| + \gamma N^{1-q}\|f^1_{\vz,\gamma}\|_1^q.
\end{equation*}
\end{lemma}

\begin{proof}
From Lemma~\ref{lem: lower bound of lq} and the definition of $\mathcal{R}_0$ we know that
\begin{equation}\label{equ: RO-Rq}
    \begin{aligned}
        \mathcal{R}_0({f}^q_{\vz,\gamma})\leq
        \left(\frac{\kappa^2}{\gamma (1-q)}\right)^{\frac{q}{2-q}} \|f_{\vz,\gamma}\|_q^q.
    \end{aligned}
\end{equation}
And then we have
\begin{equation}
    \begin{aligned}
     \mathcal{E}_\vz({f}_{\vz,\lambda})
    +\lambda\mathcal{R}_0({f}_{\vz,\lambda})
    &\overset{(a)}\leq\mathcal{E}_\vz({f}^q_{\vz,\gamma})    +\lambda\mathcal{R}_0({f}^q_{\vz,\gamma})\\
    &\overset{(\ref{equ: RO-Rq})}\leq \mathcal{E}_\vz({f}^q_{\vz,\gamma})+\gamma  \|f^q_{\vz,\gamma}\|_q^q\\
    &\leq \mathcal{E}_\vz({f}^1_{\vz,\gamma})+\gamma \|f^1_{\vz,\gamma}\|_q^q + \gamma \|f^1_{\vz,\gamma}\|_1\\
   &\overset{(b)}\leq \mathcal{E}_\vz(f_{\vz,\gamma})+ \gamma N^{-1}\sum_{i=1}^N|g_\gamma(\vx_i)| + \gamma   \|f^1_{\vz,\gamma}\|_q^q ,
    \end{aligned}
\end{equation}
where (a), (b), and (c) use the optimal property of ${f}_{\vz,\lambda}$, ${f}^q_{\vz,\gamma}$, and ${f}^1_{\vz,\gamma}$, respectively.
By reverse Holder inequality it holds that
$
   \gamma \|f^1_{\vz,\gamma}\|_q^q \leq \gamma  N^{1-q} \|f^1_{\vz,\gamma}\|_1^q.$
Combine all above inequalities together, it yields the result in Lemma~\ref{lem: hypothesis error} and we complete the proof.
\end{proof}

Lemma~\ref{lem: hypothesis error} bridges ${f}_{\vz,\lambda}$ and ${f}_{\vz,\gamma}$ via the sparse property of ${f}^q_{\vz,\gamma}$, supporting the proof of (\ref{equ: err decomposition}).
Here we are at the stage of proofing (\ref{equ: err decomposition}).
\begin{proof}
By a direct decomposition we have
\begin{equation*}
    \begin{aligned}
    &\mathcal{E}_\rho(\pi_M(f_{\vz,\lambda})) - \mathcal{E}_\rho(f_\rho) +\lambda \mathcal{R}_0(f_{\vz,\lambda})\\
    =&\{\mathcal{E}_\rho(\pi_M(f_{\vz,\lambda}))-\mathcal{E}_\vz(\pi_M(f_{\vz,\lambda}))\} +\{\mathcal{E}_\vz(\pi_M({f}_{\vz,\lambda}))-\mathcal{E}_\vz({f}_{\vz,\lambda})\} \\
    +&\left\{\mathcal{E}_\vz(f_{\vz,\lambda})  +\lambda \mathcal{R}_0(f_{\vz,\lambda})
    - \mathcal{E}_\vz(f_{\vz,\gamma}) - \gamma N^{-1}\sum_{i=1}^N|g_\gamma(\vx_i)| \right\} \\
    +&\{\mathcal{E}_\vz(f_{\vz,\gamma})-\mathcal{E}_\rho(f_{\vz,\gamma})\} +\left\{\mathcal{E}_\rho(f_{\vz,\gamma}) + \gamma N^{-1}\sum_{i=1}^N|g_\gamma(\vx_i)| -\mathcal{E}_\rho({f}_{\rho})\right\}.
    \end{aligned}
\end{equation*}
From Assumption~\ref{asmp: gt bound} and the definition of the projection operator we know that $\mathcal{E}_\vz(\pi_M(f_{\vz,\lambda}))   \leq\mathcal{E}_\vz(f_{\vz,\lambda}) $.
Therefore, the second term and the last second term are at most zero.
From Lemma~\ref{lem: hypothesis error} we know that the third term is less than $\gamma N^{1-q}\|f^1_{\vz,\gamma}\|_1^q$.
Then we get the result in (\ref{equ: err decomposition}) and  complete the proof.
\end{proof}

According to (\ref{equ: err decomposition}), one can estimate the total error by analysing the upper bound of  $\mathcal{S}_1,\;\mathcal{S}_2$ and $\|f^1_{\vz,\gamma}\|_1^q$. Here $\mathcal{S}_1$ consists of the sample error of $\pi_M(f_{\vz,\lambda})$ and $f_{\vz,\gamma}$, which is associated with the complexity of hypothesis spaces.
And $\mathcal{S}_2$ is the convergence rate of $f_{\vz,\gamma}$ to the regression function ${f}_{\rho}$ under the $\ell_1$ constraint.
The asymptotic behavior of $f_{\vz,\gamma}$ has been well investigated previously in previous works like \cite{Guo2013Learning,shi2019sparse}.
Here we directly quote the following result.

\begin{lemma}\label{lem: s2+s3}
For any $(\gamma, \delta) \in (0,1)^2$, it holds with confidence $1-\delta$ that
\begin{equation*}
    \begin{aligned}
    \mathcal{S}_2
    &\leq 8\kappa^2(2\kappa^2+1)\log^2\left(\frac{4}{\delta}\right)\left\{\frac{\mathcal{D}(\gamma)}{\gamma^2 N^2}+\frac{\mathcal{D}(\gamma)}{\gamma N}\right\} 
     + \frac{2\kappa+1}{N}\sqrt{\mathcal{D}(\gamma)}\log\left(\frac{4}{\delta}\right) + \frac{3}{2}\sqrt{\gamma\mathcal{D}(\gamma)} + 2 {\mathcal{D}(\gamma)}.
    \end{aligned}
\end{equation*}
\end{lemma}

The proof of this lemma can be found in the Lemma~1 and Proposition~4 in \cite{Guo2013Learning}.
Then in the following section we focus on the analysis of $\mathcal{S}_1$ and we give the proof of Lemma~\ref{lem: hypothesis error}.

\subsection{Sample Error Estimation via Rademacher Chaos Complexity}

Generally, the sample error is often guaranteed by the uniformly concentration inequality and the capacity assumption related to the functional space (cf. \cite{shi2013learning}). 
However, the commonly employed capacity assumptions in kernel learning or multi-kernel learning become invalid for $\mathcal{H}_{\Omega}$ as it constitutes an integral space of infinite spaces. In this section, we leverage the sparse property of $f_{\vz,\lambda}$ and the Rademacher chaos complexities to estimate the upper bound of $\mathcal{E}_\rho(\pi_M(f))-\mathcal{E}_z(\pi_M(f))$.
\begin{definition} 
Let $\mathcal{F}$ be a class of functions mapping from $\mathcal{X}\times \mathcal{X}$ to $\mathcal{R}$. Let $\vx_1,\cdots,\vx_N$ be $N$ independent and identically distributed (i.i.d.) samples.
The homogeneous Rademacher chaos process of order 2, with respect to i.i.d. Rademacher variables $\epsilon_1,\cdots,\epsilon_N$, is a random variable system defined by 
$$\mathcal{U}_f(\epsilon)=\frac{1}{N}\left[\sup_{f\in\mathcal{F}}\left|\frac{1}{N}\sum_{i,j\in\mathbb{N}_N,i<j}\epsilon_i \epsilon_j f(\vx_i,\vx_j)\right|\right],\;\; f\in\mathcal{F}.$$
Then the empirical Rademacher chaos complexities over $\mathcal{F}$ is defined as the expectation of its suprema. That is,
\begin{equation*}
    \begin{aligned}
        \mathcal{U}_N(\mathcal{F}) 
        &= \mathbb{E}_\epsilon\left[\sup_{f\in\mathcal{F}}\left|\mathcal{U}_f(\epsilon)\right|\right]
        = \mathbb{E}_\epsilon\left[\sup_{f\in\mathcal{F}}\left|\frac{1}{N}\sum_{i,j\in\mathbb{N}_N,i<j}\epsilon_i \epsilon_j f(\vx_i,\vx_j)\right|\right].
    \end{aligned}
\end{equation*}

\end{definition}

The Rademacher chaos complexities have been previously introduced for the generalization analysis of various kernel learning problems, including multiple kernel learning \citep{ying2010rademacher, zhuang2011two} and deep kernel learning \citep{zhang2023nearly}. In particular, the Rademacher chaos complexity of Gaussian-type kernels has been extensively studied in the literature \citep{ying2010rademacher}.
\begin{lemma}\label{lem: R-gau-kernel}
    (Corollary~1, \cite{ying2010rademacher}) Define a Gaussian-type kernel as follows
    \begin{equation*}
        \mathcal{K}_{\mathrm{gau}} = \left\{\exp\{-\sigma\|x-t\|^2\}:\sigma\in(0,\infty)\right\}.
    \end{equation*}
    Then it holds $\mathcal{U}_N(\mathcal{K}_{\mathrm{gau}})\leq (1+192e)\kappa^2.$
\end{lemma}
Based on this estimation, the following result shows that the sample error of $\pi_M(f_{\vz,\lambda})$ can be bounded by the empirical Rademacher chaos complexity over the kernel set derived by bandwidth set $\vTheta=\{\theta_1,\cdots,\theta_{N_{sv}}\}$.
To this end, we first define the function space that $f_{\vz,\lambda}$ exists.
\iffalse
Recall the constraint on $\mu(\sigma)$ in (\ref{equ: unconstrained}), denote $\mathcal{M}_{\sigma,\mathcal{X}_{sv}} = \overline{\mathrm{span}}\{\mathcal{K}_\sigma(x,\cdot)| x\in \mathcal{X}_{sv}\}$.
Then the function we estimate is in 
\begin{equation}\label{equ: Br}
    \begin{aligned}
        \tilde{\mathcal{M}}_{\mathcal{X}_{sv}}&:=\int_{\sigma\in \Omega} \mathcal{M}_{\sigma,\mathcal{X}_{sv}} d\mu(\sigma) = \sum_{\theta_i}  \mathcal{M}_{\theta_i,\mathcal{X}_{sv}}\\
        &=\left\{f = \sum_{\theta_i} \sum_{\vx_j\in\mathcal{X}_{sv}} (\alpha_{\theta_i})_{j}\mathcal{K}_{\theta_i}(\cdot,\vx_j): \sum_{\theta_i} \|\valpha_{\theta_i}\|_2^2 < \infty\right\},
    \end{aligned}
\end{equation}
which is a subspace of $\mathcal{H}_{\Omega}$.
\fi
Recall the constraints in (\ref{equ: unconstrained}), then we consider the following function space 
\begin{equation}\label{equ: Br}
    \mathcal{W}(R) = \left\{f:f\in\mathcal{H}_{\Omega}, \;\; {\mu}(\vsigma) = \sum_{\theta_i\in\vTheta} \delta(\vsigma - \theta_i),\;\; \mathcal{R}_0(f) \leq R,\;\; \vTheta\subset(0,+\infty)\right\}.
\end{equation}
\begin{lemma}\label{lem: s1}
Let $f_{\vz,\lambda}\in \mathcal{W}(R)$, where $\mathcal{W}(R)$ is defined by Equation~(\ref{equ: Br}) with proper radius $R$. Then, For any $\vz = \{\vx_i, y_i\}_{i=1}^{N}$ and any $\delta\in(0,1)$, with probability at least $1-\delta$, there holds
\begin{equation*}
    \begin{aligned}
        \mathcal{E}_\rho(\pi_M(f_{z,\lambda}))&-\mathcal{E}_z(\pi_M(f_{z,\lambda})) 
        \leq \frac{1}{2}C_\kappa M^2\log^\frac{1}{2}\left(\frac{2}{\delta}\right)N^{-\frac{1}{2}}R,
    \end{aligned}
\end{equation*}
where $C_\kappa = 32\kappa\left(\sqrt{384e+2}+1\right)$.
\end{lemma}

\begin{proof}
Based on Assumption~\ref{asmp: gt bound} and the definition of $\pi_M(\cdot)$, it holds that $|\pi_M(f)(\vx_i)-y_i|\leq 2M, \forall {\vx_i,y_i}\in \vz$.
Then applying McDiarmid's bounded difference inequality, the following inequality holds with at least $1-\delta/2$ probability 
\begin{equation*}
    \begin{aligned}
        \sup_{f\in\mathcal{W}(R)} [\mathcal{E}_\rho(\pi_M(f))-\mathcal{E}_z(\pi_M(f))]
        \leq \mathbb{E}\sup_{f\in\mathcal{W}(R)} [\mathcal{E}_\rho(\pi_M(f)) - \mathcal{E}_z(\pi_M(f))] + 4M^2\left(\log\frac{2}{\delta}/2N\right)^{\frac{1}{2}},
    \end{aligned}
\end{equation*}
With at least $1-\delta/2$ probability, the first term can be bounded by
\begin{equation*}
    \begin{aligned}
        \mathbb{E}\sup_{f\in\mathcal{W}(R)} \left[\mathcal{E}_\rho(\pi_M(f))-\mathcal{E}_z(\pi_M(f))\right]
        \overset{(a)}\leq 2\mathbb{E}\mathbb{E}_\epsilon \left[\sup_{f\in\mathcal{W}(R)} \frac{1}{N}\sum_{i\in\mathbb{N}_N}\epsilon_i(\pi_M(f)(\vx_i)-y_i)^2\right]&\\
         \overset{(b)}\leq 2\mathbb{E}_\epsilon \left[\sup_{f\in\mathcal{W}(R)} \frac{1}{N}\sum_{i\in\mathbb{N}_N}\epsilon_i(\pi_M(f)(\vx_i)-y_i)^2\right] + 8M^2\left(\log\frac{2}{\delta}/2N\right)^{\frac{1}{2}},&
    \end{aligned}
\end{equation*}
where $(a)$ uses the standard symmetrization arguments and $\epsilon_i$ are Rademacher variables.
Inequality $(b)$ uses McDiarmid's bounded difference inequality again.
Applying the contraction property of Rademacher averages, it holds that,
\begin{equation*}
    \begin{aligned}
        \mathbb{E}_\epsilon [\sup_{f\in\mathcal{W}(R)} \frac{1}{N}\sum_{i\in\mathbb{N}_N}\epsilon_i(\pi_M(f)(\vx_i)-y_i)^2]&
        \leq \frac{4M}{N} \mathbb{E}_\epsilon\sup_{f\in\mathcal{W}(R)} \sum_{i\in\mathbb{N}_N}\epsilon_i \pi_M(f)(\vx_i) ,
    \end{aligned}
\end{equation*}
because the Lipschitz constant of the loss function $\phi(t) = t^2, \forall |t|\leq 2M$ is bounded by $4M$.
Recall the definition of $\mathcal{W}(R)$ and $\langle \cdot,\cdot\rangle_{\mathcal{H}_{\Omega}}$, we have
\begin{equation*}
    \begin{aligned}
        \mathbb{E}_\epsilon\sup_{f\in\mathcal{W}(R)} &\sum_{i\in\mathbb{N}_N}\epsilon_i \pi_M(f)(\vx_i) 
        =\mathbb{E}_\epsilon\sup_{\vTheta\subset(0,+\infty)}\sup_{\mathcal{R}_0(f)\leq R} \sum_{i\in\mathbb{N}_N} \epsilon_i \pi_M\left(
        \sum_{\sigma\in\vTheta}  \left\langle \mathcal{K}_\sigma(\cdot, \vx_i), f_\sigma\right\rangle \right) \\
        &\leq \mathbb{E}_\epsilon\sup_{\vTheta\subset(0,+\infty)}\sup_{\mathcal{R}_0(f)\leq R} \sum_{i\in\mathbb{N}_N} \epsilon_i 
        \sum_{\sigma\in\vTheta}  \left\langle \mathcal{K}_\sigma(\cdot, \vx_i), \pi_M\left(f_\sigma \right)\right\rangle\\
        &= \mathbb{E}_\epsilon\sup_{\vTheta\subset(0,+\infty)}\sup_{\mathcal{R}_0(f)\leq R} 
        \sum_{\sigma\in\vTheta}  \left\langle \sum_{i\in\mathbb{N}_N} \epsilon_i \mathcal{K}_\sigma(\cdot, \vx_i), \pi_M\left(f_\sigma \right)\right\rangle\\
        &\overset{(a)} \leq   R\mathbb{E}_\epsilon \sup_{\sigma\in(0,+\infty)} \left\langle \sum_{i\in\mathbb{N}_N} \epsilon_i \mathcal{K}_\sigma(\cdot, \vx_i), \pi_M\left(f_\sigma \right)\right\rangle\\
        &\overset{(b)}\leq  R M\;
        \mathbb{E}_\epsilon \sup_{\sigma\in(0,+\infty)} \left| \sum_{i,j\in\mathbb{N}_N}  \epsilon_i\epsilon_j\mathcal{K}_\sigma(\vx_i, \vx_j)\right|^{\frac{1}{2}}\\
        &\leq  R M\;
        \left(\sqrt{2N}\;
        \mathbb{E}_\epsilon  \sup_{\sigma\in(0,+\infty)} \left|\sum_{i,j\in\mathbb{N}_N,i<j}  \epsilon_i\epsilon_j\mathcal{K}_\sigma(\vx_i, \vx_j)/N\right|^{\frac{1}{2}} 
        + \sup_{\sigma\in(0,+\infty)} \sqrt{\mathrm{tr}(\vK_\sigma)} \right),
    \end{aligned}
\end{equation*}
where inequality (a) is satisfied because conditions $\mathcal{R}_0(f)\leq R$ and $\sigma\in\vTheta$ together indicate the valid number of $f_\sigma$ is less than $R$, and inequality (b) is derived by the fact that $\|\pi_M\left(f_\sigma \right)\|_{\mathcal{H}_\sigma} \leq M$.
Here $\vK_\sigma$ denotes the kernel matrix defined as $[\vK_\sigma]_{ij} = \mathcal{K}_\sigma(\vx_i,\vx_j),$ $ \forall \vx_i,\vx_j\in\mathcal{X}$.
Finally, recall the definition of Rademacher chaos complexity and $\mathcal{U}_N(\mathcal{K}_{\mathrm{gau}})$, we have
\begin{equation*}
    \begin{aligned}
        \mathbb{E}_\epsilon  \sup_{\sigma\in\Omega} \left|\sum_{i,j\in\mathbb{N}_N,i<j}  \epsilon_i\epsilon_j\mathcal{K}_\sigma(\vx_i, \vx_j)/N\right|^{\frac{1}{2}}
        \leq \sqrt{\mathcal{U}_N(\mathcal{K}_{\mathrm{gau}})}.
    \end{aligned}
\end{equation*}
Combine all the above estimation with the result in Lemma~\ref{lem: R-gau-kernel}, it yields the following upper bound with at least $1-\delta$ confidence
\begin{equation*}
    \begin{aligned}
        \sup_{f\in\mathcal{W}(R)} [\mathcal{E}_\rho(\pi_M(f))-\mathcal{E}_z(\pi_M(f))]&\leq
         \frac{8M^2 R \kappa}{ \sqrt{N}}(\sqrt{384e+2}+1) + 12M^2\sqrt{\frac{\log (2/\delta)}{2N}},
    \end{aligned}
\end{equation*}
where we use the fact that $\mathrm{tr}(\vK_\sigma)\leq \kappa^2 N$.
With a proper radius $R$ such that $f_{\vz,\lambda}\in\mathcal{W}(R)$, we obtain the result in Lemma~\ref{lem: s1} and complete the proof.
\end{proof}

In a similar approach, we can obtain the following result on the sample error of $f_{z,\gamma}$.
Recall $f_{\vz,\gamma}(\cdot) = \sum_{i=1}^N \mathcal{K}_{\sigma^*}(\cdot,\vx_i)g_\gamma(\vx_i)$ with  $g_\gamma = \mathcal{L}^*_{{\mathcal{K}}}(\gamma\vI+\mathcal{L}_{\tilde{\mathcal{K}}_{\sigma^*}})^{-1}f_\rho$.
Then it holds $\|f_{z,\gamma}\|_{\infty} \leq \kappa\|g_\gamma\|_\infty\leq \frac{\kappa^2}{\gamma}\sqrt{D(\gamma)}$ (see proposition 4 in \cite{Guo2013Learning} for reference).
\begin{lemma}\label{lem: S f_zgamma}
For any $\vz = \{\vx_i, y_i\}_{i=1}^{N}$ and any $\delta\in(0,1)$, with probability at least $1-\delta$, there holds
\begin{equation*}
    \begin{aligned}
        \mathcal{E}_\rho(f_{z,\gamma})&-\mathcal{E}_z(f_{z,\gamma}) 
        \leq \frac{C_\kappa M\kappa^2\sqrt{D(\gamma)}}{4\gamma}N^{-\frac{1}{2}} + 6\sqrt{2} M^2\log^\frac{1}{2}\left(\frac{2}{\delta}\right)N^{-\frac{1}{2}}.
    \end{aligned}
\end{equation*}
\end{lemma}

\begin{proof}
    Given data $\vz = \{\vx_i, y_i\}_{i=1}^{N}$, consider $f\in\{f:f(\cdot)=\sum_{i=1}^{N}\alpha_i\mathcal{K}_{\sigma^*}(\cdot,\vx_i), \alpha\in\mathbb{R}\}$.
    From previous analysis, we know that 
    \begin{equation*}
        \begin{aligned}
            \mathcal{E}_\rho(f)-\mathcal{E}_z(f)
        \leq \frac{8M}{N} \mathbb{E}_\epsilon \sum_{i\in\mathbb{N}_N}\epsilon_i f(\vx_i) + 12M^2\sqrt{\frac{\log (2/\delta)}{2N}}.
        \end{aligned}
    \end{equation*}

Recall the definition of $f_{\vz,\gamma}$, it holds
\begin{equation*}
    \begin{aligned}
        \mathbb{E}_\epsilon \sum_{i\in\mathbb{N}_N}\epsilon_i f_{\vz,\gamma}(\vx_i) 
        &=\mathbb{E}_\epsilon \sum_{i\in\mathbb{N}_N} \epsilon_i \left\langle \mathcal{K}_{\sigma^*}(\cdot, \vx_i), f_{\vz,\gamma}\right\rangle \\
        &\overset{(a)}\leq \mathbb{E}_\epsilon\sup_{{\sigma}\in(0,+\infty)} \sum_{i\in\mathbb{N}_N} \epsilon_i 
        \left\langle \mathcal{K}_{\sigma}(\cdot, \vx_i), f_{\vz,\gamma}\right\rangle\\
        &\leq  \frac{\kappa^2}{\gamma}\sqrt{D(\gamma)}\;
        \mathbb{E}_\epsilon \sup_{\sigma\in(0,+\infty)} \left| \sum_{i,j\in\mathbb{N}_N}  \epsilon_i\epsilon_j\mathcal{K}_\sigma(\vx_i, \vx_j)\right|^{\frac{1}{2}}\\
        &\leq  \frac{\kappa^2}{\gamma}\sqrt{D(\gamma)}\;
        \left(\sqrt{2N}\;        \sqrt{\mathcal{U}_N(\mathcal{K}_{\mathrm{gau}})} 
        + \sup_{\sigma\in(0,+\infty)} \sqrt{\mathrm{tr}(\vK_\sigma)} \right),
    \end{aligned}
\end{equation*}
where (a) uses the fact that $\sigma^*\in(0,+\infty)$.
Then combine these estimation with Lemma~\ref{lem: R-gau-kernel} and $\mathrm{tr}(\vK_\sigma)\leq \kappa^2 N$, we obtain the result and complete the proof.
\end{proof}

Next we derive the estimator for the total error.
\iffalse
For $R\geq 0$, define
$
    \mathcal{W}(R) = \{\vz\in\mathcal{Z}^N:\mathcal{R}(f_{\vz,\lambda})\leq R\}.
$
Note that  $\mathcal{R}_0(f_{\vz,\lambda})$ and $\mathcal{R}(f_{\vz,\lambda})$ are different, and then a bridge between these two measures are needed.

\begin{lemma}\label{lem: alpha upper bound}
Denote the optimal solution of (\ref{equ: unconstrained}) as
$f_{\vz,\lambda}=(f_{\theta,\vz,\lambda})_{\theta\in\vTheta}$, where $f_{\theta,\vz,\lambda}(\cdot) = \sum_{\vx_i\in\vX_{tr}}(\valpha_{\theta,\vz,\lambda})_i\mathcal{K}_\theta(\cdot,\vx_i)$.
If $\vX$ is distinct and Assumption~\ref{asmp: gt bound} holds, then there exists a positive constant $M_0<\infty$ independent of $N$ such that
\begin{equation*}
    |(\valpha_{\theta,\vz,\lambda})_i|\leq M_0,\quad\forall i=1,\cdots,N,\;\forall\sigma\in\Omega.
\end{equation*}
\end{lemma}
The proof of Lemme~\ref{lem: alpha upper bound} is presented in Appendix~\ref{app: Lemma 3}.
Because there exists finite upper bound of $|(\valpha_{\theta,\vz,\lambda})_i|,\forall i$, it is easy to find that when 
$\mathcal{R}_0(f_{\vz,\lambda})\leq R$, it holds that $f_{\vz,\lambda}=(f_{\sigma,\vz,\lambda})_{\sigma\in\vTheta}$ and all components $f_{\sigma,\vz,\lambda}\in\mathcal{B}_{M_0R}$.
\fi

\begin{proposition}\label{prop: rough bound}
Suppose Assumption~\ref{asmp: gt bound} and  Assumption~\ref{asmp: D} hold with $0<\beta\leq 1$. 
If $(\gamma, q,\delta)\in (0,1)^3$, $\lambda=(\frac{1-q}{\kappa^2})^\frac{q}{2-q}\gamma^\frac{2}{2-q}$, and $R>1$, then there exists a subset $\mathcal{Z}_R$ of $\mathcal{Z}^N$ with measurement at most $\delta$ such that for any $\vz\in\mathcal{W}(R)\setminus\mathcal{Z}(R)$, 
\begin{equation*}
    \begin{aligned}
    \mathcal{E}_\rho&(\pi_M(f_{\vz,\lambda})) - \mathcal{E}_\rho(f_\rho) +\lambda \mathcal{R}_0(f_{\vz,\lambda})\leq C_\kappa M^2\log^\frac{1}{2}\left(\frac{6}{\delta}\right)N^{-\frac{1}{2}}R+\left(\frac{3}{2}\sqrt{c_\beta}+3c_\beta\right)\gamma^\beta
    \\    
    &
    + \frac{1}{4}C_\kappa M\kappa^2\sqrt{c_\beta} \gamma^{\beta/2-1}N^{-\frac{1}{2}}
    +{C_1}\log^2\left(\frac{12}{\delta}\right)\max\left\{\gamma^{\beta-2} N^{-2},\gamma^{\beta-1} N^{-1}\right\}
    +\gamma N^{1-q}\|f^1_{\vz,\gamma}\|_1^q ,
    \end{aligned}
\end{equation*}
where $C_1 = 16\kappa^2(2\kappa^2+1)+(2\kappa+1)\sqrt{c_\beta}.$

\end{proposition}

\begin{proof}
By properly choosing some $R$, we can directly apply lemma~\ref{lem: s1} and know that there exists $\mathcal{Z}_1\subset\mathcal{Z}^N$ with the measure at most $\delta/3$ such that for each $\vz\in\mathcal{W}(R)\setminus\mathcal{Z}_1$,
\begin{equation*}
\begin{aligned}
    \mathcal{E}_\rho(\pi_M(f_{z,\lambda}))&-\mathcal{E}_z(\pi_M(f_{z,\lambda})) 
        \leq \frac{1}{2}C_\kappa M^2\log^\frac{1}{2}\left(\frac{6}{\delta}\right)N^{-\frac{1}{2}}R.
\end{aligned}
\end{equation*}
Similarly, from  Lemma~\ref{lem: s2+s3} and Lemma~\ref{lem: S f_zgamma}, we know that there exists $\mathcal{Z}_2,\mathcal{Z}_3\subset\mathcal{Z}^N$ with the measure at most $\delta/3$ such that
\begin{equation*}
    \begin{aligned}
    \mathcal{S}_2
    &\leq 16\kappa^2(2\kappa^2+1)\log^2\left(\frac{12}{\delta}\right)\max\left\{\frac{\mathcal{D}(\gamma)}{\gamma^2 N^2},\frac{\mathcal{D}(\gamma)}{\gamma N}\right\} \\
    &\quad\quad + \frac{2\kappa+1}{N}\sqrt{\mathcal{D}(\gamma)}\log\left(\frac{12}{\delta}\right) + \frac{3}{2}\sqrt{\gamma\mathcal{D}(\gamma)} + 2 {\mathcal{D}(\gamma)}, \quad\forall\vz\in\mathcal{Z}^m\setminus\mathcal{Z}_2.
    \end{aligned}
\end{equation*}
and for any  $\vz\in\mathcal{Z}^m\setminus\mathcal{Z}_3$, it holds
\begin{equation*}
    \begin{aligned}
            \mathcal{E}_\rho(\pi_M(f_{z,\gamma}))&-\mathcal{E}_z(\pi_M(f_{z,\gamma})) 
        \leq \frac{C_\kappa M\kappa^2\sqrt{D(\gamma)}}{4\gamma}N^{-\frac{1}{2}} + 6\sqrt{2} M^2\log^\frac{1}{2}\left(\frac{6}{\delta}\right)N^{-\frac{1}{2}}.
    \end{aligned}
\end{equation*}
Note that $R\geq 1$. Then we take the above three bounds together and obtain
\begin{equation*}
    \begin{aligned}
    \mathcal{E}_\rho(\pi_M(f_{\vz,\lambda})) &- \mathcal{E}_\rho(f_\rho) +\lambda \mathcal{R}_0(f_{\vz,\lambda})\leq C_\kappa M^2\log^\frac{1}{2}\left(\frac{6}{\delta}\right)N^{-\frac{1}{2}}R + \frac{C_\kappa M\kappa^2\sqrt{D(\gamma)}}{4\gamma}N^{-\frac{1}{2}}\\
     &
    +16\kappa^2(2\kappa^2+1)\log^2\left(\frac{12}{\delta}\right)\max\left\{\frac{\mathcal{D}(\gamma)}{\gamma^2 N^2},\frac{\mathcal{D}(\gamma)}{\gamma N}\right\}  + \frac{2\kappa+1}{N}\sqrt{\mathcal{D}(\gamma)}\log\left(\frac{12}{\delta}\right) \\
    & + \frac{3}{2}\sqrt{\gamma\mathcal{D}(\gamma)} + 2 {\mathcal{D}(\gamma)}    
    +\gamma N^{1-q}\|f^1_{\vz,\gamma}\|_1^q,\quad\forall\vz\in\mathcal{W}(R)\setminus(\mathcal{Z}_1\cup\mathcal{Z}_2\cup\mathcal{Z}_3).\\
    \end{aligned}
\end{equation*}
Recall Assumption~\ref{asmp: D}, and thus we have
\begin{equation*}
    \begin{aligned}
    &\mathcal{E}_\rho(\pi_M(f_{\vz,\lambda})) - \mathcal{E}_\rho(f_\rho) +\lambda \mathcal{R}_0(f_{\vz,\lambda})
    \leq C_\kappa M^2 \log^\frac{1}{2}\left(\frac{6}{\delta}\right)N^{-\frac{1}{2}}R+ \frac{1}{4}C_\kappa M\kappa^2\sqrt{c_\beta}\gamma^{\beta/2-1}N^{-\frac{1}{2}}\\
    &+{C_1}\log^2\left(\frac{12}{\delta}\right)\max\left\{\gamma^{\beta-2} N^{-2},\gamma^{\beta-1} N^{-1}\right\} 
    +\left(\frac{3}{2}\sqrt{c_\beta}+3c_\beta\right)\gamma^\beta
    +\gamma N^{1-q}\|f^1_{\vz,\gamma}\|_1^q.
    \end{aligned}
\end{equation*}
Thus we complete our proof. 
\end{proof}

\subsection{Bounding $\mathcal{R}_0(f_{\vz,\lambda})$ by Iteration Technique}
%\subsection{Bounding the estimator by iteration}
In Proposition~\ref{prop: rough bound}, the radius $R$ of $\mathcal{R}_0(f_{\vz,\lambda})$ is assumed to be properly chosen.
Then this section determines the specific value of $R$, for which we need the conclusion on the bound of $\|f_{\vz,\gamma}^1\|_1$.
To this end, we need assumption on the covering number of the corresponding RKHS $\mathcal{H}_{\sigma^*}$.
The normalized  $\ell_2$-metric $d_2$ is defined as $d_2(\vx,\vx') = \left(\frac{1}{k}\sum_{i=1}^k|x_i-x'_i|^2\right).$
We consider balls $\mathcal{B}_\sigma(s_j,\epsilon) = \{s\in\mathbb{R}^k: d_2(s,s_j)\leq \epsilon\} $ in 
$\mathbb{R}^k$.
And the $\ell_2$-empirical covering number of $\mathcal{S}$ w.r.t.  $\epsilon$ and  $d_2$ is 
\begin{equation*}
\begin{aligned}
    \mathcal{N}(\mathcal{S},\epsilon, d_2) = &\min\left\{l\in\mathbb{N}:\mathcal{S}\subset \bigcup_{j=1}^l \mathcal{B}_\sigma(s_j,\epsilon) \;\;\mathrm{for\;some\;}\{s_j\}_{j=1}^l\subset \mathcal{H}_\sigma\right\},
\end{aligned}
\end{equation*}
which means the minimal number of balls with radius $\epsilon$ to cover the set $\mathcal{S}$ in $\mathcal{H}_{\sigma^*}$.
Let $\mathcal{F}$ be a set of function on $\mathcal{X}$, $\vx=\{x_i\}_{i=1}^k\subset\mathcal{X}^k$ and $\mathcal{F}|_{\vx} = \{(f(x_i))_{i=1}^k:\;f\in\mathcal{F}\}\subset\mathbb{R}^k$
Then its $\ell_2$ empirical covering number is defined as 
\begin{equation*}
    \mathcal{N}_2(\mathcal{F},\epsilon) = \sup_{k\in\mathbb{N}}\sup_{\vx\subset\mathcal{X}^k}\mathcal{N}(\mathcal{F}|_\vx,\epsilon,d_2).
\end{equation*}
We consider the linear combination of functions $\{\mathcal{K}_{\sigma^*}(\cdot,\vx_i)|\vx_i\in\mathcal{X}\}$ under the $\ell_1$ constraint, denoted as $\mathcal{B}_{{\sigma^*},R}$, with some $R>0$
\begin{equation}
    \mathcal{B}_{\sigma,R} = \left\{\sum_{i=1}^N\alpha_i\mathcal{K}_\sigma(\cdot,\vx_i),\;N\in\mathbb{N},\;\vx_i\in\mathcal{X},\;\alpha_i\in\mathbb{R},\;\mathrm{and}\;\sum_{i=1}^N|\alpha_i|\leq R\right\}.
\end{equation}
Then we use the following classical covering number assumption for $\mathcal{B}_{\sigma^*,1}$.
\begin{assumption}\label{asmp: covering number}
Let $\sigma^*$ defined by the minimizer of (\ref{equ: D-gamma}).
For the RBF kernel $\mathcal{K}$ derived by $\sigma^*$, there exists $p\in(0,2)$ and a constant $c_{\sigma^*,p}>0$ independent of $\epsilon$ such that
$$\log_2\mathcal{N}_2(\mathcal{B}_{\sigma^*,1},\epsilon)\leq c_{\sigma^*,p}\epsilon^{-p},\;\;\forall \epsilon>0.$$
\end{assumption}

Under this assumption, there is existing result on the upper bound of $\|f_{\vz,\gamma}^1\|_1$.
\begin{lemma}\label{lem: f1 bound}
(\cite{shi2019sparse}, Proposition 16) Let $f_{\vz,\gamma}^1$ be the optimal solution of (\ref{equ: f lq}) when $q=1$.
Assume Assumption~\ref{asmp: D} and Assumption~\ref{asmp: covering number} hold with $0<\beta\leq 1$ and $0<\gamma\leq 1$.
Take $\gamma = N^{-\tau}$ with $0<\tau<\frac{2}{2+p}$ and $0<\delta<1$. Then with $1-\delta$ confidence, it holds that
\begin{equation*}
    \begin{aligned}
    \|f_{\vz,\gamma}^1\|_1\leq
    C'_2\left(\log(1/\delta)+\log(J(\tau,p))\right)^3N^{(1-\beta)\tau},
    \end{aligned}
\end{equation*}
where $J(\tau,p)$ is a constant defined by
\begin{equation*}
    J(\tau,p) = \max\left\{2,\frac{\log\frac{(2-(2+p)\tau)p}{(1-p\tau)(2+p)}}{\log\frac{2p}{2+p}}\right\},
\end{equation*} 
and 
$C'_2 = 64((2Cc^{\frac{1}{2}}_{\sigma^*,p}(2-p)^{-1}M^2)^{\frac{2+p}{2-p}}M^2 + 4C'_1+12\sqrt{c_\beta} + 24c_\beta)$ with 
\begin{equation*}
\begin{aligned}
    C'_1 = 2(12(20+2Cc^{\frac{1}{2}}_{\sigma^*,p}(2-p)^{-1})&(3M+\kappa)^2(2\kappa^2+1)\\
    &+176M^2+40\kappa^2(2\kappa^2+1)+3)c_\beta + (4\kappa+5)\sqrt{c_\beta}    
\end{aligned}
\end{equation*}
and a universal constant $C$.
\end{lemma}
Then we can bound $\mathcal{R}_0(f_{\vz,\lambda})$ by a commonly-used iteration technique (e.g. \cite{Smale2007LearningTE, Wu2006LearningRO, Shi2011ConcentrationEF}) and obtain the following proposition.

\begin{proposition}\label{lem: r bound}
Under the Assumption~(1-3), let $0<\delta<1$, $\lambda = (\frac{1-q}{\kappa^2})^\frac{q}{2-q}N^{-\frac{2}{2-q}\tau}$ and $\gamma = N^{-\tau}$ with $\frac{1-q}{1-q(1-\beta)}<\tau<\min\{\frac{1}{2\beta+1},\frac{2-q}{4}\}$ and $1>q>0$.
Then with at least $1-\delta$ confidence we have
\begin{equation}\label{equ: R bound}
    \begin{aligned}
    \mathcal{R}_0(f_{\vz,\lambda}) \leq C_3\left(\log \tilde{J}\right)^{3q}\left(\log\left(\frac{24}{\delta}\right) + \log(\tilde{J}+1)\right)^{2\tilde{J}} N^{\left(\frac{q}{2-q}+q(1-\beta)\right)\tau+1-q},
    \end{aligned}
\end{equation}
where $\tilde{J}$ is a positive constant defined by
\begin{equation}\label{equ: def J}
    \tilde{J} = \max\left\{2,\frac{\log\frac{(2-(2+p)\tau)p}{(1-p\tau)(2+p)}}{\log\frac{2p}{2+p}},\frac{4\tau}{2-q-4\tau}\right\},
\end{equation} and $C_3=\left(M^{2(\tilde{J}+1)}+C_2\right)     \left(\frac{2\kappa^2C_\kappa }{1-q}\right)^{\tilde{J}}$.

\end{proposition}

\begin{proof}
Let $\lambda = (\frac{1-q}{\kappa^2})^\frac{q}{2-q}N^{-\frac{2}{2-q}\tau}$ and $\gamma = N^{-\tau}$ with $\frac{1-q}{1-q(1-\beta)}<\tau\leq \frac{1}{2\beta+1}$ and $1>q>0$.
And from Proposition~\ref{prop: rough bound} and Lemma~\ref{lem: f1 bound} we know that with confidence $1-\delta$ it holds
\begin{equation}
    \begin{aligned}
    &\mathcal{E}_\rho(\pi_M(f_{\vz,\lambda})) - \mathcal{E}_\rho(f_\rho) +\lambda \mathcal{R}_0(f_{\vz,\lambda})\\
    &\leq 
    C_\kappa M^2 \log^\frac{1}{2}\left(\frac{6}{\delta}\right)N^{-\frac{1}{2}}R
    +\left(C_1+\frac{3}{2}\sqrt{c_\beta}+3c_\beta + \frac{1}{4}C_\kappa M\kappa^2\sqrt{c_\beta}\right)\log^2\left(\frac{24}{\delta}\right) N^{-\beta\tau}\\
    &+C_2^{'q}\left(\log\left(\frac{2}{\delta}\right)+\log(J(\tau,p))\right)^{3q}N^{(q(1-\beta)-1)\tau+1-q}.
    \end{aligned}
\end{equation}
Note the fact that $-\beta\tau \leq (q(1-\beta)-1)\tau+1-q$ and $q/(2-q) < 1$, then we have
\begin{equation}
    \begin{aligned}
    \mathcal{R}_0(f_{\vz,\lambda}) \leq \max\{a_N R,b_N\},\quad\forall\vz\in\mathcal{W}(R)\setminus\mathcal{Z}_R,
    \end{aligned}
\end{equation}
where the measure of $\mathcal{Z}_R$ is no more than $\delta$ and
\begin{equation*}
    \begin{aligned}
    &a_N = \frac{2\kappa^2C_\kappa M^2}{1-q} \log^{1/2}\left(\frac{12}{\delta}\right)N^{-\frac{1}{2}+\frac{2}{2-q}\tau} ,\\
    &b_N = \frac{\kappa^2 C_2}{1-q} \log(J(\tau,p))^{3q}\log^2\left(\frac{24}{\delta}\right)N^{\left(\frac{q}{2-q}+q(1-\beta)\right)\tau+1-q},
    \end{aligned}
\end{equation*}
where 
$C_2 = 2C_1+3\sqrt{c_\beta}+6c_\beta+C_\kappa M\kappa^2\sqrt{c_\beta}/2 + 2C_2$.
This follows that
\begin{equation}\label{equ: Wr}
    \mathcal{W}(R) \subseteq \mathcal{W}(\max\{a_N  R,b_N\}) \cup \mathcal{Z}_R .
\end{equation}
Then we can determine $\mathcal{R}_0(f_{\vz,\lambda})$ by iteratively applying (\ref{equ: Wr}) on a sequence of radii $\{R^{(j)}\}$, which is defined as $R^{(0)}=M^2/\lambda$ and 
\begin{equation}\label{equ: def Rj}
\begin{aligned}
    R^{(j)} &= \max\{a_N R^{(j-1)},b_N\}
    \quad \forall j\in\mathbb{N}.
\end{aligned}
\end{equation}
Recall the measure of $\mathcal{Z}(R^{(j)})$ is no more than $\delta$, and from the optimality of $f_{\vz,\lambda}$ we know that
\begin{equation*}
    \lambda \mathcal{R}_0(f_{\vz,\lambda})\leq\mathcal{E}_z(f_{\vz,\lambda}) + \lambda \mathcal{R}_0(f_{\vz,\lambda})\leq\mathcal{E}_z({\bf 0}) + \lambda \mathcal{R}_0({\bf 0}) \leq M^2,
\end{equation*} 
which indicates that $\mathcal{W}(R^{(0)})= \mathcal{Z}^N$.
Then apply the inclusion (\ref{equ: Wr}) for $j=1,\cdots,J$, we have
\begin{equation}\label{equ: inclusion}
    \begin{aligned}
    \mathcal{Z}^N =     \mathcal{W}(R^{(0)}) \subseteq \mathcal{W}(R^{(1)}) \cup \mathcal{Z}(R^{(0)})
    \subseteq \cdots \subseteq \mathcal{W}(R^{(J)}) \cup \left(\bigcup_{j=0}^{J-1}\mathcal{Z}(R^{(j)})\right),
    \end{aligned}
\end{equation}
where the measure of $\left(\bigcup_{j=0}^{J-1}\mathcal{Z}(R^{(j)})\right)$ is no more than $J\delta$ and therefore the measure of $\mathcal{W}(R^{(J)})$ at least $1-J\delta$.
By the definition (\ref{equ: def Rj}) we have
\begin{equation}
    \begin{aligned}
    R^{(J)} = \max\{(a_N)^{J}R^{(0)},&(a_N)^{J-1}b_N,\cdots, a_Nb_N,b_N\}.
    \end{aligned}
\end{equation}
The first term can be bounded as
\begin{equation}
    \begin{aligned}
    &(a_N)^{J} R^{(0)} \leq (a_N)^{J} M^2 \lambda^{-1}
    \leq \left(\frac{2\kappa^2C_\kappa}{1-q}\log^{\frac{1}{2}}\left(\frac{12}{\delta}\right)\right)^{J}M^{2(J+1)}
    N^{-\frac{J}{2}+\frac{2(J+1)}{2-q}\tau}.
    \end{aligned}
\end{equation}
And the rest terms can be reduced as
\begin{equation}
    \begin{aligned}
    &\max\{(a_N)^{J-1}b_N,\cdots, a_Nb_N,b_N\}= \max\{(a_N)^{J-1},1\}b_N.
    \end{aligned}
\end{equation}

Define
\begin{equation*}
    \begin{aligned}
    &A_\delta = \frac{2\kappa^2C_\kappa }{1-q}\log^{1/2}\left(\frac{12}{\delta}\right),\\
    &B_\delta = \frac{\kappa^2 C_2}{1-q} \log(J(\tau,p))^{3q}\log^2\left(\frac{24}{\delta}\right),\\
    & \tilde{\alpha} = \left(\frac{q}{2-q}+q(1-\beta)\right)\tau+1-q,
    \end{aligned}
\end{equation*}
then we have
\begin{equation}
    \begin{aligned}
    R^{(J)} =\max\left\{A_\delta^{J}M^{2(J+1)},B_\delta ,A_\delta^{J-1} B_\delta\right\} N^{\tilde{\nu}}
    \end{aligned}
\end{equation}
where
\begin{equation*}
\begin{aligned}
    \tilde{\nu} &= \max\left\{-\frac{J}{2}+\frac{2(J+1)}{2-q}\tau, \;\tilde{\alpha},\;\tilde{\alpha} + (J-1)\left(-\frac{1}{2}+\frac{2}{2-q}\tau\right)
    \right\}.
\end{aligned}
\end{equation*}
We choose $\tau$ by restricting
\begin{equation}
    \frac{2\tau}{2-q}-\frac{1}{2}\leq 0
\end{equation}
Then we can determine $J$ under this restriction as the minimal integer number satisfying
\begin{equation*}
\begin{aligned}
    -\frac{J+1}{2}+\frac{2(J+2)}{2-q}\tau\leq0,
\end{aligned}
\end{equation*}
and that is,
\begin{equation*}
\begin{aligned}
    \max\left\{1,\frac{4\tau}{2-q-4\tau}-1\right\}\leq J<\max\left\{2,\frac{4\tau}{2-q-4\tau}\right\}.
\end{aligned}
\end{equation*}
Then recall (\ref{equ: inclusion}) and we have that with confidence at least $1-2\delta$,
\begin{equation*}
    \mathcal{R}_0(f_{\vz,\lambda})\leq \left(M^{2(J+1)}+C_2\right)
     \left(\frac{2\kappa^2C_\kappa }{1-q}\right)^J \log^{2J}\left(\frac{24}{\delta}\right)(\log J(\tau,p))^{3q}N^{\tilde{\alpha}}.
\end{equation*}
Then we can derive the bound by scaling $(J+1)\delta$ to $\delta$ and complete our proof.
\end{proof}

\subsection{Proof of Main Result}

Now we are at the stage of proofing Theorem~\ref{the: main result}.
\begin{proof}
From the definition of $f_\rho$, we have that
$
    \|\pi_M(f_{\vz,\lambda})-f_\rho\|_{L_{\rho_\vX}} = \mathcal{E}_\rho(\pi_M(f_{\vz,\lambda})) - \mathcal{E}_\rho(f_\rho).
$
Then by Proposition~\ref{prop: rough bound} we know that for some $R>0$ there exists a $\mathcal{Z}(R)$ whose measurement is at most $\delta\;(0<\delta<1)$ such that $\forall \vz\in\mathcal{W}(R)\setminus\mathcal{Z}(R)$,
\begin{equation*}
    \begin{aligned}
    \|\pi_M(f_{\vz,\lambda})-f_\rho\|_{L_{\rho_\vX}}
    \leq 
    C_\kappa M^2 \log^{1/2}\left(\frac{12}{\delta}\right)RN^{-1/2}+\left(\frac{3}{2}\sqrt{c_\beta}+3c_\beta\right)\gamma^\beta&
    \\
    +{C_1}\log^2\left(\frac{12}{\delta}\right)\max\left\{\gamma^{\beta-2} N^{-2},\gamma^{\beta-1} N^{-1}\right\} 
    +\gamma N^{1-q}\|f^1_{\vz,\gamma}\|_1^q.&
    \end{aligned}
\end{equation*}
Let $R$ be the right hand side of (\ref{equ: R bound}) and then the measurement of $\mathcal{W}(R)$ is at least $1-\delta$.
Lemma~\ref{lem: f1 bound} guarantee that with at least $1-\delta$ confidence that 
\begin{equation*}
    \begin{aligned}
    \|f_{\vz,\gamma}^1\|_1\leq
    C'_2\left(\log(1/\delta)+\log(J(\tau,p))\right)^3N^{(1-\beta)\tau}.
    \end{aligned}
\end{equation*}
Let $\lambda =(\frac{1-q}{\kappa^2})^\frac{q}{2-q} N^{-\frac{2}{2-q}\tau}$ and $\gamma = N^{-\tau}$ with $\frac{1-q}{1-q(1-\beta)}<\tau<\min\{\frac{1}{2\beta+1},\frac{2-q}{4}\}$  and $1>q>0$.
Combining the above three bound together we have that with at least $1-3\delta$ confidence that
\begin{equation*}
    \begin{aligned}
    \|\pi_M(f_{\vz,\lambda})&-f_\rho\|_{L_{\rho_\vX}}
    \leq \tilde{C_3}\log^2\left(\frac{12}{\delta}\right)\left(\log\left(\frac{24}{\delta}\right)+\log(\tilde{J}+1)\right)^{2\tilde{J}}N^{-\Delta},
    \end{aligned}
\end{equation*}
where $\tilde{C_3}=16M^2(\kappa+1)(\sqrt{384e+2}+1)C_3+\frac{1}{2}C_2$ and 
\begin{equation}
    \Delta=\min\left\{\frac{1}{2}-
    \tilde{\alpha},
    \;\frac{2\tau}{2-q}- \tilde{\alpha},
    \;\beta\tau
    \right\}.
\end{equation}
Then under the restriction on $\tau,\beta$, and $q$, we know that
\begin{equation*}
    \begin{aligned}
        \Delta&=\frac{2}{2-q}\tau-\tilde{\alpha}=\left(1-q(1-\beta)\right)\tau-(1-q).
    \end{aligned}
\end{equation*}
Finally, we consider the assumptions and restrictions.
Recall that Gaussian kernels are considered and we suppose $f_\rho\in\mathcal{H}^s(\mathcal{X})$ for some $s>0$. Then from the Proposition 22 in \cite{ying2007learnability} we know that Assumption~\ref{asmp: D} holds true.
Recall that $\sigma^*$ belongs to a pre-given closed interval $\Omega$, according to previous result in \cite{Shi2011ConcentrationEF}, the capacity assumption~\ref{asmp: covering number} is satisfied for Gaussian kernel $\mathcal{K}_{\sigma^*}$ and we can choose $p=d/s$.
Besides, we choose $q = 1-\xi$ and $\tau = \frac{\beta/4}{\xi+\beta-\xi\beta}$ with arbitrarily small $\frac{\beta}{4}>\xi>0$.
One can verify this choice satisfies the restriction of $\tau$ and $q$ since $\beta\in(0,1]$.
And then we obtain $\Delta = \frac{\beta}{4}-\epsilon$ and
\begin{equation*}
\begin{aligned}
    &\frac{4\tau}{2-q-4\tau} = \frac{\beta}{(1-\beta)\xi^2 + \xi} >2,
    \\
    &\frac{\log\frac{(2-(2+p)\tau)p}{(1-p\tau)(2+p)}}{\log\frac{2p}{2+p}} 
    = \frac{\log\frac{1-p\tau}{1-\frac{p+2}{a}\tau}}{\log\frac{p+2}{2p}} +1
    =\frac{\log\frac{2s\beta+2s\xi(1-\beta)-2d\beta}{6s\beta+8s\xi(1-\beta)-d\beta}}{\log\frac{d+2s}{2d}}+1.    
\end{aligned}
\end{equation*}

By scaling $3\delta$ to $\delta$ we then derive the total bound and complete our proof.
\end{proof}

\section{Numerical Experiments}
\label{sec: exp}
This section presents results that support our earlier theoretical analysis and highlight the outstanding performance of the proposed kernel ridgeless model. We compare these results to advanced regression methods using real datasets, with a specific emphasis on examining the impact of the number of training data and the number of support data. This analysis sheds light on the crucial effects of these factors on the algorithm's performance.

\subsection{Experiment Setting}
\textbf{Datasets.} Synthetic data are generated from typical nonlinear regression test functions provided by \cite{Cherkassky1996Comparison} with the following formulations:
\begin{equation*}
    \begin{aligned}
        &f_1(\vx)=\frac{1+\sin(2x(1)+3x(2))}{3.5+\sin(x(1)-x(2))},\;D=[-2,2]^2,\\
        &f_2(\vx)=10\sin(\pi x(1)x(2))+20(x(3)-0.5)^2+5x(4)+10x(5)+0x(6),\;D=[-1,1]^6,\\
        &f_3(\vx)=\exp(2\pi x(1)(\sin(x(4)))+\sin(x(2) x(3))),\;D=[-0.25,0.25]^4,
    \end{aligned}
\end{equation*}
where $D=[a,b]^n=\{\vx|\vx\in \mathbf{R}^n,a\leq x(i)\leq b,\forall 1\leq i \leq n\}$.
Real datasets include: Yacht \citep{misc_yacht_hydrodynamics_243}, Airfoil \citep{misc_airfoil_self-noise_291}, Parkinson \citep{tsanas2009accurate}, SML \citep{misc_sml2010_274}, Electrical \citep{misc_electrical_grid_stability_simulated_data__471}, Tomshardware \citep{misc_buzz_in_social_media__248} from UCI dataset \citep{asuncion2007uci}, Tecator from StatLib \citep{vlachos2005statlib}, Comp-active from Toronto University,  and KC House from Kaggle \citep{harlfoxem2016house}.
MNIST dataset \citep{deng2012mnist} and Fashion-MNIST \citep{xiao2017fashion} are used for testing classification task, where we use the given training and test set.
MNIST and Fashion-Mnist contains images of $28\times 28$ pixels, from digit $0$ to digit $9$.
We vectorize each image to a $784\times 1$ vector.
Each feature dimension of data and the label are normalized to $[-1,1]$.
Other detailed description of datasets are provided in Appendix~\ref{apd: exp}.

\textbf{Measurement.} 
%To eliminated the effect of labels' scale, the following 
We use R-squared ($R^2$), also known as the coefficient of determination (refer to \cite{gelman2019r} for more details), on the test set $\mathcal{Z}_{test}$ to evaluate the regression performance. 
\begin{displaymath}
R^2=1-\frac{\sum_{(\vx_i,y_i)\in \mathcal{Z}_{test}}(y_i-{\hat f(\vx_i)})^2}{\sum_{(\vx_i,y_i)\in \mathcal{Z}_{test}}(y_i-\bar{y})^2},
\end{displaymath}
where $\hat{f}$ is the estimated function, and $\bar{y}$ is the mean of labels.
%For the binary classification task, we use the classification accuracy as the metric.
%Except where specified, 

\textbf{Compared methods.} 
We compared 9 regression methods, including 2 traditional kernel regression methods using RBF (RBF KRR, \citep{vovk2013kernel}) and indefinite TL1 kernels (TL1 KRR, \citep{Tl12018Huang}). Additionally, there are multiple kernel learning methods applied on support vector regression, denoted as SVR-MKL and R-SVR-MKL (using only RBF kernel candidates). We also consider 3 recent kernel methods: Falkon \citep{rudi2017falkon, meanti2022efficient}, EigenPro3.0 \citep{Eigenpro}, Recursive feature machines (RFMs, \citep{radhakrishnan2022feature}), with the first 2 being based on the Nystr\"om method. Finally, 2 neural network-based methods are included: ResNet  \citep{chen2020deep}, and wide neural network (WNN).
All setting and hyper-parameters of these methods are given in Appendix~\ref{apdx: 3}.

Except where specified, all the following experiments randomly take $80\%$ of the total data as training data and the rest as testing data, and are repeated 50 times.
In Algorithm~\ref{alg: AKL}, we perform the inverse operation on $\vK_\vTheta(\vX,\vX) + \lambda\vI_N$, where $\lambda=1e-5$, instead of directly on $\vK_\vTheta(\vX,\vX)$, in order to mitigate potential numerical issues.
All the experiments were conducted using Python on a computer equipped with an AMD Ryzen 9 5950X 16-Core 3.40 GHz processor, 64GB RAM, and an NVIDIA GeForce RTX 4060 GPU with 8GB memory.
The code is publicly accessible at \url{https://github.com/hefansjtu/LABRBF_kernel}.

\subsection{Experimental Result}

\begin{figure}[tbp]
\begin{center}
\subfloat[Synthetic data $f_1$: 600 training data, 2 features]{\includegraphics[width=0.85\textwidth]{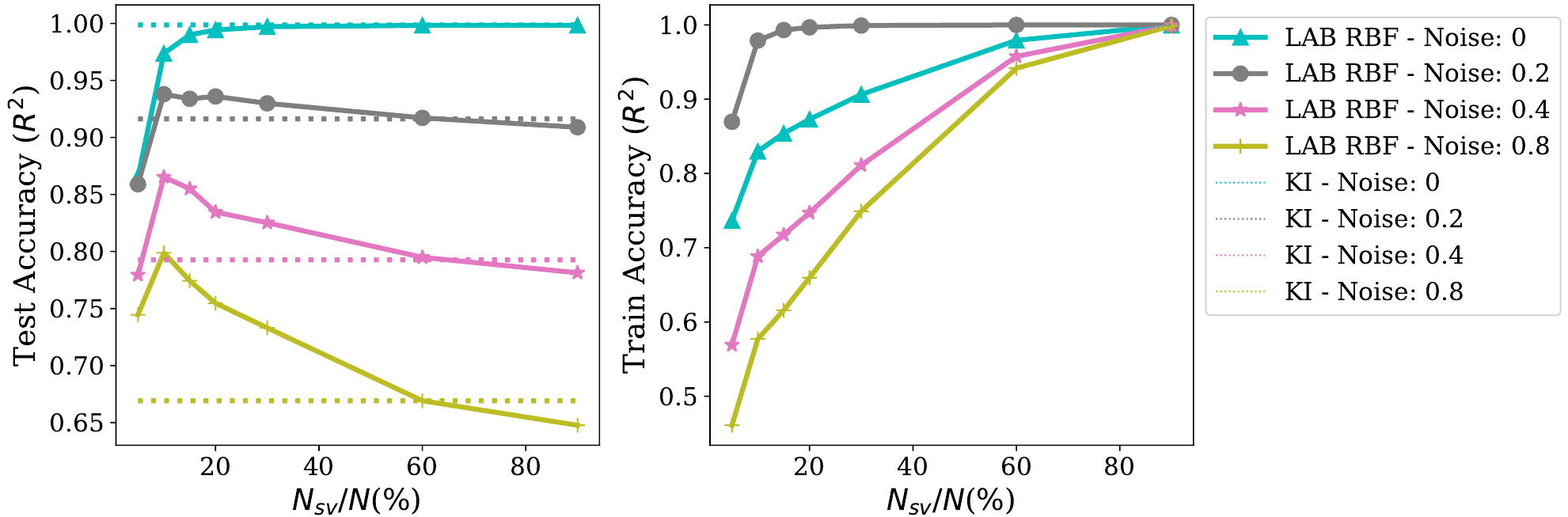}}
\vspace{3pt}
\subfloat[Synthetic data $f_2$: 600 training data, 6 features]{\includegraphics[width=0.85\textwidth]{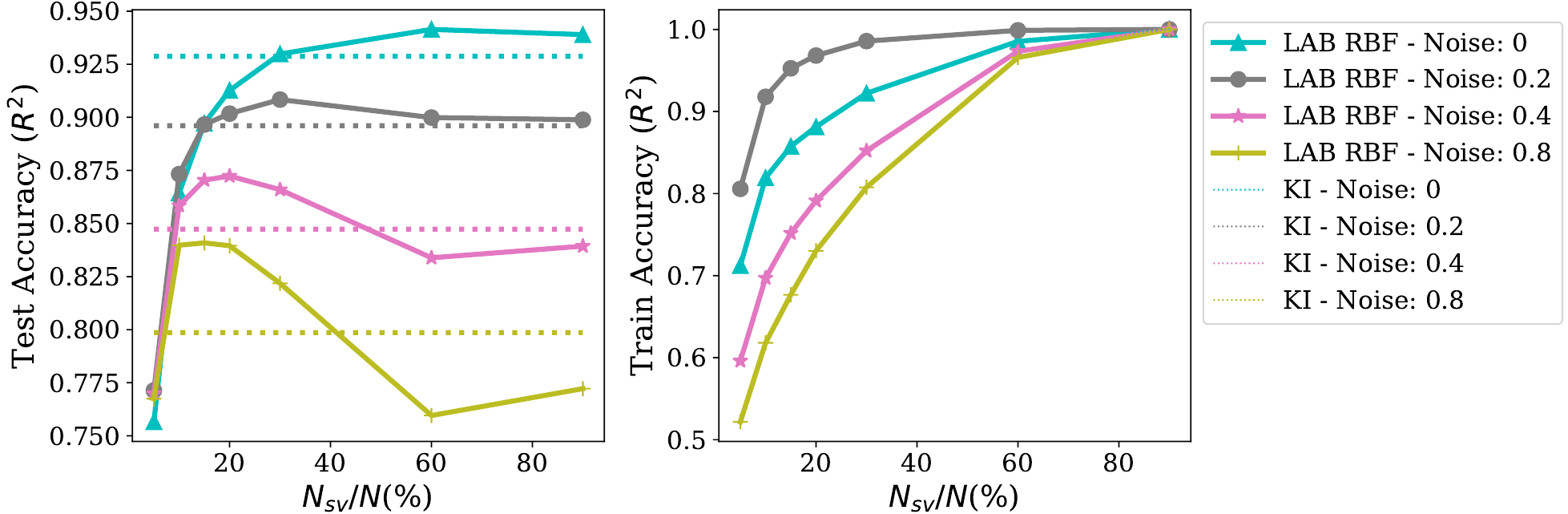}}
\vspace{3pt}
\subfloat[Synthetic data $f_3$: 600 training data, 4 features]{\includegraphics[width=0.85\textwidth]{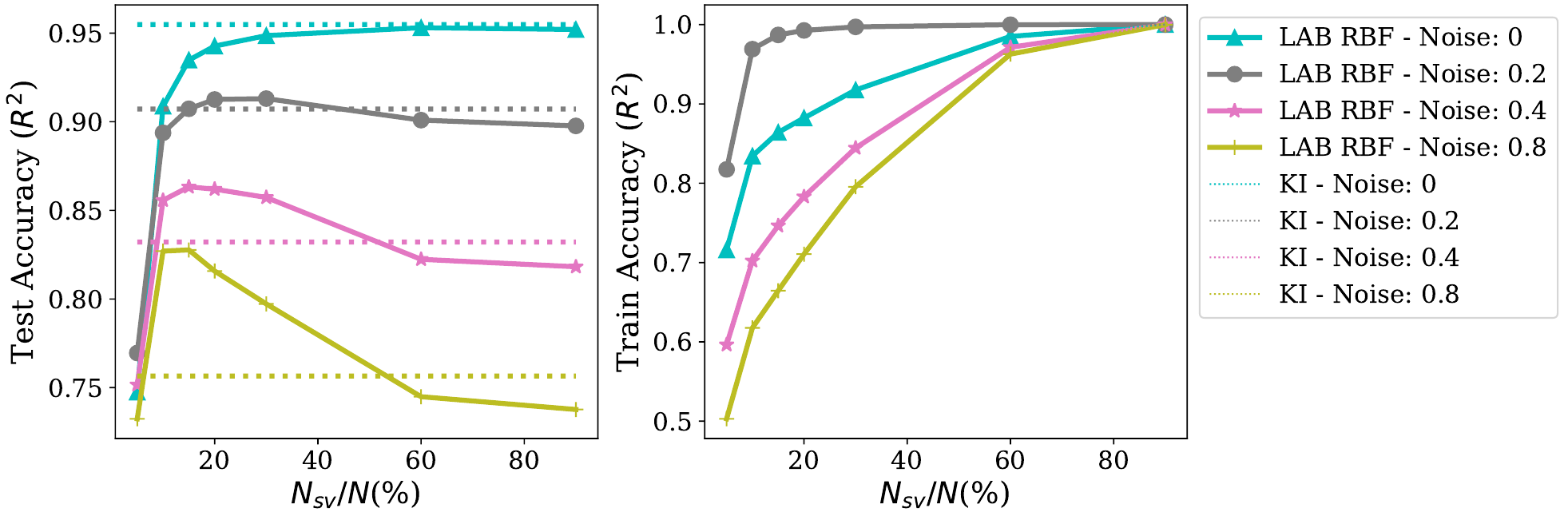}}
\caption{Effect of the number of support data on the performance of Algorithm~\ref{alg: AKL}. Three synthetic are used. Results of Algorithm~\ref{alg: AKL} is presented in solid lines, and results of traditional kernel interpolation models are shown in dash lines. Various levels of noise are introduced into the training data.} 
\label{fig: exp2-1}
\end{center}
\end{figure}

\begin{figure}[htbp]
\begin{center}
\includegraphics[width=0.9\textwidth]{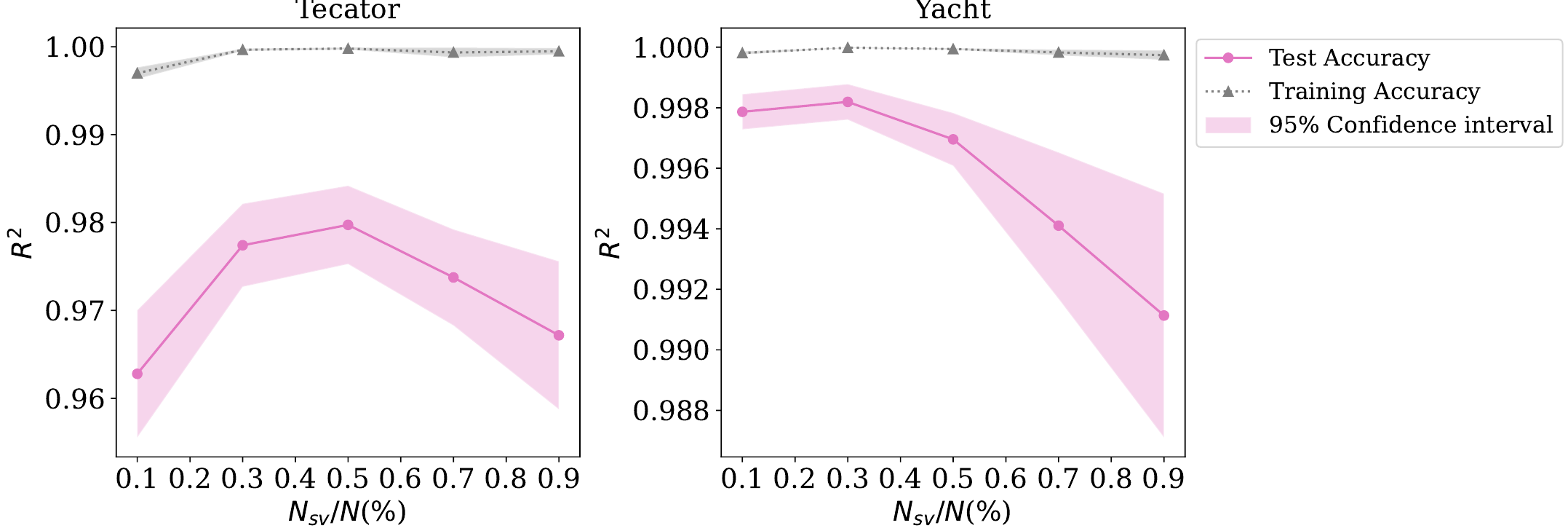}
\caption{Effect of the number of support data on the performance of Algorithm~\ref{alg: AKL}. Two real datasets are used. Test accuracy is presented in solid lines, and training accuracy is shown in dash lines. } 
\label{fig: exp2-2}
\end{center}
\end{figure}

\textbf{The impact of support data number.}
Figures~\ref{fig: exp2-1} provide a detailed illustration of the impact of the support data number, displaying training and test accuracy curves in relation to the ratio of support data. To thoroughly evaluate the effect of noise, synthetic data with varying noise levels are considered in Figure~\ref{fig: exp2-1}, measured by the ratio of noise variance to the label variance. The results of the standard RBF kernel interpolation model (denoted as KI) are marked by the dashed lines. It is observed that KI is not robust to noise, as the accuracy sharply decreases with higher noise levels. In contrast, the proposed LAB RBF kernel-based ridgeless regression exhibits good robustness when an appropriate number of support data is selected. This validates our previous analysis, indicating that controlling support data number can enhance the model's generalization ability.

Then, we utilize two small real datasets to closely examine the impact of the number of support data points.
Figure~\ref{fig: exp2-2} shows that having too few support data limits the capacity of the hypothesis space to fit the data, resulting in underfitting. Conversely, if the number of support data is excessively large, the remaining training data becomes insufficient to provide necessary information for learning bandwidths. This can cause the model to behave more like a simple kernel-based interpolation, making it less robust to noise and prone to overfitting, as evident from Figure~\ref{fig: exp2-2}. Therefore, selecting an appropriate number of support data is crucial to strike a balance between model complexity and overfitting. %Our findings indicate that Algorithm~\ref{alg: AKL} performs optimally when the ratio of support data is in the range of $30\%-60\%$.

\textbf{Representational ability of the estimator.}
Four more real datasets with varying feature dimensions are studies in Figure~\ref{fig: exp1}, which illustrates the accuracy curve of kernel ridgeless regression with LAB RBF kernels in relation to the number of training data.  As the number of training data increases, approximating all training data becomes more challenging, evident in the decline of the training accuracy curve. Conversely, the test accuracy improves with more information, indicating that the proposed algorithm gradually captures the underlying function. It is important to note that with only hundreds of support data, our approach effectively learns from tens of thousands of training data. For instance, with just 500 support data points, our method achieves over 99.5\% training accuracy and 97.5\% testing accuracy on the MNIST dataset, demonstrating the strong representational ability of the estimator.

In Figure\ref{fig: exp1}, results for different numbers of support data are also presented. It is observed that using more support data leads to better accuracy on the training dataset, which again aligns with our theoretical analysis. As the support data number increases, the capacity of the hypothesis space increases, allowing it to capture more complex underlying patterns in the data and resulting in higher training accuracy. However, its effect on generalization ability is not always positive. For instance, the function with 400 support data performs worse than that with only 200 support data in Figure \ref{fig: exp1} (b). As analyzed previously, a large number of support data implies a complex hypothesis space, which might bring larger sample error.

\textbf{Comparison with other regression methods on more real datasets.}
The regression results of 10 methods on small-scale datasets are presented in Table~\ref{tab: exp-1}. It is evident that greater model flexibility leads to improved regression accuracy, thus highlighting the benefits of flexible models. 
Notably, TL1 KRR outperforms RBF KRR in most datasets due to its indefinite nature. R-SVR-MKL, which considers a larger number of RBF kernels, exhibits much better performance than RBF KRR. While SVR-MKL, which considers a wider range of kernel types, achieves even higher accuracy compared to R-SVR-MKL. 
Among the neural network models, both ResNet and WNN demonstrate superior performance to the aforementioned methods.  
Advanced kernel methods, including Falkon, EigenPro3.0, and RFMs, also present significant improvement over traditionay kernel methods.
Overall, LAB RBF achieves the highest regression accuracy, significantly increasing the $R^2$ compared to the baseline. Notably, LAB RBF performs better than ResNet in certain datasets, indicating that LAB RBF kernels offer sufficient flexibility and training bandwidths on the training dataset is indeed effective to enhance the model generalization ability.

\begin{figure}[H]
\begin{center}
\subfloat[Electricity: 10000 training data, 11 features]{\includegraphics[width=0.83\textwidth]{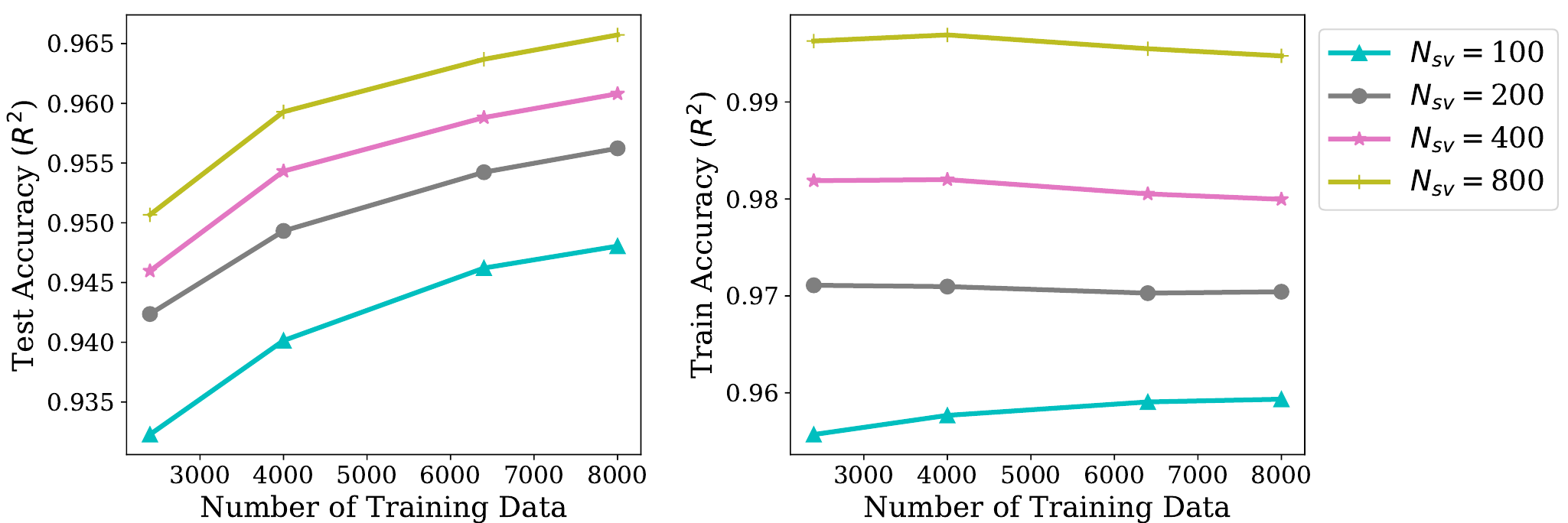}}
\vspace{3pt}
\subfloat[Tomshardware: 28179 training data, 96 features]{\includegraphics[width=0.83\textwidth]{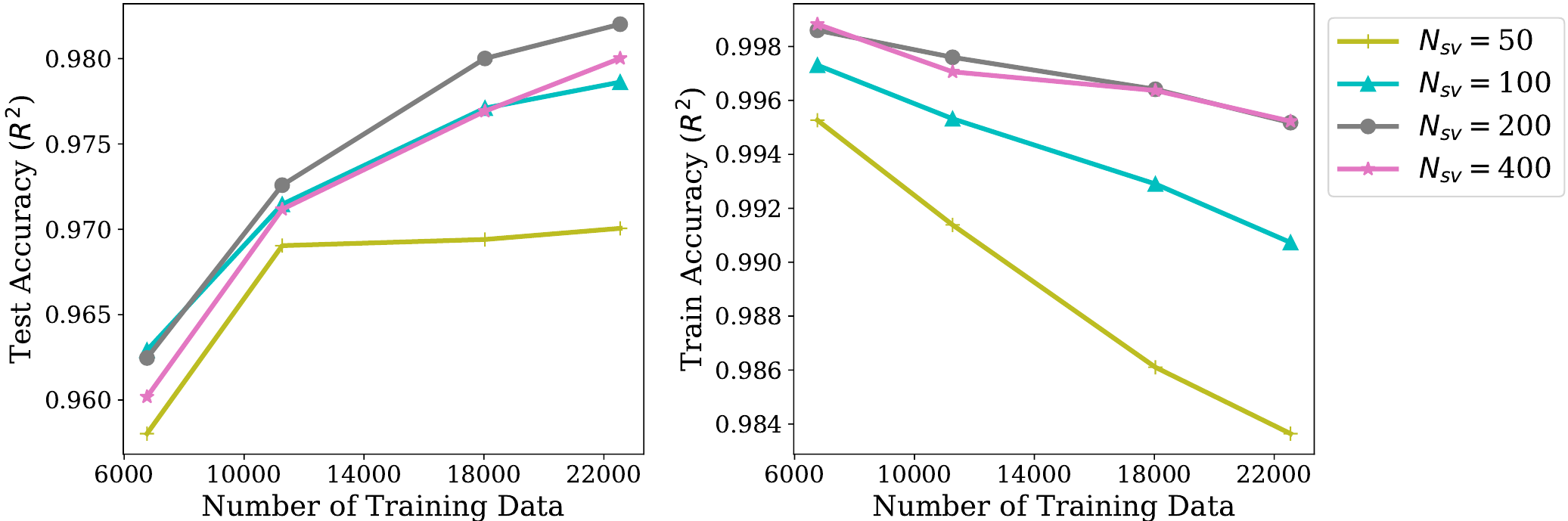}}
\vspace{3pt}
\subfloat[MNIST: 60000 training data, 784 features]{\includegraphics[width=0.83\textwidth]{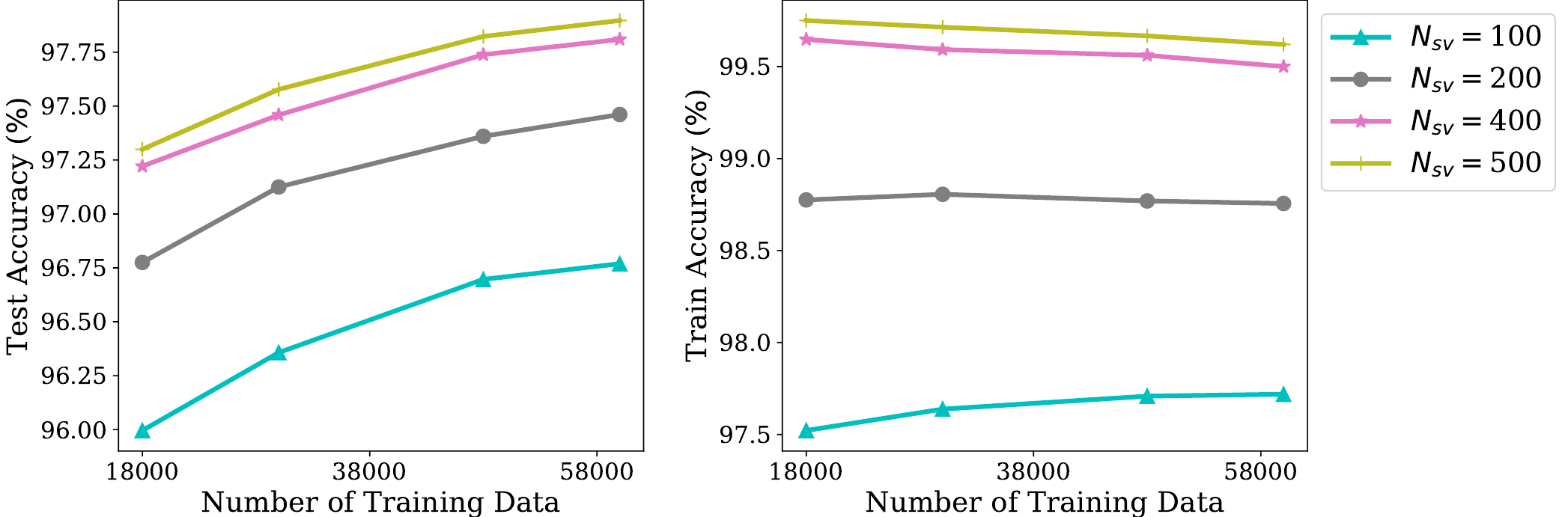}}
\vspace{3pt}
\subfloat[FashionMNIST: 60000 training data, 784 features]{\includegraphics[width=0.83\textwidth]{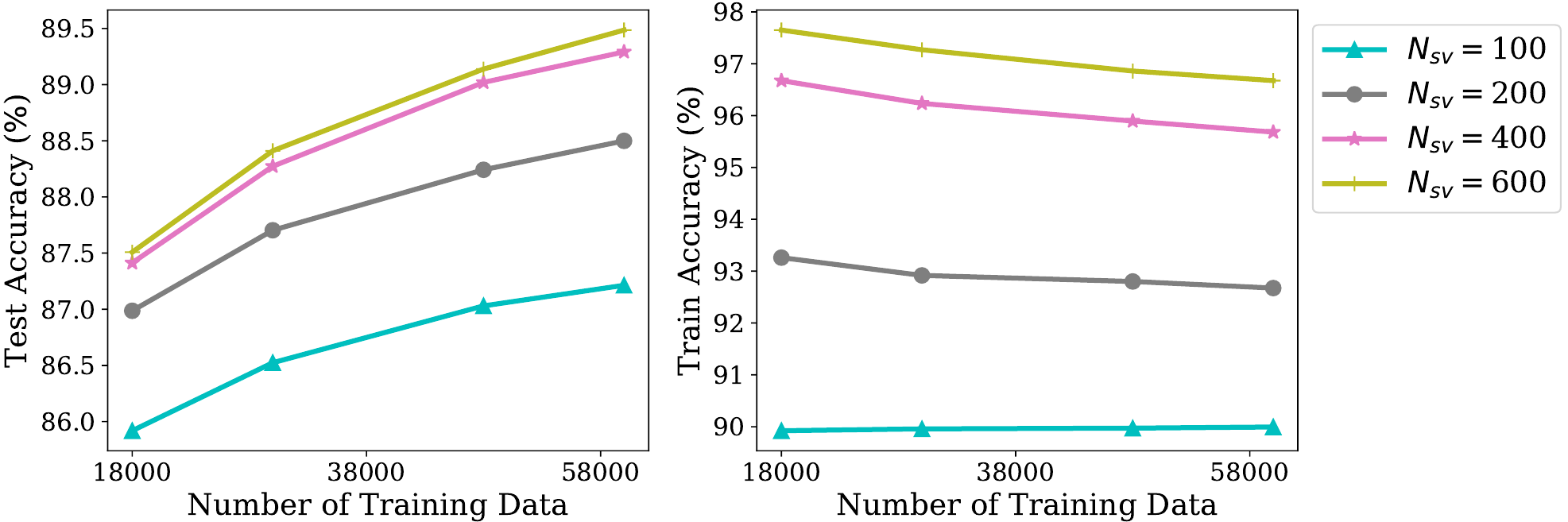}}
\caption{Effect of the number of training data on the performance of Algorithm~\ref{alg: AKL}. Two regression and two classification real datasets are used. Results of models with different support data number is shown in different color.}
\label{fig: exp1}
\end{center}
\end{figure}

Table~\ref{tab: exp-1} additionally reports the number of support vectors of kernel methods, which enables a more intuitive understanding of the sizes of decision models. Specifically, it provides the maximal support data number in Algorithm~\ref{alg: AKL} for LAB RBF, and the predefined number of centers for Falkon, and the average number of support vectors for SVR methods. 
It should be noted that KRR uses all training data as support data, which results in a much larger complexity of the decision model compared to other kernel methods. 
This observation further underscores the advantage of enhancing kernel flexibility and learning kernels, as demonstrated by the compact size of the decision model achieved with our proposed LAB RBF kernel.

Traditional kernel-based algorithms are inefficient on large-scale datasets due to the matrix inverse operation on the large kernel matrix.
Consequently, we compare our algorithm with three advanced kernel methods and two neural-network-based methods.
The results are presented in Table~\ref{tab: exp-2}, which more prominently underscores the capability of LAB RBF kernels in effectively reducing the required number of support data. Furthermore, it achieves a comparable level of regression accuracy to other advanced choices designed for such datasets. Notably, these advanced methods exhibit substantial model sizes. For instance, ResNet has a substantial number of parameters, and RFMs utilize all of the training data as support data. Although Falkon and EigenPro3.0 are based on the Nystr\"om method, their reliance on symmetric kernel functions forces them to use a large amount of training data as support data to achieve high accuracy. In contrast, LAB RBF kernels maintain a comparatively low number of support data, attributed to the high flexibility provided by locally adaptive bandwidths and the kernel learning algorithm.

\section{Related Works}
\label{sec: Link}
\textbf{Kernel ridgeless regression.}
The focus of this paper is on kernel ridgeless regression \citep{liang2020just}, a kernel-based interpolation model that helps the understanding of benign overfitting phenomenon and over-parameterized models. Due to its solid theoretical foundation and straightforward algorithm, kernel ridgeless regression has continued to be widely studied in the machine learning community. 
Recent advancements have confirmed the phenomenon of benign overfitting, particularly in high-dimensional regimes \citep{hastie2022surprises,mei2022generalization}. However, it is worth noting that such results often rely on the assumption of input dimensions tending towards infinity, a condition not always reflective of real-world datasets and target functions. Contrarily, in scenarios with lower dimensions \citep{buchholz2022kernel} or fixed-dimensional setups \citep{beaglehole2023inconsistency}, interpolating kernel machines do not exhibit the benign overfitting phenomenon for commonly used kernels like Gaussian, Laplace, and Cauchy kernels.

\textbf{RBF kernels with diverse bandwidths.}
RBF kernels that allow for different bandwidths in local regions have been investigated for a long time in the fields of kernel regression and kernel density estimation, e.g. \cite{Abramson1982OnBV, Brockmann1993LocallyAB, Zheng2013AdaptivelyWK}.
These pioneer works have focused on the selection of optimal bandwidths from data and have demonstrated that locally adaptive bandwidth estimators perform better than global bandwidth estimators in both theory and simulation studies. However, due to limited computing power and problem settings, these works have primarily analyzed one-dimensional algorithms and have not considered the generalization ability. In the field of machine learning, works like \cite{steinwart2016learning, hang2021optimal,radhakrishnan2022feature} propose the use of feature-adaptive bandwidths and theoretically validate the improvement of flexible bandwidths. Successful experimental attempts  have been achieved by directly applying asymmetric kernel functions \citep{moreno2003kullback, koide2006asymmetric} or by incorporating them into existing kernel-based learning models along with asymmetric metric learning \citep{wu2010asymmetric, pintea2018asymmetric}. However, many of these studies lack a robust theoretical explanation, leaving the meaning of corresponding models and hypothesis space (no longer RKHS) still unknown. In this paper, for the first time, we demonstrate that LAB RBF kernels are actually involved with the integral space of RKHSs.

\begin{sidewaystable}[htbp]\tiny
\caption{Mean and standard of $R^2\;(\uparrow)$  of different regression methods on real datasets. The number of support vectors of different kernel-based methods are presented in blanket. The best and second-best results are indicated in bold and italic, respectively.}
\label{tab: exp-1}
\centering
\begin{threeparttable}[b]
\begin{tabular}{c|c|c|c|c|c|c}
\toprule
{\multirow{2}{*}{Dataset}}    & Tecator     & Yacht       & Airfoil   & SML     & Parkinson       & Comp-active\tnote{a}   \\\cline{2-7}
&{N=240,M=122} & {N=308,M=6}  & {N=1503,M=5}  & {N=4137,M=22} & {N=5875,M=20}  & {N=8192,M=21}  \\
\hline
RBF KRR   & 0.9586$\pm$0.0071 (192) &   0.9889$\pm$0.0025 (247) & 0.8634$\pm$0.0248 (1203) &  0.9779$\pm$0.0013 (3310) & 0.8919$\pm$0.0091 (4700) & 0.9822 (6554)\\\hline
TL1 KRR   &  0.9670$\pm$0.0113 (192) & 0.9705$\pm$0.0033 (247) & 0.9464$\pm$0.0065 (1203) &  0.9947$\pm$0.0005 (3310) & 0.9475$\pm$0.0034 (4700) & 0.9801 (6554) \\\hline
R-SVR-MKL    & 0.9711$\pm$0.0212 (174.2) &  0.9945$\pm$0.0008 (224.7) &  0.9201$\pm$0.01099 (953.1) & 0.9959$\pm$0.0006 (2844)  &  0.9032$\pm$0.0122 (4047) &   0.9834 (1397) \\\hline
{SVR-MKL}    & 0.9698$\pm$0.0157 (160.7)    &  0.9957$\pm$0.0022  (144.5) &   0.9535$\pm$0.0042 (1035) & 0.9970$\pm$ 0.0006 (1424) &  0.9011$\pm$0.0110 (3759)
  &  0.9829 (1423) \\\hline
EigenPro3.0    & 0.9758$\pm$0.0029 (192) & {0.9944$\pm$0.0036} (247) & 0.9262$\pm$0.0166 (1203) & 0.9934$\pm$0.0009 (3310) & 0.9260$\pm$0.0079 (4700) & 0.9830 (6554) \\\hline
RFMs    & 0.9811$\pm$0.0078 (192)& {0.9947$\pm$0.0018} (247) & 0.9394
$\pm$0.0079  (1203) & 0.9960$\pm$0.0007 (3310)  & \textit{0.9988$\pm$0.0004 (4700)} & \textbf{0.9852 (6554)}\\\hline
Falkon    & {0.9769$\pm$0.0086} (100) & \textbf{0.9982$\pm$0.0024 (200)} & 0.9377
$\pm$0.0067 (900)   & 0.9960$\pm$0.0007 (2000) & 0.9492$\pm$0.0063 (4000)& 0.9808 (1500)\\\hline
ResNet    & \textit{0.9841$\pm$0.0067}   & 0.9940$\pm$0.0003   &   \textit{0.9538$\pm$0.0066}  &  \textit{0.9976$\pm$0.0004}   &  0.9906$\pm$0.0048 &  \textit{0.9836 } \\\hline
WNN    & \textbf{0.9875$\pm$0.0044} & 0.9924$\pm$0.0025  & 0.9128$\pm$0.0089   & 0.9926$\pm$0.0008  & 0.9139$\pm$0.0055 & 0.9817  \\\hline
{LAB RBF}  &  0.9782$\pm$0.0151 (76) &   \textit{0.9980$\pm$0.0009 (73)}  &   \textbf{0.9649$\pm$0.0091 (800)} & \textbf{0.9985$\pm$0.00007 (400)}  & \textbf{ 0.9990$\pm$0.0005} (400) &  0.9835 (70)           \\
\bottomrule
\end{tabular}
    \begin{tablenotes}
        \item [a] The test set of Comp-activ is pre-given. 
        \item  Notations N, M denote the data number and the feature dimension, respectively.
    \end{tablenotes}
\end{threeparttable}

\caption{Mean and standard of $R^2(\uparrow)$ of different algorithms in large-scale datasets. The number of support vectors of different kernel-based methods are presented in blanket. The best and second-best results are indicated in bold and italic, respectively.}
\label{tab: exp-2} 
\begin{threeparttable}
\begin{tabular}{c|c|c|c}
\toprule
{\multirow{2}{*}{Dataset}}& Electrical        & KC House        & TomsHardware                                \\ \cline{2-4} 
& N=10000,M= 11                                 & N=21623,M=14                                 & N=28179,M=96                                \\ 
                         \hline
{WNN}     & 0.9617$\pm$0.0027                           & 0.8501$\pm$0.0194                           & 0.9248$\pm$0.0303                           \\ \hline
{ResNet}     & \textbf{0.9705$\pm$0.0025} & 0.8823$\pm$0.0117                          & \textit{ 0.9697$\pm$0.0021  }                         \\ \hline
{EigenPro3.0 }   &                     0.9513$\pm$0.0024 (8000)
	&0.8636$\pm$0.0119 (17291)
&	0.9436$\pm$0.0282 (20000)
\\  \hline
{RFMs}    & 0.9582$\pm$0.0028   (8000)                 & \textit{0.9008$\pm$0.0072  (17291)}           & 0.9115$\pm$0.0031  (22544)      \\ \hline
{Falkon}        &     0.9532$\pm$0.0025    (3000)                      & 0.8640$\pm$0.0145  (5000)                       & 0.9001$\pm$0.0143 (5000)                         \\  \hline
{LAB RBF }  & \textit{0.9654$\pm$0.0034  (300})                   & \textbf{0.9103$\pm$0.0061(550)} & \textbf{0.9809$\pm$0.0028 (500)}\\  
\bottomrule
\end{tabular}    
\end{threeparttable}
\end{sidewaystable}

\textbf{Asymmetric kernel-based learning.}
Existing research in asymmetric kernel learning has primarily proposed frameworks based on SVD \citep{suykens2016svd} and least square SVM \citep{he2023learning}.
However, for regression tasks, current works \cite{mackenzie2004asymmetric, pintea2018asymmetric} directly incorporate asymmetric kernels into symmetric-kernel-based learning models, lacking interpretability. Additionally, other works primarily focus on interpreting associated optimization models \citep{wu2010asymmetric,lin2022reproducing}, where the corresponding functional space is regarded as a Reproducible Kernel Banach Space. This, however, is not currently applicable to LAB RBF kernels, as their reproducible property remains undetermined.
Despite notable progress in theory, current applications of asymmetric kernel matrices often rely on datasets (e.g. the directed graph in \cite{he2023learning}) or recognized asymmetric similarity measures (e.g. the Kullback-Leibler kernels in \cite{moreno2003kullback}) 
This yields improved performance in specific scenarios but leaving a significant gap in addressing diverse datasets.
With the help of trainable LAB RBF kernels, this paper proposes a robust groundwork for utilizing asymmetric kernels in tackling general regression tasks.
\iffalse
Among the existing methods, LAB RBF kernels have a close relationship with MKL, especially in theory. 
MKL selects kernels from a kernel dictionary, while our learning algorithm based on LAB RBF kernels extends this kernel selection to an infinite dictionary. As a result, our hypothesis space is extended from a direct sum of finite RKHSs to the integral space of RKHSs, which significantly enhances kernel flexibility, thanks to the deep and non-linear parameters involved. Even though more complex function spaces are considered, the computational complexity of the LAB RBF kernel is lower than that of MKL, thanks to the deep and non-linear parameters involved. 
Recent advanced works in MKL have shown that to approximate the training data, the number of valid kernels is at an order of the number of training data (cf. \cite{aziznejad2021multikernel}). 
This conclusion highlights the relation between data and kernel numbers, as well as raises concerns about computational complexity, as it involves $N$ $N \times N$ kernel matrices. In LAB RBF kernels, there are $M \times N_{sv}$ bandwidths to train, and we can approximate the training data with relatively fewer support data, i.e., $N_{sv} \leq N$. Moreover, only one $N_{sv} \times N_{sv}$ kernel matrix is needed, leading to high computational efficiency.
\fi

\section{Conclusion}\label{Chapt:conclusion}
In this paper, we enhance the kernel ridgeless regression with trainable LAB RBF kernels and investigated it from the approximation theory viewpoint. The LAB RBF kernel is highly flexible due to its over-parameterized form, where the bandwidths are data-adaptive and can vary depending on the size of the training data. While the 1-dimensional case of the LAB RBF kernel has been previously studied in statistics, its high-dimensional case and application in machine learning had not been explored. 
We presented an iterative learning algorithm based on a ridgeless model to determine the bandwidths of the LAB RBF kernel, with controllable support vectors and applicable gradient methods for training the bandwidths.
Experimental results on real regression datasets show that our algorithm achieves state-of-the-art accuracy.
This demonstrates the benefits of increasing kernel flexibility, and verifies the effectiveness of our proposing learning algorithm.

To investigate the source of the generalization ability in the proposed model without explicit regularization, we introduced the $\ell_0$-regularized model in the integral space of RKHSs. The optimal function of this model is equivalent to the interpolation function derived by a well-learned LAB RBF kernel.  Through the analysis of this model, we gained insights into the advantages of kernel ridgeless regression with LAB RBF kernels. Our theoretical analysis was based on the standard error decomposition technique, where we utilized the latest results on $\ell_q$-regularization models, conclusions on Rademacher chaos complexity of Gaussian kernels,  and a refined iteration technique for $\ell_0$-regularization. We demonstrated that at the optimal point of our proposed $\ell_0$-regularized model, the integral space of RKHSs reduces to a sum space of RKHSs, enabling us to bound the sample error in a complex space.

Our analysis revealed that the excellent representation ability of the proposed model is due to the large hypothesis spaces introduced by  LAB RBF kernels, i.e., the integral space of RKHSs, which enables our algorithm to interpolate the training dataset with only a few support vectors. Meanwhile, the natural sparsity of LAB RBF kernels, controlled by the number of support vectors, guarantees their good generalization ability, as seen from the analysis of the sample error. The number of support vectors plays a crucial role in the trade-off between the approximation ability in the training data and the generalization ability in the test data, which is also validated by our experimental results.

Considering the fundamental role of non-Mercer and asymmetric kernels in modern deep learning architectures like transformers \citep{wright2021transformers,chen2024primal}, we hope our analysis of kernel ridgeless regression and LAB RBF kernels will inspire further research on asymmetric kernel learning and the integral space of RKHSs in machine learning.

% Acknowledgements and Disclosure of Funding should go at the end, before appendices and references

\acks{The author would like to thank Mingzhen He for his insightful suggestions on this work. 
The research leading to these results received funding from the European Research Council under the European Union's Horizon 2020 research and innovation program / ERC Advanced
Grant E-DUALITY (787960). This paper reflects only the authors' views and the Union is not liable for any use that may be made of the contained information.
This work was also supported in part by Research Council KU Leuven: iBOF/23/064; Flemish Government (AI Research Program). Johan Suykens is also affiliated with the KU Leuven Leuven.AI institute.
Additionally, this work received partial support from the National Natural Science Foundation of China under Grants 62376155 and 12171093, as well as from the Shanghai Science and Technology Program under Grants 22511105600, 20JC1412700, and 21JC1400600. Further support was obtained from the Shanghai Municipal Science and Technology Major Project under Grant 2021SHZDZX0102. 
Xiaolin Huang and Lei Shi are the corresponding authors.}


\newpage

\appendix
\addcontentsline{toc}{section}{Appendices}
\renewcommand{\thesection}{\Alph{section}}

\section{Proof of Theorem~\ref{the: solution}} \label{apdx: 0}
Here we present the proof of Theorem~\ref{the: solution}.
\begin{proof}
    Based on Equation~(\ref{equ: eq}),    
    we can express $\vw^*$ as:
    \begin{equation}\label{equ: w*}
    \begin{aligned}        
        \vw^* &= \psi(\vX)(\phi^\top(\vX)\psi(\vX)+\lambda \vI_N)^{-1}\vY 
        &\overset{(a)}= (\lambda \vI_F + \psi(\vX)\phi^\top(\vX))^{-1} \psi(\vX)\vY
    \end{aligned}
    \end{equation}
    where equation (a) is derived from (\ref{equ: eq}) with $\vA=\vI_F, \; \vB=\psi(\vX),\; \vC =\phi^\top(\vX),\; \vD=\vI_N$.
    Similarly, for $\vv^*$, we have
    \begin{equation}\label{equ: v*}
        \begin{aligned}
            \vv^* &= \phi(\vX)(\psi^\top(\vX)\phi(\vX)+\lambda \vI_N)^{-1}\vY
            &\overset{(b)}=(\lambda \vI_F + \phi(\vX)\psi^\top(\vX))^{-1} \phi(\vX)\vY,
        \end{aligned}
    \end{equation}
    where equation (b) again applies (\ref{equ: eq}) with $\vA=\vI_F, \; \vB=\phi(\vX),\; \vC =\psi^\top(\vX),\; \vD=\vI_N$.
    Take the derivation of the objective function with respect to $\vw$ and $\vv$ at point $(\vw^*,\vv^*)$, we observe:
    \begin{equation*}
        \begin{aligned}
            \frac{\partial L}{\partial \vw}|_{\vw= \vw^*\atop \vv=\vv^*} &=(\lambda \vI_F + \phi(\vX)\psi^\top(\vX))(\lambda \vI_F + \phi(\vX)\psi^\top(\vX))^{-1} \phi(\vX)\vY -\phi(\vX) \vY =0,\\
            \frac{\partial L}{\partial \vv}|_{\vw= \vw^*\atop \vv=\vv^*}
            &=(\lambda \vI_F + \psi(\vX)\phi^\top(\vX))(\lambda \vI_F + \psi(\vX)\phi^\top(\vX))^{-1} \psi(\vX)\vY -\psi(\vX) \vY=0.
        \end{aligned}
    \end{equation*}
    This verifies that the point $(\vw^*,\vv^*)$ satisfies the stationarity condition.
\end{proof}

\section{Alternative derivation of Asymmetric KRR and Function Explanation}\label{apdx: f explanation}
%As classical KRR have more than one models, 
We can also derive a similar result in Theorem~\ref{the: solution} in a LS-SVM-like approach \citep{suykens1999least}, from which we can better understand the relationship between the two regression functions.
%By introducing error variables $e_i = y_i - \phi(\vx_i)^\top\vw$ and $r_i = y_i - \psi(\vx_i)^\top\vv$, the last three terms in (\ref{equ: asy krr}) equals to $\sum_i( e_i^2 +  r_i^2 - (e_i-r_i)^2) $, which further equals to $\sum_i e_ir_i$.
By introducing error variables $e_i = y_i - \phi(\vx_i)^\top\vw$ and $r_i = y_i - \psi(\vx_i)^\top\vv$, the last term in (\ref{equ: asy krr}) equals to $\sum_i e_ir_i$.
According to this result, we have the following optimization:
\begin{equation}\label{equ: ls akrr}
    \begin{aligned}        \min_{\vw,\vv,\ve,\vr} \;\;&\lambda \vw^\top\vv + \sum_{i=1}^N e_i r_i\\
    \mathrm{s.t.}\;\;& r_i = y_i - \psi(\vx_i)^\top\vv,\quad \forall i = 1,2,\cdots, N,\\
    &e_i = y_i - \phi(\vx_i)^\top\vw,\quad \forall i = 1,2,\cdots, N.
    \end{aligned}
\end{equation}
From the Karush-Kuhn-Tucker (KKT) conditions \citep{boyd2004convex}, we can obtain the following result on the KKT points.
\begin{theorem}\label{the: solution2}
Let $\valpha = [\alpha_1,\cdots,\alpha_N]^\top \in\mathbb{R}^N$ and $\vbeta= [\beta_1,\cdots,\beta_N]^\top \in\mathbb{R}^N$ be Lagrange multipliers of constraints  $r_i = y_i - \psi(\vx_i)^\top\vv$ and $e_i = y_i - \phi(\vx_i)^\top\vw, \forall i=1,\cdots,N$, respectively. Then one of the KKT points of (\ref{equ: ls akrr}) is
\begin{align*}
    &\vw^* = \frac{1}{\lambda}\psi(\vX)\valpha^* ,   \quad\quad 
    \vv^* = \frac{1}{\lambda}\phi(\vX)\vbeta^*, \\
    &\ve^*=\valpha^*= \lambda(\phi^\top(\vX)\psi(\vX)+\lambda \vI_N)^{-1}\vY,\\
    &\vr^*=\vbeta^* = \lambda(\psi^\top(\vX)\phi(\vX)+\lambda \vI_N)^{-1}\vY.    
\end{align*}
\end{theorem}
The proof is presented in Appendix~\ref{apdx: 1}.
This model shares a close relationship with existing models. For instance, by modifying the regularization term from $\vw^\top\vv$ to $\vw^\top\vw+\vv^\top\vv$ and flipping the sign of $\sum_{i=1}^N e_i r_i$, we arrive at the kernel partial least squares model as outlined in \cite{hoegaerts2004primal}. In the specific case where $\psi=\phi$, its KKT conditions align with those of the LS-SVM setting for ridge regression \citep{Saunders1998RidgeRL,Suykens2002LeastSS}. Furthermore, under the same condition of $\psi=\phi$ and when the regularization parameter is set to zero, it reduces to ordinary least squares regression \citep{hoegaerts2005subset}.

With the aid of error variables $\ve$ and $\vr$, a clearer perspective on the relationship between $f_1$ and $f_2$ emerges. As clarified in Theorem~\ref{the: solution2}, the approximation error on training data is equal to the value of dual variables, a computation facilitated through the kernel trick. Consequently, this reveals that $f_1$ and $f_2$ typically diverge when $\phi$ and $\psi$ are not equal, as they exhibit distinct approximation errors.
A complementary geometric insight arises from the term $\sum_{i=1}^N e_i r_i$ within the objective function. This signifies that, in practice, $f_1$ and $f_2$ tend to approach the target $y$ from opposite directions because the signs in their approximation errors tend to be dissimilar. For practical applications, one may opt for the regression function with the smaller approximation error.

\section{Proof of Theorem~\ref{the: solution2}} \label{apdx: 1}
Here we present the proof of Theorem~\ref{the: solution2}.
\begin{proof}

The Lagrangian of (\ref{equ: ls akrr}) is 
\begin{equation}
   \mathcal{L} = \lambda \vw^\top\vv + \sum_{i=1}^N e_i r_i + \sum_i \beta_i(y_i - e_i -\phi(\vx_i)^\top\vw) + \sum_i \alpha_i(y_i - r_i -\psi(\vx_i)^\top\vv),
\end{equation}
where $\valpha\in\mathbb{R}^N$ and $\vbeta\in\mathbb{R}^N$ are Lagrange multipliers.
The KKT conditions lead to
\begin{align*}
    &\frac{\partial \mathcal{L}}{\partial \vv} = \lambda \vw -\psi(\vX)\valpha = 0 && \Longrightarrow  \vw = \frac{1}{\lambda}\psi(\vX)\valpha,\\
    &\frac{\partial \mathcal{L}}{\partial \vw} = \lambda \vv -\phi(\vX)\vbeta = 0 && \Longrightarrow  \vv = \frac{1}{\lambda}\phi(\vX)\vbeta,\\
    &\frac{\partial \mathcal{L}}{\partial r_i} =e_i - \alpha_i = 0 && \Longrightarrow  e_i = \alpha_i,\\
    & \frac{\partial \mathcal{L}}{\partial e_i}= r_i - \beta_i = 0 && \Longrightarrow  r_i = \beta_i,\\
    &\frac{\partial \mathcal{L}}{\partial \beta_i} =y_i - e_i -\phi(\vx_i)^\top\vw = 0 && \Longrightarrow  e_i = y_i -\phi(\vx_i)^\top\vw,\\
    & \frac{\partial \mathcal{L}}{\partial \alpha_i}= y_i - r_i -\psi(\vx_i)^\top\vv = 0 && \Longrightarrow  r_i = y_i - \psi(\vx_i)^\top\vv.\\
\end{align*}
Substitute the first four lines into the last two lines, we can eliminate primal variables $\vw,\vv, \ve,\vr$:
\begin{align*}    
    &\valpha^* = \vY - \frac{1}{\lambda}\phi(\vX)^\top\psi(\vX)\valpha^*
    && \Longrightarrow
    \valpha^* = \lambda(\lambda \vI_N + \phi(\vX)^\top\psi(\vX))^{-1}\vY,\\
    &\vbeta^* = \vY - \frac{1}{\lambda}\psi(\vX)^\top\phi(\vX)\vbeta^*
    && \Longrightarrow
    \vbeta^* = \lambda(\lambda \vI_N + \psi(\vX)^\top\phi(\vX))^{-1}\vY.
\end{align*}
    Thus, we get the result in Theorem~\ref{the: solution2} and the proof is completed.
\end{proof}

\section{Proof of Proposition~\ref{lem: alg err}}\label{apdx: 5}
Here we present the proof of Proposition~\ref{lem: alg err}.

\begin{proof}
The proof is achieved by constructing a equivalent constrained version of optimization (\ref{equ: unconstrained}).
\begin{equation}\label{equ: constrained}
    \begin{aligned}     &f_\vz = \arg\min_{\substack{f\in\mathcal{H}_\Omega\\ \{\vtheta_i\}\subset\Omega}}
    \mathcal{E}_\vz({f}) \quad \;
    \mathrm{s.t.} \;\;\mathcal{R}_0(f) = N_{sv},
    \;\;{\mu}(\vsigma) = \sum_{i=1}^{N_{sv}} \delta(\vsigma - \theta_i).
    \end{aligned}
\end{equation}
We firstly show that  $f_{\mathcal{Z}_{sv},\vTheta}$ belongs to this integral space.
Recall that  $f_{\mathcal{Z}_{sv},\vTheta}$ has an analytical formulation as presented in (\ref{equ: lab-function}).
Let $\vTheta=\{\vtheta_i\}_{i=1}^{N_{sv}}$ denotes the bandwidth set of these support data.
Then there exists a interval $\Omega \subset \mathbb{R}^M_+$ satisfying that $\theta_i\in\Omega,\forall i$ as $\vTheta$ is a discrete set.
Without loss of generality, we assume that $\|\valpha\|_2 <\infty$.
Then define
\begin{equation}\label{equ: lab-function}
    \begin{aligned}
    &\tilde{f}_{\theta_i}(\vt) \triangleq \alpha_i\exp\{-\|\theta_i\odot(\vt-\vx_i)\|_2^2\}
    =\alpha_i \mathcal{K}_{\vtheta_i}(\vt,\vx_i)
    \end{aligned}
\end{equation}
and $\tilde{\mu}(\vsigma) = \sum_{i=1}^{N_{sv}} \delta(\vsigma - \theta_i)$, where $\delta(\cdot)$ denotes the Dirac delta function. 
Under this definition we have $f_{\mathcal{Z}_{sv},\vTheta}(\vx) = \int_\vsigma \tilde{f}_\vsigma(\vx) d\tilde{\mu}(\vsigma)$ and 
$$\int_{\vsigma\in\Omega}\|\tilde{f}_\vsigma\|_{\mathcal{H}_\vsigma} d\tilde{\mu}(\vsigma) 
=  \sum_{i=1}^{N_{sv}} \|\tilde{f}_{\theta_i}\|_{\mathcal{H}_{\theta_i}}
\leq \|{\valpha}\|_2 <\infty,$$
which indicates that $f_{\mathcal{Z}_{sv},\vTheta}\in\mathcal{H}_{\Omega}.$

The formulation in (\ref{equ: lab-function}) means that for every kernel $\mathcal{K}_{\vtheta_i}$, only one coefficient are non-zero.
Then recall the definition in (\ref{equ: def R0}), we have $\mathcal{R}_0(f_{\mathcal{Z}_{sv},\vTheta}) = N_{sv}$, indicating that $f_{\mathcal{Z}_{sv},\vTheta}$ is a feasible solution of problem (\ref{equ: constrained}).
recall the stopping condition in the dynamic strategy (\ref{equ: strategy}), we have 
$0<\mathcal{E}_\vz(f_{\mathcal{Z}_{sv},\vTheta})-\mathcal{E}_\vz({f}_{\vz})\leq B.$
Finally, as all constrain in (\ref{equ: constrained}) are equalities, we can always find a suitable $\lambda>0$ such that optimization (\ref{equ: unconstrained}) shares the same optimizer as that of (\ref{equ: constrained}).
That is, $f_{\vz} = f_{\vz,\lambda}$. Then, we obtain the desired conclusion and complete the proof.

\end{proof}
\section{Experiment Details}\label{apd: exp}
The used datasets can be download from:
\begin{itemize}\setlength\itemsep{0em}
    \item 
    \textbf{Tecator}: \url{http://lib.stat.cmu.edu/datasets/tecator}.
    \item 
    \textbf{Yacht}: \url{https://archive.ics.uci.edu/dataset/243/yacht+hydrodynamics}.
    \item 
    \textbf{Airfoil}: \url{https://archive.ics.uci.edu/dataset/291/airfoil+self+noise}.
    \item 
    \textbf{SML}: \url{https://archive.ics.uci.edu/dataset/274/sml2010}.
    \item 
    \textbf{Parkinson}: \url{https://archive.ics.uci.edu/dataset/189/parkinsons+telemonitoring}.
    \item 
    \textbf{Comp-activ}:  \url{https://www.cs.toronto.edu/~delve/data/comp-activ/desc.html}.
    \item 
    \textbf{TomsHardware}: \url{https://archive.ics.uci.edu/dataset/248/buzz+in+social+media}.
    \item 
    \textbf{KC House}: \url{https://www.kaggle.com/datasets/shivachandel/kc-house-data}.
    \item 
    \textbf{Electrical}:  \url{https://archive.ics.uci.edu/dataset/471/electrical+grid+stability+simulated+data}.
    \item 
    \textbf{MNIST}:\url{http://yann.lecun.com/exdb/mnist/}
    \item 
    \textbf{Fashion-MNIST}: \url{https://github.com/zalandoresearch/fashion-mnist}
\end{itemize}

\section{Details of Compared methods and hyper-parameter setting. } \label{apdx: 3}
\textbf{Compared methods:} nine regression methods are compared in this experiment, including:  
\begin{itemize}\setlength\itemsep{0em}
    \item RBF KRR \citep{vovk2013kernel}: classical kernel ridge regression with conventional RBF kernels, served as the baseline.
    \item TL1 KRR: classical kernel ridge regression employing an indefinite kernel named Truncated $\ell_1$ kernel \citep{Tl12018Huang}. The expression of TL1 kernel is 
    $     \mathcal{K}(\vx,\vx')=\max\{\rho - \|\vx-\vx'\|_1, 0\},    $
    where $\rho>0$ is a pre-given hyper-parameter. The TL1 kernel is a piecewise linear indefinite kernel and is expected to be more flexible and have better performance than the conventional RBF kernel.
    \item SVR-MKL: Multiple kernel learning applied on support vector regression. The kernel dictionary includes RBF kernels, Laplace kernels, and polynomial kernels. Results for R-SVR-MKL with only RBF kernels are also provided. The implementation of MKL is available in the Python package MKLpy \citep{EasyMKL,lauriola2020mklpy}.
    \item Falkon \citep{rudi2017falkon, meanti2022efficient}: An advanced and well-developed algorithm for KRR that employs hyper-parameter tuning techniques to enhance accuracy and utilizes Nyström approximation to reduce the number of support data points, enabling it to handle large-scale datasets. We used the public code of Falkon, available at \url{https://github.com/FalkonML/falkon}.
    \item EigenPro3.0 \citep{Eigenpro}: An advanced general kernel machine for large datasets, utilizing Nyst\"om methods and projected dual preconditioned SGD. We used the public code of EigenPro3.0, available at \url{https://github.com/EigenPro/EigenPro3}.
    \item RFMs \citep{radhakrishnan2022feature}: Recursive feature machines is advanced kernel methods which utilizes the mechanism of deep feature learning, resulting high efficient algorithms and ability to handle large datasets. We used the public code of RFMs, available at \url{https://github.com/aradha/recursive_feature_machines}.
    \item ResNet: The regression version of ResNet follows the structure in \cite{chen2020deep}, and the code is available in \url{https://github.com/DowellChan/ResNetRegression}.
    \item WNN: The regression version of a wide neural network, which is fully-connected and has only one hidden layer. 
\end{itemize}

\textbf{Implementation details.}
Among the compared methods, Kernel Ridge Regression (KRR) stands as the fundamental technique that combines the Tikhonov regularized model with the kernel trick. The coefficients of kernels for both SVR-MKL and R-SVR-MKL are calculated following the approach in EasyMKL \citep{EasyMKL}. For Falkon, the code is available at \url{https://github.com/FalkonML/falkon}. LAB RBF, ResNet, and WNN are optimized using gradient methods with varying hyper-parameters such as initial points, learning rate, and batch size. The initial weights of both ResNet and WNN are set according to the Kaiming initialization introduced in \cite{He2015DelvingDI}. In the subsequent experiments, the Adam optimizer is initially used, and upon stopping, the SGD optimizer is applied. Early stopping is implemented for the training of ResNet and WNN, where $10\%$ of the training data is sampled to form a validation set, and validation loss is assessed every epoch. The epoch with the best validation loss is selected for testing. Detailed hyper-parameters of all compared methods are provided in Table~\ref{tab: exp-setting} (for small-scale datasets) and Table~\ref{tab: exp-setting-2} (for large-scale datasets).

The regression version of ResNet follows the structure in \cite{chen2020deep}, which has available code in \url{https://github.com/DowellChan/ResNetRegression}.
Following the structures in \cite{chen2020deep}, the ResNet block has two types: Identity Block (where the dimension of input and output are the same) and Dense Block (where the dimension of input and output are different).
The details of these two block are presented in Figure~\ref{fig: resnet}.
Considering the different dataset sizes, we use two structures of ResNet in our experiments, denoted by ResNet and ResNetSmall.
For the ResNet, we use two Dense Blocks (M-W-100) and two Identity Block (100-100-100) and a linear predict layer (100-1).
For the ResNetSmall, we use two Dense Blocks (M-W-50)  and a linear predict layer (50-1).
Here W is a pre-given width for the network.

\begin{figure}
    \centering    \includegraphics[width=0.8\textwidth]{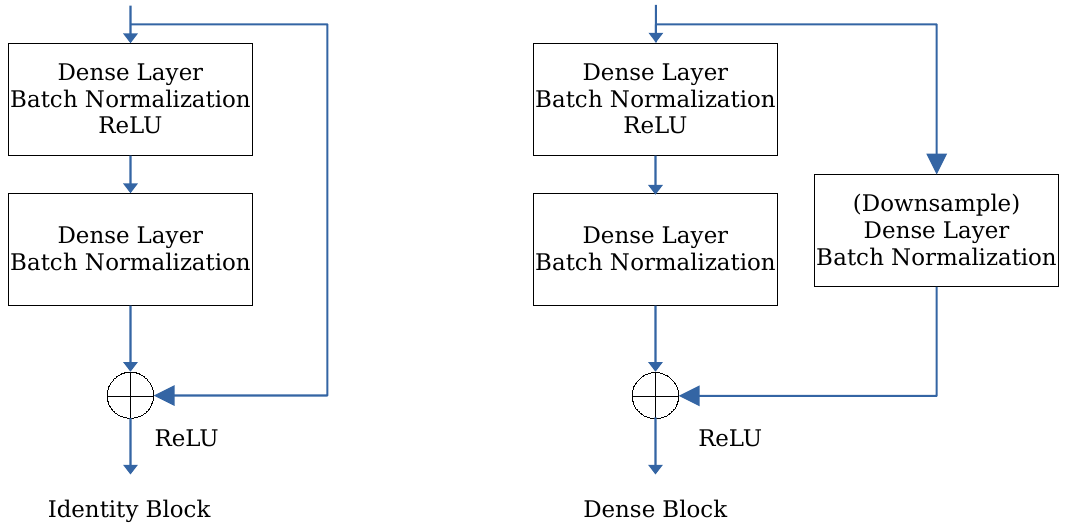}
    \caption{The structures of Identity block and Dense block.}
    \label{fig: resnet}
\end{figure}

\begin{table}[tbp]\scriptsize
    \caption{Hyper-parameters of eight regression methods for real datasets.}
    \label{tab: exp-setting}
    \centering
    \begin{threeparttable}[b]
    \begin{tabular}{c|c|c|c|c|c|c|c}
        \toprule
        &  Hyper-parameters & Tecator & Yacht & Airfoil & SML & Parkinson & Comp\_activ \\ \midrule
        \multirow{3}{*}{LAB RBF} & lr &   0.001 &  0.01  &  0.01  &  0.05   & 0.001 &    0.001 \\ \cline{2-8}
        & Batch size &   16  & 128&   128 &128 &128      &  128        \\ \cline{2-8}
        & $\sigma_0$ &  0.5  &   3  &   10  &  50 &     30   &   0.1     \\ \hline
\multirow{3}{*}{R-SVR-MKL}& C &   1000      &   1000    &     1    &   1000  &    10       &       1      \\ \cline{2-8}
        & $\epsilon$ &    0.001    &   0.001   &    0.01   &   0.001  &   0.01     &    0.01      \\ \cline{2-8} 
        &   Dictionary & \multicolumn{6}{|c} {RBF kernels: $[100,50,10,1,0.1,0.01,0.001]$}
          \\ \hline
\multirow{5}{*}{SVR-MKL}& C &   1000      &   1000    &     100    &   1000  &    1000      &       1000    \\ \cline{2-8}
        & $\epsilon$ &    0.01     &    0.01   &     0.01   &   0.01  &   0.01     &    0.01         \\ \cline{2-8} 
        &  \multirow{3}{*}{Dictionary} & \multicolumn{6}{|c} {RBF kernels: $[100,1,0.1,0.001]$} \\ \cline{3-8} 
        &   & \multicolumn{6}{|c} {Laplace kernels: $[100,1,0.1,0.001]$} \\ \cline{3-8} 
        &   & \multicolumn{6}{|c} {Polynomial kernels: $[1,2,4,10]$} \\ \hline     
\multirow{2}{*}{RBF KRR} & $\sigma$ &    1     &    5   &      80   &  5   &      20     &    10         \\ \cline{2-8}
        & $\lambda$ &     0.01    &  0.001     &   0.001   &   0.01  &   0.001    &    0.001         \\ \hline                                   
\multirow{2}{*}{TL1 KRR }& $\rho$&   98      &    6   &  2.5       &  22   &       14    &     15        \\ \cline{2-8}
        & $\lambda$ &   0.001      &  0.001    &  0.001       &   0.1  &      0.01     &    0.001   \\ \hline
\multirow{3}{*}{Falkon}   & $\lambda$ &      1e-6   &  1e-7     &    1e-6   &   1e-5  &        1e-7   &   1e-6     \\ \cline{2-8}
& Center\tnote{a} &     100    &  200     &    900    &   2000 &   4000    &     1500   \\ \cline{2-8}
& $\sigma$    &     10    &   1    &   2    &  1   &    0.7     &   2.5    \\ \hline
\multirow{2}{*}{EigenPro3.0}   & $\sigma$ &      3   &  1     &    0.5   &   1  &    0.5   &  10    \\ \cline{2-8}
& Center\tnote{a} &     197    &  247     &    1203    &   3310 &   4700    &     6554   \\ \hline
%\multirow{2}{*}{RFMs}   & $\sigma$ &      0.5  &  0.5     &   0.5   &   5 &       5  &  5    \\ \cline{2-8}
%& Center\tnote{a} &     100    &  200     &    900    &   2000 &   4000    &     1500   \\ \hline
\multirow{3}{*}{ResNet}   & lr &      0.001   &  0.001     &    0.001   &   0.001  &        0.001   &   0.001      \\ \cline{2-8}
        & Batch size &     32    &  32     &    128    &    128&   128    &      128   \\ \cline{2-8}
        & Structure (Width)\tnote{b}    &     2(500)    &   2(500)    &   1(1000)    &  1(1000)   &     1(500)     &   1(2000)    \\ \hline
\multirow{3}{*}{WNN}   & lr &      0.001   &  0.001     &    0.001   &   0.001  &        0.001   &   0.001      \\ \cline{2-8}
        & Batch size &     32    &  32     &    128    &    128&   128    &      128   \\ \cline{2-8}
        & Width    &     800    &   500    &   6000    &  1500   &     3000     &   9000 \\\bottomrule
    \end{tabular}
    \begin{tablenotes}
    \item [a] The center number of Nystr\"om approximation.
        \item [b] Structure 1:$M-W-100-100-100-1$, Structure 2: $M-W-50-1$.
    \end{tablenotes}
\end{threeparttable}
\end{table}

\begin{table}[tbp]\small
    \caption{Hyper-parameters of four regression methods for real datasets.}
    \label{tab: exp-setting-2}
    \centering
    \begin{threeparttable}
    \begin{tabular}{c|c|c|c|c}
        \toprule
        &  Hyper-parameters & TomsHardware  & Electrical & KC House \\ \midrule
        \multirow{3}{*}{LAB RBF} & lr &   0.001 &   0.001  &   0.001  \\ \cline{2-5}
        & Batch size &   256  & 256&   256      \\ \cline{2-5}
        & $\sigma_0$ &  0.1 &  0.1  &   1     \\ \hline
        \multirow{3}{*}{Falkon}   & $\lambda$ &     1e-6  &  1e-6    &    1e-6    \\ \cline{2-5}
        & Center &     3000    &  3000     &    5000    \\ \cline{2-5}
        & $\sigma$    &    2   &   10    &  5  \\ \hline
        \multirow{2}{*}{EigenPro3.0}   & $\sigma$ &    7  &  1    &    5   \\ \cline{2-5}
        & Center &     20000    &  8000    &    17291   \\ \hline
        %\multirow{2}{*}{RFMs}   & $\sigma$ &     5 &  5   &  5    \\ \cline{2-5}
        %& Center &     3000    &  3000     &    5000    \\ \hline
        \multirow{3}{*}{ResNet}   & lr &      0.001   &  0.001     &    0.001     \\ \cline{2-5}
        & Batch size &     128    &  256     &    256    \\ \cline{2-5}
        & Structure (Width)\tnote{a}    &     1(3000)   &   1(2000)    &  1(2000)   \\ \hline
\multirow{3}{*}{WNN}   & lr &      0.01   &  0.01     &    0.01 \\ \cline{2-5}
        & Batch size &     128    &  256     &    128   \\ \cline{2-5}
        & Width    &     3000    &   3000    &   5000   \\\bottomrule
    \end{tabular}
        \begin{tablenotes}
        \item [a] Structure 1:$M-W-100-100-100-1$.
    \end{tablenotes}
\end{threeparttable}
\end{table}

\section{Study on different selection strategy of initial support data} \label{apdx: sv}

The selection of support data has the significant influence on the performance of the proposed algorithm. In light of this, we have introduced a dynamic strategy aimed at mitigating the impact of initial support data selection in the manuscript.

In this section, we delve deeper into the effects of various methods for selecting initial support data and assess the efficacy of the introduced dynamic strategy.
We will explore three different approaches to initial data selection: two rational methods (Y-based and X-based) and one irrational method (Extreme Y).
\begin{itemize}
    \item Y-based (utilized in the manuscript): data is sorted based on their labels, and support data is uniformly selected.
    \item X-based: k-means is applied to the training data to identify cluster centers, followed by the selection of data points closest to these centers.
    \item Extreme Y: data is sorted based on their labels, and those with the largest Y values are selected.
\end{itemize}
Table~\ref{tab: ini-sv} presents the performance of Algorithm~\ref{alg: AKL} with these selection methods on Yacht and Parkinson datasets.
The results indicate that the poor selection method does have a detrimental impact on our performance, particularly evident in the case of Yacht where we struggle to fit the data. In contrast, the other two sensible methods demonstrate good and comparable performance.

In order to further improve, we introduce a dynamic strategy at the end of Section 3. In this strategy, we dynamically incorporate hard samples into the support dataset. We then integrate these approaches with the proposed dynamic strategy to evaluate its effectiveness, of which the results are presented in Table~\ref{tab: ini-sv-dyn}.
Based on these results, it is evident that the proposed dynamic strategy has a significantly positive impact on performance. It not only enhances accuracy but also reduces variance, resulting in more stable solutions. Even with the bad selection selection, the final performance is improved to a satisfactory level.
\begin{table}[]
\centering
\caption{Performance of Algorithm1 with different selection methods of initial support data.}\label{tab: ini-sv}
\begin{tabular}{@{}c|c|c|c|c|c|c@{}}
\toprule
Dataset            & Yacht     & Yacht   & Yacht   & Parkinson & Parkinson & Parkinson \\ \hline
Selection Approach & Extreme Y & Y-based & X-based & Extreme Y & Y-based   & X-based   \\ \hline
Mean of  $R^2$         & 0.0012    & 0.9957  & 0.9953  & 0.8115    & 0.9921    & 0.9928    \\ \hline
Std of  $R^2$          & 0.4805    & 0.0025  & 0.0032  & 0.0126    & 0.0015    & 0.0016   \\
\bottomrule
\end{tabular}
\end{table}

% \usepackage{booktabs}
\begin{table}[]
\centering
\caption{Performance of Algorithm1 with dynamic strategy and different selection methods of initial support data.}\label{tab: ini-sv-dyn}
\begin{tabular}{@{}c|c|c|c|c|c|c@{}}
\toprule
Dataset            & Yacht     & Yacht   & Yacht   & Parkinson & Parkinson & Parkinson \\ \hline
Selection Approach & Extreme Y & Y-based & X-based & Extreme Y & Y-based   & X-based   \\ \hline
Mean of R2         & 0.9961    & 0.9982  & 0.9981  & 0.9712    & 0.9972    & 0.9966    \\ \hline
Std of R2          & 0.0126    & 0.0015  & 0.0016  & 0.0049    & 0.0007    & 0.0013   \\
\bottomrule
\end{tabular}
\end{table}

\vskip 0.2in
\bibliography{thesis}

\begin{thebibliography}{94}
\providecommand{\natexlab}[1]{#1}
\providecommand{\url}[1]{\texttt{#1}}
\expandafter\ifx\csname urlstyle\endcsname\relax
  \providecommand{\doi}[1]{doi: #1}\else
  \providecommand{\doi}{doi: \begingroup \urlstyle{rm}\Url}\fi

\bibitem[Abedsoltan et~al.(2023)Abedsoltan, Belkin, and Pandit]{Eigenpro}
A.~Abedsoltan, M.~Belkin, and P.~Pandit.
\newblock Toward large kernel models.
\newblock In \emph{Proceedings of the 40th International Conference on Machine Learning}, volume 202 of \emph{Proceedings of Machine Learning Research}, pages 61--78. PMLR, 23--29 Jul 2023.

\bibitem[Abramson(1982)]{Abramson1982OnBV}
I.~Abramson.
\newblock On bandwidth variation in kernel estimates-a square root law.
\newblock \emph{Annals of Statistics}, 10:\penalty0 1217--1223, 1982.

\bibitem[Adlam and Pennington(2020)]{adlam2020neural}
B.~Adlam and J.~Pennington.
\newblock The neural tangent kernel in high dimensions: Triple descent and a multi-scale theory of generalization.
\newblock In \emph{International Conference on Machine Learning}, pages 74--84. PMLR, 2020.

\bibitem[Aiolli and Donini(2015)]{EasyMKL}
F.~Aiolli and M.~Donini.
\newblock Easymkl: a scalable multiple kernel learning algorithm.
\newblock \emph{Neurocomputing}, 169:\penalty0 215--224, 2015.

\bibitem[Allen-Zhu et~al.(2019{\natexlab{a}})Allen-Zhu, Li, and Liang]{allen2019learning}
Z.~Allen-Zhu, Y.~Li, and Y.~Liang.
\newblock Learning and generalization in overparameterized neural networks, going beyond two layers.
\newblock \emph{Advances in neural information processing systems}, 32, 2019{\natexlab{a}}.

\bibitem[Allen-Zhu et~al.(2019{\natexlab{b}})Allen-Zhu, Li, and Song]{allen2019convergence}
Z.~Allen-Zhu, Y.~Li, and Z.~Song.
\newblock A convergence theory for deep learning via over-parameterization.
\newblock In \emph{International conference on machine learning}, pages 242--252. PMLR, 2019{\natexlab{b}}.

\bibitem[Aronszajn(1950)]{aronszajn1950theory}
N.~Aronszajn.
\newblock Theory of reproducing kernels.
\newblock \emph{Transactions of the American Mathematical Society}, 68\penalty0 (3):\penalty0 337--404, 1950.

\bibitem[Arzamasov(2018)]{misc_electrical_grid_stability_simulated_data__471}
V.~Arzamasov.
\newblock {Electrical Grid Stability Simulated Data }.
\newblock UCI Machine Learning Repository, 2018.
\newblock {DOI}: https://doi.org/10.24432/C5PG66.

\bibitem[Asuncion and Newman(2007)]{asuncion2007uci}
A.~Asuncion and D.~Newman.
\newblock {UCI} machine learning repository, 2007.

\bibitem[Bach(2022)]{bach2022information}
F.~Bach.
\newblock Information theory with kernel methods.
\newblock \emph{IEEE Transactions on Information Theory}, 69\penalty0 (2):\penalty0 752--775, 2022.

\bibitem[Bartlett et~al.(2020)Bartlett, Long, Lugosi, and Tsigler]{bartlett2020benign}
P.~L. Bartlett, P.~M. Long, G.~Lugosi, and A.~Tsigler.
\newblock Benign overfitting in linear regression.
\newblock \emph{Proceedings of the National Academy of Sciences}, 117\penalty0 (48):\penalty0 30063--30070, 2020.

\bibitem[Beaglehole et~al.(2023)Beaglehole, Belkin, and Pandit]{beaglehole2023inconsistency}
D.~Beaglehole, M.~Belkin, and P.~Pandit.
\newblock On the inconsistency of kernel ridgeless regression in fixed dimensions.
\newblock \emph{SIAM Journal on Mathematics of Data Science}, 5\penalty0 (4):\penalty0 854--872, 2023.

\bibitem[Belkin et~al.(2018)Belkin, Ma, and Mandal]{belkin2018understand}
M.~Belkin, S.~Ma, and S.~Mandal.
\newblock To understand deep learning we need to understand kernel learning.
\newblock In \emph{International Conference on Machine Learning}, pages 541--549. PMLR, 2018.

\bibitem[Boyd and Vandenberghe(2004)]{boyd2004convex}
S.~P. Boyd and L.~Vandenberghe.
\newblock \emph{Convex optimization}.
\newblock Cambridge university press, 2004.

\bibitem[Brockmann et~al.(1993)Brockmann, Gasser, and Herrmann]{Brockmann1993LocallyAB}
M.~Brockmann, T.~Gasser, and E.~Herrmann.
\newblock Locally adaptive bandwidth choice for kernel regression estimators.
\newblock \emph{Journal of the American Statistical Association}, 88:\penalty0 1302--1309, 1993.

\bibitem[Brooks et~al.(2014)Brooks, Pope, and Michael]{misc_airfoil_self-noise_291}
T.~Brooks, D.~Pope, and M.~Michael.
\newblock {Airfoil Self-Noise}.
\newblock UCI Machine Learning Repository, 2014.
\newblock {DOI}: https://doi.org/10.24432/C5VW2C.

\bibitem[Buchholz(2022)]{buchholz2022kernel}
S.~Buchholz.
\newblock Kernel interpolation in sobolev spaces is not consistent in low dimensions.
\newblock In \emph{Conference on Learning Theory}, pages 3410--3440. PMLR, 2022.

\bibitem[Cao et~al.(2022)Cao, Chen, Belkin, and Gu]{cao2022benign}
Y.~Cao, Z.~Chen, M.~Belkin, and Q.~Gu.
\newblock Benign overfitting in two-layer convolutional neural networks.
\newblock \emph{Advances in neural information processing systems}, 35:\penalty0 25237--25250, 2022.

\bibitem[Chatterji and Long(2021)]{chatterji2021finite}
N.~S. Chatterji and P.~M. Long.
\newblock Finite-sample analysis of interpolating linear classifiers in the overparameterized regime.
\newblock \emph{The Journal of Machine Learning Research}, 22\penalty0 (1):\penalty0 5721--5750, 2021.

\bibitem[Chen et~al.(2020)Chen, Hu, Nian, and Yang]{chen2020deep}
D.~Chen, F.~Hu, G.~Nian, and T.~Yang.
\newblock Deep residual learning for nonlinear regression.
\newblock \emph{Entropy}, 22\penalty0 (2):\penalty0 193, 2020.

\bibitem[Chen et~al.(2004)Chen, Wu, Ying, and Zhou]{Chen2004Support}
D.~R. Chen, Q.~Wu, Y.~Ying, and D.~X. Zhou.
\newblock Support vector machine soft margin classifiers: {Error} analysis.
\newblock \emph{Journal of Machine Learning Research}, 5\penalty0 (3):\penalty0 1143--1175, 2004.

\bibitem[Chen et~al.(2024)Chen, Tao, Tonin, and Suykens]{chen2024primal}
Y.~Chen, Q.~Tao, F.~Tonin, and J.~A.~K. Suykens.
\newblock Primal-attention: Self-attention through asymmetric kernel svd in primal representation.
\newblock \emph{Advances in Neural Information Processing Systems}, 36, 2024.

\bibitem[Cherkassky et~al.(1996)Cherkassky, Gehring, and Mulier]{Cherkassky1996Comparison}
V.~Cherkassky, D.~Gehring, and F.~Mulier.
\newblock Comparison of adaptive methods for function estimation from samples.
\newblock \emph{IEEE Transactions on Neural Networks}, 7\penalty0 (4):\penalty0 969--984, 1996.

\bibitem[Cucker and Zhou(2007)]{Cucker2007Learning}
F.~Cucker and D.~X. Zhou.
\newblock Learning {Theory}: An {Approximation Theory Viewpoint}.
\newblock \emph{Cambridge University Press Cambridge}, 2007.

\bibitem[Deng(2012)]{deng2012mnist}
L.~Deng.
\newblock The mnist database of handwritten digit images for machine learning research [best of the web].
\newblock \emph{IEEE Signal Processing Magazine}, 29\penalty0 (6):\penalty0 141--142, 2012.

\bibitem[Eberts and Steinwart(2011)]{eberts2011optimal}
M.~Eberts and I.~Steinwart.
\newblock Optimal learning rates for least squares svms using gaussian kernels.
\newblock \emph{Advances in neural information processing systems}, 24, 2011.

\bibitem[Gelman et~al.(2019)Gelman, Goodrich, Gabry, and Vehtari]{gelman2019r}
A.~Gelman, B.~Goodrich, J.~Gabry, and A.~Vehtari.
\newblock R-squared for {Bayesian} regression models.
\newblock \emph{The American Statistician}, 2019.

\bibitem[Gerritsma et~al.(2013)Gerritsma, Onnink, , and Versluis]{misc_yacht_hydrodynamics_243}
J.~Gerritsma, R.~Onnink, , and A.~Versluis.
\newblock {Yacht Hydrodynamics}.
\newblock UCI Machine Learning Repository, 2013.
\newblock {DOI}: https://doi.org/10.24432/C5XG7R.

\bibitem[Ghorbani et~al.(2020)Ghorbani, Mei, Misiakiewicz, and Montanari]{ghorbani2020neural}
B.~Ghorbani, S.~Mei, T.~Misiakiewicz, and A.~Montanari.
\newblock When do neural networks outperform kernel methods?
\newblock \emph{Advances in Neural Information Processing Systems}, 33:\penalty0 14820--14830, 2020.

\bibitem[G{\"o}nen and Alpayd{\i}n(2011)]{gonen2011multiple}
M.~G{\"o}nen and E.~Alpayd{\i}n.
\newblock Multiple kernel learning algorithms.
\newblock \emph{The Journal of Machine Learning Research}, 12:\penalty0 2211--2268, 2011.

\bibitem[Guo and Shi(2013)]{Guo2013Learning}
Z.~C. Guo and L.~Shi.
\newblock Learning with coefficient-based regularization and $\ell_1$-penalty.
\newblock \emph{Advances in Computational Mathematics}, 39\penalty0 (3-4):\penalty0 493--510, 2013.

\bibitem[Hang and Steinwart(2021)]{hang2021optimal}
H.~Hang and I.~Steinwart.
\newblock Optimal learning with anisotropic gaussian svms.
\newblock \emph{Applied and Computational Harmonic Analysis}, 55:\penalty0 337--367, 2021.

\bibitem[Harlfoxem(2016)]{harlfoxem2016house}
Harlfoxem.
\newblock House sales in king county, usa.
\newblock 2016.
\newblock URL \url{https://www.kaggle.com/harlfoxem/housesalesprediction}.

\bibitem[Hastie et~al.(2022)Hastie, Montanari, Rosset, and Tibshirani]{hastie2022surprises}
T.~Hastie, A.~Montanari, S.~Rosset, and R.~J. Tibshirani.
\newblock Surprises in high-dimensional ridgeless least squares interpolation.
\newblock \emph{Annals of statistics}, 50\penalty0 (2):\penalty0 949, 2022.

\bibitem[He et~al.(2015)He, Zhang, Ren, and Sun]{He2015DelvingDI}
K.~He, X.~Zhang, S.~Ren, and J.~Sun.
\newblock Delving deep into rectifiers: Surpassing human-level performance on imagenet classification.
\newblock \emph{2015 IEEE International Conference on Computer Vision (ICCV)}, pages 1026--1034, 2015.

\bibitem[He et~al.(2023)He, He, Shi, Huang, and Suykens]{he2023learning}
M.~He, F.~He, L.~Shi, X.~Huang, and J.~A.~K. Suykens.
\newblock Learning with asymmetric kernels: Least squares and feature interpretation.
\newblock \emph{IEEE Transactions on Pattern Analysis and Machine Intelligence}, 45\penalty0 (8):\penalty0 10044--10054, 2023.

\bibitem[Hoegaerts et~al.(2004)Hoegaerts, Suykens, Vandewalle, and De~Moor]{hoegaerts2004primal}
L.~Hoegaerts, J.~A.~K. Suykens, J.~Vandewalle, and B.~De~Moor.
\newblock Primal space sparse kernel partial least squares regression for large scale problems.
\newblock In \emph{2004 IEEE International Joint Conference on Neural Networks (IEEE Cat. No. 04CH37541)}, volume~1, pages 561--563. IEEE, 2004.

\bibitem[Hoegaerts et~al.(2005)Hoegaerts, Suykens, Vandewalle, and De~Moor]{hoegaerts2005subset}
L.~Hoegaerts, J.~A.~K. Suykens, J.~Vandewalle, and B.~De~Moor.
\newblock Subset based least squares subspace regression in {RKHS}.
\newblock \emph{Neurocomputing}, 63:\penalty0 293--323, 2005.

\bibitem[Hotz and Fabian(2012)]{Hotz2012integrating}
T.~Hotz and T.~Fabian, JE.
\newblock Representation by integrating reproducing kernels.
\newblock \emph{arXiv preprint arXiv:1202.4443}, 2012.

\bibitem[Huang et~al.(2018)Huang, Suykens, Wang, Hornegger, and Maier]{Tl12018Huang}
X.~Huang, J.~A.~K. Suykens, S.~Wang, J.~Hornegger, and A.~Maier.
\newblock Classification with truncated $\ell _{1}$ distance kernel.
\newblock \emph{IEEE Transactions on Neural Networks and Learning Systems}, 29\penalty0 (5):\penalty0 2025--2030, 2018.
\newblock \doi{10.1109/TNNLS.2017.2668610}.

\bibitem[Jacot et~al.(2018)Jacot, Gabriel, and Hongler]{jacot2018neural}
A.~Jacot, F.~Gabriel, and C.~Hongler.
\newblock Neural tangent kernel: Convergence and generalization in neural networks.
\newblock \emph{Advances in neural information processing systems}, 31, 2018.

\bibitem[Jerbi et~al.(2023)Jerbi, Fiderer, Poulsen~Nautrup, K{\"u}bler, Briegel, and Dunjko]{jerbi2023quantum}
S.~Jerbi, L.~J. Fiderer, H.~Poulsen~Nautrup, J.~M. K{\"u}bler, H.~J. Briegel, and V.~Dunjko.
\newblock Quantum machine learning beyond kernel methods.
\newblock \emph{Nature Communications}, 14\penalty0 (1):\penalty0 517, 2023.

\bibitem[Kawala et~al.(2013)Kawala, Douzal, Gaussier, and Diemert]{misc_buzz_in_social_media__248}
F.~Kawala, A.~Douzal, E.~Gaussier, and E.~Diemert.
\newblock {Buzz in social media }.
\newblock UCI Machine Learning Repository, 2013.
\newblock {DOI}: https://doi.org/10.24432/C56G6V.

\bibitem[Koide and Yamashita(2006)]{koide2006asymmetric}
N.~Koide and Y.~Yamashita.
\newblock Asymmetric kernel method and its application to {Fisher}'s discriminant.
\newblock In \emph{18th International Conference on Pattern Recognition (ICPR'06)}, volume~2, pages 820--824. IEEE, 2006.

\bibitem[Lauriola and Aiolli(2020)]{lauriola2020mklpy}
I.~Lauriola and F.~Aiolli.
\newblock Mklpy: a python-based framework for multiple kernel learning.
\newblock \emph{arXiv preprint arXiv:2007.09982}, 2020.

\bibitem[Liang and Rakhlin(2020)]{liang2020just}
T.~Liang and A.~Rakhlin.
\newblock Just interpolate: Kernel “ridgeless” regression can generalize.
\newblock \emph{THE ANNALS}, 48\penalty0 (3):\penalty0 1329--1347, 2020.

\bibitem[Lin et~al.(2022)Lin, Zhang, and Zhang]{lin2022reproducing}
R.~R. Lin, H.~Z. Zhang, and J.~Zhang.
\newblock On reproducing kernel {Banach} spaces: Generic definitions and unified framework of constructions.
\newblock \emph{Acta Mathematica Sinica, English Series}, 38\penalty0 (8):\penalty0 1459--1483, 2022.

\bibitem[Lin et~al.(2024)Lin, Chang, and Sun]{lin2024kernel}
S.-B. Lin, X.~Chang, and X.~Sun.
\newblock Kernel interpolation of high dimensional scattered data.
\newblock \emph{SIAM Journal on Numerical Analysis}, 62\penalty0 (3):\penalty0 1098--1118, 2024.

\bibitem[Liu et~al.(2022)Liu, Suykens, and Cevher]{liu2022double}
F.~Liu, J.~A.~K. Suykens, and V.~Cevher.
\newblock On the double descent of random features models trained with sgd.
\newblock \emph{Advances in Neural Information Processing Systems}, 35:\penalty0 34966--34980, 2022.

\bibitem[Ma et~al.(2017)Ma, Bassily, and Belkin]{Ma2017ThePO}
S.~Ma, R.~Bassily, and M.~Belkin.
\newblock The power of interpolation: Understanding the effectiveness of sgd in modern over-parametrized learning.
\newblock In \emph{International Conference on Machine Learning}, 2017.

\bibitem[Mackenzie and Tieu(2004)]{mackenzie2004asymmetric}
M.~Mackenzie and A.~K. Tieu.
\newblock Asymmetric kernel regression.
\newblock \emph{IEEE transactions on neural networks}, 15\penalty0 (2):\penalty0 276--282, 2004.

\bibitem[Mao et~al.(2023)Mao, Shi, and Zhou]{mao2023approximating}
T.~Mao, Z.~Shi, and D.-X. Zhou.
\newblock Approximating functions with multi-features by deep convolutional neural networks.
\newblock \emph{Analysis and Applications}, 21\penalty0 (01):\penalty0 93--125, 2023.

\bibitem[Meanti et~al.(2022)Meanti, Carratino, De~Vito, and Rosasco]{meanti2022efficient}
G.~Meanti, L.~Carratino, E.~De~Vito, and L.~Rosasco.
\newblock Efficient hyperparameter tuning for large scale kernel ridge regression.
\newblock In \emph{International Conference on Artificial Intelligence and Statistics}, pages 6554--6572. PMLR, 2022.

\bibitem[Mei and Montanari(2022)]{mei2022generalization}
S.~Mei and A.~Montanari.
\newblock The generalization error of random features regression: Precise asymptotics and the double descent curve.
\newblock \emph{Communications on Pure and Applied Mathematics}, 75\penalty0 (4):\penalty0 667--766, 2022.

\bibitem[Montanari and Zhong(2020)]{Montanari2020TheIP}
A.~Montanari and Y.~Zhong.
\newblock The interpolation phase transition in neural networks: Memorization and generalization under lazy training.
\newblock \emph{ArXiv}, abs/2007.12826, 2020.

\bibitem[Moreno et~al.(2003)Moreno, Ho, and Vasconcelos]{moreno2003kullback}
P.~Moreno, P.~Ho, and N.~Vasconcelos.
\newblock A {Kullback-Leibler} divergence based kernel for {SVM} classification in multimedia applications.
\newblock \emph{Advances in neural information processing systems}, 16, 2003.

\bibitem[Natarajan(1995)]{natarajan1995sparse}
B.~K. Natarajan.
\newblock Sparse approximate solutions to linear systems.
\newblock \emph{SIAM journal on computing}, 24\penalty0 (2):\penalty0 227--234, 1995.

\bibitem[Oglic and G{\"a}rtner(2018)]{oglic2018learning}
D.~Oglic and T.~G{\"a}rtner.
\newblock Learning in reproducing kernel kre{\i}n spaces.
\newblock In \emph{International conference on machine learning}, pages 3859--3867. PMLR, 2018.

\bibitem[Petersen and Pedersen(2008)]{petersen2008matrix}
K.~B. Petersen and M.~S. Pedersen.
\newblock The matrix cookbook.
\newblock \emph{Technical University of Denmark}, 7\penalty0 (15):\penalty0 510, 2008.

\bibitem[Pintea et~al.(2018)Pintea, van Gemert, and Smeulders]{pintea2018asymmetric}
S.~L. Pintea, J.~C. van Gemert, and A.~W. Smeulders.
\newblock Asymmetric kernel in gaussian processes for learning target variance.
\newblock \emph{Pattern Recognition Letters}, 108:\penalty0 70--77, 2018.

\bibitem[Radhakrishnan et~al.(2022)Radhakrishnan, Beaglehole, Pandit, and Belkin]{radhakrishnan2022feature}
A.~Radhakrishnan, D.~Beaglehole, P.~Pandit, and M.~Belkin.
\newblock Feature learning in neural networks and kernel machines that recursively learn features.
\newblock \emph{arXiv preprint arXiv:2212.13881}, 2022.

\bibitem[Rakhlin and Zhai(2019)]{rakhlin2019consistency}
A.~Rakhlin and X.~Zhai.
\newblock Consistency of interpolation with laplace kernels is a high-dimensional phenomenon.
\newblock In \emph{Conference on Learning Theory}, pages 2595--2623. PMLR, 2019.

\bibitem[Romeu-Guallart and Zamora-Martinez(2014)]{misc_sml2010_274}
P.~Romeu-Guallart and F.~Zamora-Martinez.
\newblock {SML2010}.
\newblock UCI Machine Learning Repository, 2014.
\newblock {DOI}: https://doi.org/10.24432/C5RS3S.

\bibitem[Rudi et~al.(2017)Rudi, Carratino, and Rosasco]{rudi2017falkon}
A.~Rudi, L.~Carratino, and L.~Rosasco.
\newblock Falkon: An optimal large scale kernel method.
\newblock \emph{Advances in neural information processing systems}, 30, 2017.

\bibitem[Saunders et~al.(1998)Saunders, Gammerman, and Vovk]{Saunders1998RidgeRL}
C.~Saunders, A.~Gammerman, and V.~Vovk.
\newblock Ridge regression learning algorithm in dual variables.
\newblock In \emph{International Conference on Machine Learning}, 1998.

\bibitem[Shi(2013)]{shi2013learning}
L.~Shi.
\newblock Learning theory estimates for coefficient-based regularized regression.
\newblock \emph{Applied and Computational Harmonic Analysis}, 34\penalty0 (2):\penalty0 252--265, 2013.

\bibitem[Shi et~al.(2011)Shi, Feng, and Zhou]{Shi2011ConcentrationEF}
L.~Shi, Y.~Feng, and D.-X. Zhou.
\newblock Concentration estimates for learning with $\ell_1$-regularizer and data dependent hypothesis spaces.
\newblock \emph{Applied and Computational Harmonic Analysis}, 31:\penalty0 286--302, 2011.

\bibitem[{Shi} et~al.(2019){Shi}, {Huang}, {Feng}, and {Suykens}]{shi2019sparse}
L.~{Shi}, X.~{Huang}, Y.~{Feng}, and J.~A.~K. {Suykens}.
\newblock Sparse kernel regression with coefficient-based $\ell_q-$regularization.
\newblock \emph{Journal of Machine Learning Research}, 20\penalty0 (161):\penalty0 1--44, 2019.

\bibitem[Smale and Zhou(2007)]{Smale2007LearningTE}
S.~Smale and D.-X. Zhou.
\newblock Learning theory estimates via integral operators and their approximations.
\newblock \emph{Constructive Approximation}, 26:\penalty0 153--172, 2007.

\bibitem[Steinwart and Christmann(2008)]{steinwart2008support}
I.~Steinwart and A.~Christmann.
\newblock \emph{Support Vector Machines}.
\newblock Springer Science \& Business Media, 2008.

\bibitem[Steinwart et~al.(2016)Steinwart, Thomann, and Schmid]{steinwart2016learning}
I.~Steinwart, P.~Thomann, and N.~Schmid.
\newblock Learning with hierarchical gaussian kernels.
\newblock \emph{arXiv preprint arXiv:1612.00824}, 2016.

\bibitem[Suykens(2016)]{suykens2016svd}
J.~A.~K. Suykens.
\newblock {SVD} revisited: A new variational principle, compatible feature maps and nonlinear extensions.
\newblock \emph{Applied and Computational Harmonic Analysis}, 40\penalty0 (3):\penalty0 600--609, 2016.

\bibitem[Suykens and Vandewalle(1999)]{suykens1999least}
J.~A.~K. Suykens and J.~Vandewalle.
\newblock Least squares support vector machine classifiers.
\newblock \emph{Neural processing letters}, 9:\penalty0 293--300, 1999.

\bibitem[Suykens et~al.(2002)Suykens, Van~Gestel, De~Brabanter, De~Moor, and Vandewalle]{Suykens2002LeastSS}
J.~A.~K. Suykens, T.~Van~Gestel, J.~De~Brabanter, B.~De~Moor, and J.~Vandewalle.
\newblock \emph{Least Squares Support Vector Machines}.
\newblock World Scientific, 2002.

\bibitem[Suzuki(2011)]{suzuki2011unifying}
T.~Suzuki.
\newblock Unifying framework for fast learning rate of non-sparse multiple kernel learning.
\newblock \emph{Advances in Neural Information Processing Systems}, 24, 2011.

\bibitem[Tsanas et~al.(2009)Tsanas, Little, McSharry, and Ramig]{tsanas2009accurate}
A.~Tsanas, M.~Little, P.~McSharry, and L.~Ramig.
\newblock Accurate telemonitoring of parkinson’s disease progression by non-invasive speech tests.
\newblock \emph{Nature Precedings}, pages 1--1, 2009.

\bibitem[Tsigler and Bartlett(2023)]{tsigler2023benign}
A.~Tsigler and P.~L. Bartlett.
\newblock Benign overfitting in ridge regression.
\newblock \emph{Journal of Machine Learning Research}, 24\penalty0 (123):\penalty0 1--76, 2023.

\bibitem[Vlachos and Meyer(2005)]{vlachos2005statlib}
P.~Vlachos and M.~Meyer.
\newblock Statlib datasets archive.
\newblock \emph{http://lib.stat.cmu.edu/datasets}, 2005.

\bibitem[Vovk(2013)]{vovk2013kernel}
V.~Vovk.
\newblock Kernel ridge regression.
\newblock In \emph{Empirical inference}, pages 105--116. Springer, 2013.

\bibitem[Williams and Seeger(2000)]{williams2000using}
C.~Williams and M.~Seeger.
\newblock Using the {N}ystr{\"o}m method to speed up kernel machines.
\newblock \emph{Advances in neural information processing systems}, 13, 2000.

\bibitem[Wils(1970)]{Wils1970DirectIO}
W.~Wils.
\newblock Direct integrals of hilbert spaces i.
\newblock \emph{Mathematica Scandinavica}, 26:\penalty0 73--88, 1970.

\bibitem[Wright and Gonzalez(2021)]{wright2021transformers}
M.~A. Wright and J.~E. Gonzalez.
\newblock Transformers are deep infinite-dimensional non-mercer binary kernel machines.
\newblock \emph{arXiv preprint arXiv:2106.01506}, 2021.

\bibitem[Wu et~al.(2006)Wu, Ying, and Zhou]{Wu2006LearningRO}
Q.~Wu, Y.~Ying, and D.-X. Zhou.
\newblock Learning rates of least-square regularized regression.
\newblock \emph{Foundations of Computational Mathematics}, 6:\penalty0 171--192, 2006.

\bibitem[Wu et~al.(2010)Wu, Xu, Li, and Oyama]{wu2010asymmetric}
W.~Wu, J.~Xu, H.~Li, and S.~Oyama.
\newblock Asymmetric kernel learning.
\newblock \emph{Technical Report, Microsoft Research}, 2010.

\bibitem[Xiao et~al.(2017)Xiao, Rasul, and Vollgraf]{xiao2017fashion}
H.~Xiao, K.~Rasul, and R.~Vollgraf.
\newblock Fashion-mnist: a novel image dataset for benchmarking machine learning algorithms.
\newblock \emph{arXiv preprint arXiv:1708.07747}, 2017.

\bibitem[Ye and Zhou(2008)]{ye2008learning}
G.-B. Ye and D.-X. Zhou.
\newblock Learning and approximation by gaussians on riemannian manifolds.
\newblock \emph{Advances in Computational Mathematics}, 29:\penalty0 291--310, 2008.

\bibitem[Ying and Campbell(2010)]{ying2010rademacher}
Y.~Ying and C.~Campbell.
\newblock Rademacher chaos complexities for learning the kernel problem.
\newblock \emph{Neural computation}, 22\penalty0 (11):\penalty0 2858--2886, 2010.

\bibitem[Ying and Zhou(2007)]{ying2007learnability}
Y.~Ying and D.-X. Zhou.
\newblock Learnability of gaussians with flexible variances.
\newblock \emph{Journal of Machine Learning Research}, 8:\penalty0 249--276, 2007.

\bibitem[Zhang et~al.(2009)Zhang, Xu, and Zhang]{zhang2009reproducing}
H.~Zhang, Y.~Xu, and J.~Zhang.
\newblock Reproducing kernel banach spaces for machine learning.
\newblock \emph{Journal of Machine Learning Research}, 10\penalty0 (12), 2009.

\bibitem[Zhang and Zhang(2023)]{zhang2023nearly}
Y.~Zhang and M.-L. Zhang.
\newblock Nearly-tight bounds for deep kernel learning.
\newblock In \emph{International Conference on Machine Learning}, pages 41861--41879. PMLR, 2023.

\bibitem[Zheng et~al.(2013)Zheng, Gallagher, and Kulasekera]{Zheng2013AdaptivelyWK}
Q.~Zheng, C.~M. Gallagher, and K.~B. Kulasekera.
\newblock Adaptively weighted kernel regression.
\newblock \emph{Journal of Nonparametric Statistics}, 25:\penalty0 855 -- 872, 2013.

\bibitem[Zhou(2003)]{Zhou2003CapacityOR}
D.-X. Zhou.
\newblock Capacity of reproducing kernel spaces in learning theory.
\newblock \emph{IEEE Transactions on Information Theory}, 49:\penalty0 1743--1752, 2003.

\bibitem[Zhou and Huo(2024)]{zhou2024learning}
T.-Y. Zhou and X.~Huo.
\newblock Learning ability of interpolating deep convolutional neural networks.
\newblock \emph{Applied and Computational Harmonic Analysis}, 68:\penalty0 101582, 2024.

\bibitem[Zhuang et~al.(2011)Zhuang, Tsang, and Hoi]{zhuang2011two}
J.~Zhuang, I.~W. Tsang, and S.~C. Hoi.
\newblock Two-layer multiple kernel learning.
\newblock In \emph{Proceedings of the fourteenth international conference on artificial intelligence and statistics}, pages 909--917. JMLR Workshop and Conference Proceedings, 2011.

\end{thebibliography}

\end{document}